\documentclass[twoside,11pt]{article}

\usepackage{jmlr2e}
\usepackage{amsmath}
\usepackage{subcaption}
\usepackage{multirow}
\usepackage{multicol}
\usepackage{array}
\usepackage{graphicx}
\usepackage{booktabs}
\hypersetup{colorlinks,breaklinks,
             linkcolor=blue,urlcolor=blue,
             anchorcolor=blue,citecolor=blue}

\newcommand{\blue}{\normalcolor}
\newcommand{\nc}{\normalcolor}
\newcommand{\tcb}[1]{#1}
\usepackage{color}
\newcommand{\vv}{p}
\newcommand{\V}{P}
\newcommand{\tu}{\tau}
\newcommand{\zrd}{\sigma}
\newcommand{\eps}{\varepsilon}
\newcommand{\bfu}{\mathbf{u}}
\newcommand{\x}{\mathbf{x}}
\newcommand{\z}{\mathbf{z}}

\renewcommand{\l}[1]{{\ell^{#1}(\X)}}

\newcommand{\lx}[1]{{\ell^{#1}(\X^2)}}
\newcommand{\ipg}[1]{\left\langle #1\right\rangle_{\l2}}

\newcommand{\ipv}[1]{\left\langle #1\right\rangle_{\lx2}}

\renewcommand{\ng}[2]{\left\| #2\right\|_{\l{#1}}}
\newcommand{\nv}[1]{\left\| #1\right\|_{\lx2}}

\renewcommand{\L}{\mathcal{L}}
\renewcommand{\u}{\mathbf{u}}
\renewcommand{\v}{\mathbf{v}}
\newcommand{\X}{{\mathcal X}}
\newcommand{\N}{\mathbb{N}}
\newcommand{\U}{\mathcal{U}}
\newcommand{\R}{\mathbb{R}}
\let\div\relax
\DeclareMathOperator{\div}{div}

\newcommand{\Loss}{\mathcal{J}}
\newcommand{\kth}{$k^{\rm th}$}

\DeclareMathOperator*{\argmax}{\arg\!\max}

\usepackage{lastpage}
\jmlrheading{25}{2024}{1-\pageref{LastPage}}{12/24; Revised 12/25}{TBD}{21-0000}{Jason Brown, Bohan Chen, Harris Hardiman-Mostow, Jeff Calder, Andrea L. Bertozzi}

\ShortHeadings{GLL: A Differentiable Graph Learning Layer}{Brown$^{*}$, Chen$^{*}$, Hardiman-Mostow$^{*}$, Calder, and Bertozzi}
\firstpageno{1}

\begin{document}

\title{GLL: A Differentiable Graph Learning Layer for Neural Networks}
% \title{GLL: A Parameter-Free Graph Learning Layer for Representation Learning via Backpropagation}

\author{\name Jason Brown\thanks{Equal Contribution} \email jasbrown@g.ucla.edu \\
       \addr Department of Mathematics \\
       University of California, Los Angeles \\
       Los Angeles, CA 90095, USA
       \AND
       \name Bohan Chen$^*$ \email bhchen@caltech.edu \\
       \addr Computing + Mathematical Sciences (CMS) Department\\
       California Institute of Technology \\
       Pasadena, CA 91125, USA
       \AND
       \name Harris Hardiman-Mostow$^*$ \email hhm@math.ucla.edu \\
       \addr Department of Mathematics\\
       University of California, Los Angeles\\
       Los Angeles, CA 90095, USA
       \AND
       \name Jeff Calder \email jwcalder@umn.edu \\
       \addr School of Mathematics\\
       University of Minnesota\\
       Minneapolis, MN 55455, USA
       \AND
       \name Andrea L. Bertozzi \email bertozzi@math.ucla.edu \\
       \addr Department of Mathematics\\
       University of California, Los Angeles\\
       Los Angeles, CA 90095, USA}

\editor{TBD}

\maketitle

\begin{abstract}
Standard deep learning architectures used for classification generate label predictions with a projection head and softmax activation function. Although successful, these methods fail to leverage any relational information between samples for generating label predictions. In recent works, graph-based learning techniques, namely Laplace learning (\cite{zhu2003semi}) or \textit{label propagation}, have been heuristically combined with neural networks for both supervised and semi-supervised learning (SSL) tasks. However, prior works approximate the gradient of the loss function with respect to the graph learning algorithm or decouple the processes; end-to-end integration with neural networks is not achieved. In this work, we derive backpropagation equations, via the adjoint method, for the inclusion of a general family of graph learning layers into a neural network.
\tcb{The resulting method, distinct from graph neural networks, allows us to precisely integrate similarity graph construction and graph Laplacian-based label propagation into a neural network layer, replacing a projection head and softmax activation function for general classification tasks}. Our experimental results demonstrate smooth label transitions across data, improved generalization and robustness to adversarial attacks, and improved training dynamics compared to a standard softmax-based approach\footnote{Our code is available at \url{https://github.com/jwcalder/GraphLearningLayer}}.
\end{abstract}
\begin{keywords}
 Graph learning, deep learning, Laplace learning, label propagation, adversarial robustness
\end{keywords}

% introduction
\section{Introduction}

Graph-based learning is a powerful class of transductive machine learning techniques that leverage 
%the topology of a similarity graph constructed from data #this first sentence is more general, similarity graphs are introduced below.
graph structures in data for tasks such as semi-supervised learning, data visualization, and dimension reduction. Beginning with the seminal work by \cite{zhu2003semi} on semi-supervised learning, graph-based learning techniques have since proliferated \citep{belkin2004semi,belkin2004regularization, bengio2006label,zhou2011semi,wang2013dynamic,ham2005semisupervised,lee2013graph,calder2020poisson,calder2020properly,shi2017weighted} and found applications in many problem areas.
%Breaking this sentence apart: Graphs may be intrinsically present in a data set, such as in network science problems, where relational information is present in links between servers in a network or between users on social media platforms, or the data points themselves may be graphs (e.g., molecules in drug discovery). Also added "and is useful for imposing a graph structure when there is none given"
 In some settings, graphs are intrinsically present in a data set: links between servers in a network or between users on social media platforms are common in network science, while other problems may represent data as graphs themselves, such as molecules in drug discovery.
 In general, a \emph{similarity graph} can be constructed over a data set by comparing pairs of data points using any available notion of similarity.
 %and is useful for imposing a graph structure when there is none given. 
 In this paper, we focus on graph-based learning using similarity graphs; that is, we do not assume the graph is given as part of the metadata in the problem. Thus, our work is orthogonal and complementary to much of the graph neural network literature, where the graph is \emph{a priori} fixed.

Similarity graphs and corresponding graph techniques have been extensively studied theoretically using variational analysis, random walks, harmonic functions, and partial differential equations (see e.g., \cite{calder_consistency_2019, calder2022improved, calder2020properly, calder2018game, slepcev2019analysis,dunlop2020large,garcia2020error,hein2007graph,hein2005graphs,calder2022hamilton, calder2023rates}). Graph learning also contrasts with modern deep learning approaches in that it is a data-dependent model without any trainable parameters. Previous works \citep{chapman2023novel,chen2023batch,chen2023graphigarss} show that graph learning outperforms many popular classifiers like support vector machine (\cite{cortes1995support}), random forest (\cite{ho1995random}), and neural networks on classification tasks, especially at very low label rates.

% This motivates the development of a method that seamlessly integrates graph learning and neural networks; namely, one that can train end-to-end through both the neural network feature extractor and the graph-based classifier. 

% \jason{Note for consistency: we are trying to replace MLP with softmax or just softmax? (BH: replace MLP with softmax)}

Recently, many works have focused on combining graph-based learning with deep learning. In such works, unsupervised or semi-supervised neural networks learn an embedding of the data into a feature space and use the similarity between features to construct the graph. Examples of such feature extraction networks include variational autoencoders (VAEs) (\cite{kingma2013auto,calder2020poisson, calder2022hamilton}), convolutional neural networks (CNNs) (\cite{lecun1989backpropagation, enwright2023deep}), convolutional VAEs (\cite{pu2016variational,miller2022graph,chapman2023novel}), and contrastive learning (\cite{chen2020simple,brown2023utilizing, brown2023material}). Some works have used graph learning to generate pseudolabels to train a deep neural network in low label rate settings (\cite{sellars2021laplacenet,iscen2019label}). Of particular relevance to our work is a recent approach that replaces the softmax activation layer with graph learning to improve generalization and adversarial robustness (\cite{wang2021graph}). \tcb{However, \citet{wang2021graph} do not compute the true gradients with respect to the loss, due to difficulties in backpropagating through graph learning, and instead replace them with a heuristic. Our work was partially motivated by the issues raised by \cite{wang2021graph}. }
%however, the gradients of the loss with respect to the graph learning classifier are only approximated by an equivalently-sized linear classifier. #I don't think this last sentence describes their work well.

%This growing body of work motivates the need to integrate neural networks and graph learning in an exact way, where the backpropagation gradients are computed exactly. 

%Our work adopts graph-based learning as a classifier within a neural network architecture, replacing the usual ``projection head" of a neural network with graph-based learning. [REWRITE] %- that is, a single or multilayer perceptron (MLP) that maps from the feature space to logits, followed by a softmax classifier (see Figure~\ref{fig:flowchart}). 
% Removed a comma here
A significant gap exists in the literature as no work fully integrates graph learning into the gradient-based learning framework of neural networks.  
In this paper, we leverage the adjoint graph Laplace equation to derive the exact backpropagation gradients for similarity graph construction and a general family of graph learning methods, and show how to implement them efficiently.  
This significantly expands upon the previous work by \cite{enwright2023deep}, where the backpropagation equations were only derived for Laplace learning on a fully connected graph. The backpropagation equations allow us to design a neural network layer that replaces the standard projection head and softmax classifier with a graph learning classifier (see Figure~\ref{fig:flowchart} for a comparison). We call this novel layer the \emph{Graph Learning Layer} (GLL). A GLL (Figure~\ref{fig:flowchart_jason}) can be combined with \emph{any} neural network architecture and enables end-to-end training in a gradient-based optimization framework.
Through a variety of experiments, we show that the GLL learns the intrinsic geometry of data, improves the training dynamics and test accuracy of both shallow and deep networks across levels of supervision, and is significantly more robust to adversarial attacks compared to a projection head and softmax classifier. 

%In addition to Laplace Learning, we also derive the gradients of the recently developed Poisson Reweighted Laplace Learning by \cite{miller2023poisson}. Through further experiments, we also demonstrate GLL can reduce overfitting, and can learn the underlying data manifold - a fundamentally different learning approach than an MLP.

\subsection{Related Work}

%removed em dashes because it may make it look AI and 'i.e.,'
Graph-based semi-supervised learning (SSL) methods construct a graph that allows inferences of unlabeled nodes using labeled nodes and the graph topology. The seminal work of \textit{Laplace learning} (\cite{zhu2003semi}), also known as \textit{label propagation}, led to a proliferation of related techniques \citep{belkin2004semi,belkin2004regularization, bengio2006label,zhou2011semi,wang2013dynamic,ham2005semisupervised,lee2013graph,garcia2014multiclass,calder2020properly,shi2017weighted}. Laplace learning computes a graph harmonic function to extend the given labels to the remainder of the graph, essentially finding the smoothest function that agrees with the given labels.  When the label rate is low (close to one label per class), Laplace learning degenerates and suffers from a ``spiking'' phenomenon around labeled points, while predictions are roughly constant (\cite{calder2020poisson,el2016asymptotic,nadler2009infiniteunlabelled}). Several methods have been proposed for these very low label rate regimes, including reweighted Laplace learning (\cite{shi2017weighted, calder2020properly, miller2023poisson}), $p$-Laplace methods (\cite{zhou2011semi,el2016asymptotic,flores2022analysis,slepcev2019analysis}), and Poisson learning (\cite{calder2020poisson}). Analysis of the degeneracy of Laplace learning from a random walk and variational perspective can be found in (\cite{calder2020poisson,calder2023rates}). Some recent work has also focused on graph total variation problems using graph cuts (\cite{bertozzi2016diffuse, merkurjev2017modified, merkurjev2018semi}).

Another related, but distinct, area of work in machine learning on graphs is graph neural networks, or GNNs \citep{kipf2016semi,defferrard2016convolutional,wu2020comprehensive}. The GNN literature considers how to incorporate \emph{a priori} graph information relevant to the task, such as co-authorship in a citation graph, within a neural network architecture for various learning tasks, such as learning node embeddings or semi-supervised learning. \blue The setting and problem space is quite different from ours. \nc
%since the graph-based learning methods considered in this work construct a similarity graph based on the node features alone (through e.g., a $k$-nearest neighbors search), and 
In particular, we do not assume the data has an inherent graph structure. Moreover, GNNs contain trainable weight matrices, while the graph-based learning methods we consider here have no trainable parameters. There are some works, such as graph attention networks \citep{velivckovic2017graph}, which learn how to assign weights to a graph, but the graph adjacency structure is still fixed \emph{a priori} through masking in the attention mechanism.

Previous work by \cite{agrawal2019differentiable} studied backpropagation techniques for the differentiation of convex optimization problems. Our paper extends this work to nonlinear equations on graphs involving graph Laplacians. 

Deep learning methods have increasingly leveraged graph-based approaches or drawn inspiration from them.
\tcb{\cite{chen2024survey} provides a broad overview of applications and modalities of graph structures and neural networks in computer vision.
%that extend beyond the scope of this paper
%,but we will highlight inspiration and relevant works to our own
}
Two recent works propagate labels on a graph to generate pseudolabels in a semi-supervised setting (\cite{iscen2019label,sellars2021laplacenet}). 
\tcb{
Other works have applied transformers directly to a generated similarity graph for node classification (\cite{wu2022nodeformer}, \cite{han2022vision}, \cite{zheng2022graph}). 
% Other works graphs have been constructed from image patches drawn from a single full image, essentially using the patches as nodes, before classification (\cite{zhang2022differentiable}, ). 
% wrong cite key?
}
However, in all of these works, the graph construction and label propagation process is disconnected from the neural network and is not directly integrated into the backpropagation or learning process. 
Contrastive learning (CL) uses pair-wise similarity, akin to a complete similarity graph, to facilitate the learning process (\cite{chen2020big,chen2020simple,khosla2020supervised,wang2023message}). \cite{wang2021graph} utilizes graph learning as a classifier within a neural network, but they approximate the backpropagation gradients with a linear layer instead of computing them exactly. 

\begin{figure}[t]
\centering
\begin{subfigure}{\textwidth}
  \centering
  \includegraphics[trim=10pt 85pt 30pt 110pt,clip,height=5cm]{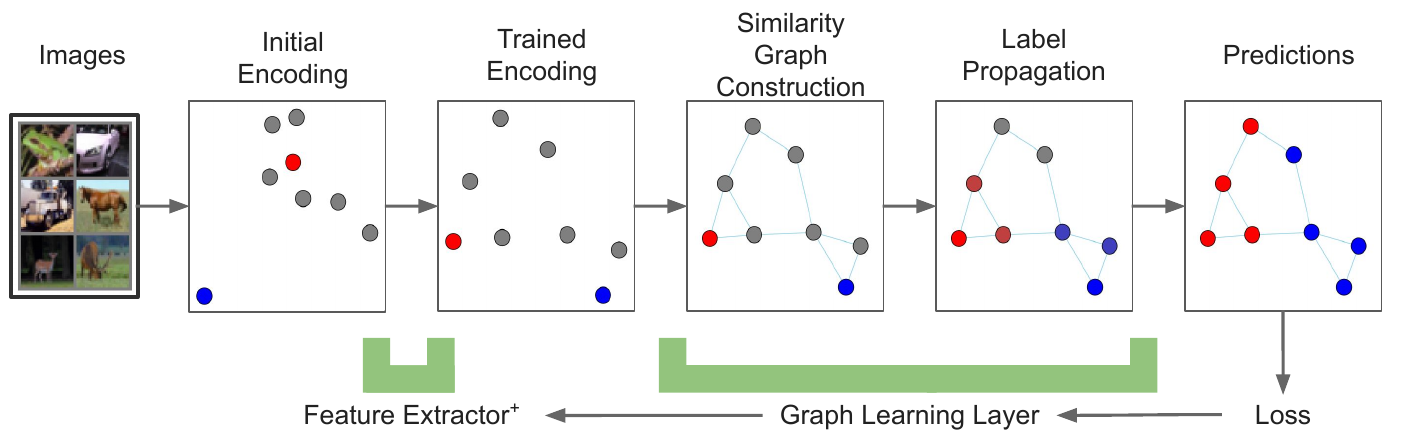}
  \caption{Graph Learning Layer}
\end{subfigure}
\begin{subfigure}{\textwidth}
  \centering
  \includegraphics[trim=10pt 100pt 150pt 100pt,clip,height=5cm]{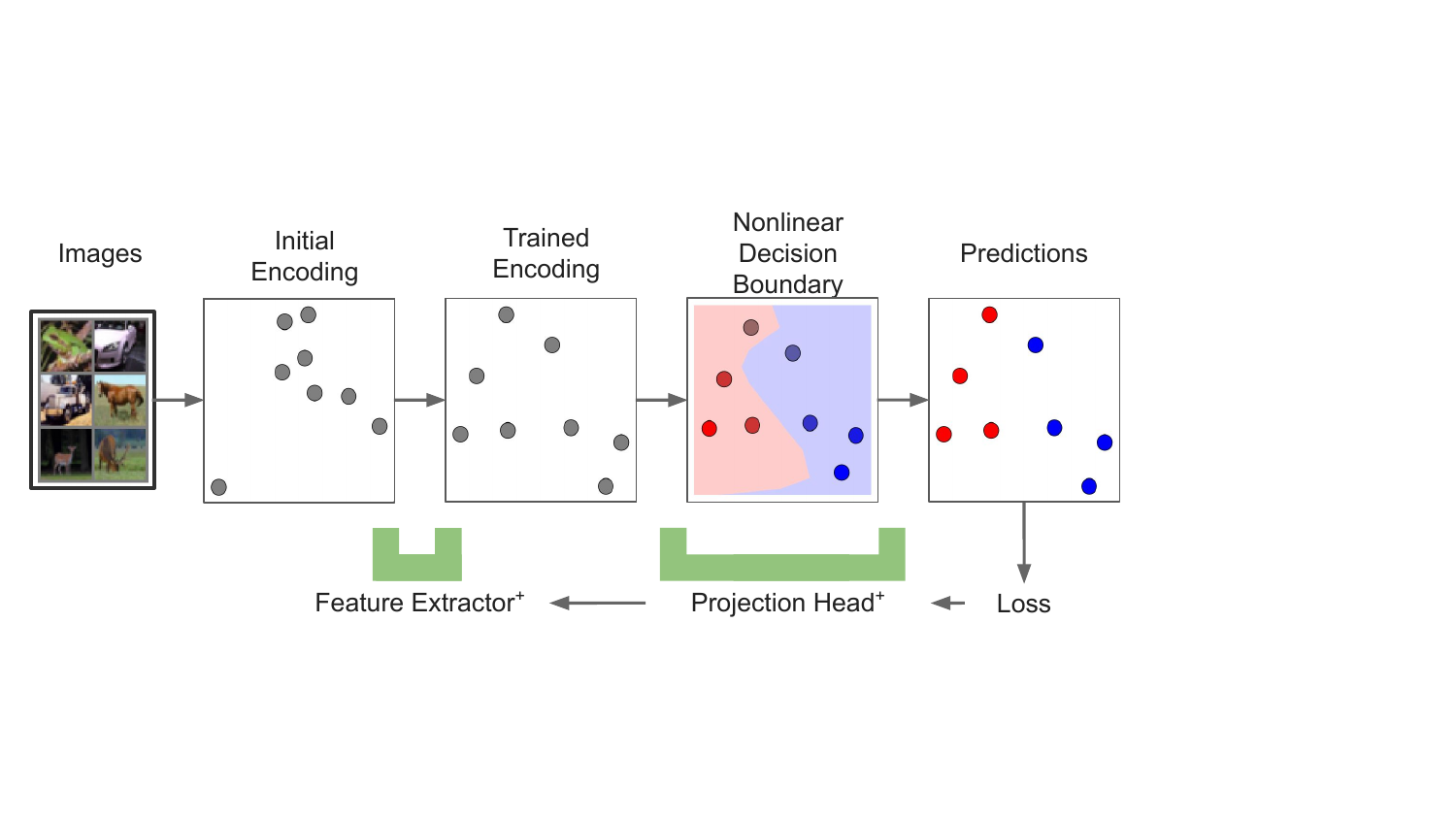} %width=0.95\textwidth
  \caption{Standard MLP projection head}
\end{subfigure}
\caption{\tcb{Visualization of the graph learning layer (GLL) within a neural network pipeline compared to a standard multilayer perceptron (MLP) projection head. For the GLL, a combination of labeled and unlabeled input images are batched and normalized into data for the feature extractor. The feature network encodes the images in feature space. The GLL is a combination of two steps; it generates a similarity graph from the encoded data and then propagates the labels across the graph, giving predictions on the unlabeled data. These predictions are used in the desired loss function and gradients can flow through the GLL to update the feature extractor. By contrast, the MLP projection head pipeline does not involve any initially labeled nodes and the classification method ignores relational information in the encodings. The \(^+\) indicates learnable parameters in the network; note GLL has none.}}
\label{fig:flowchart_jason}
\end{figure}

\begin{figure}[t]
\centering
\begin{subfigure}{\textwidth}
  \centering
  \includegraphics[width=0.7\textwidth]{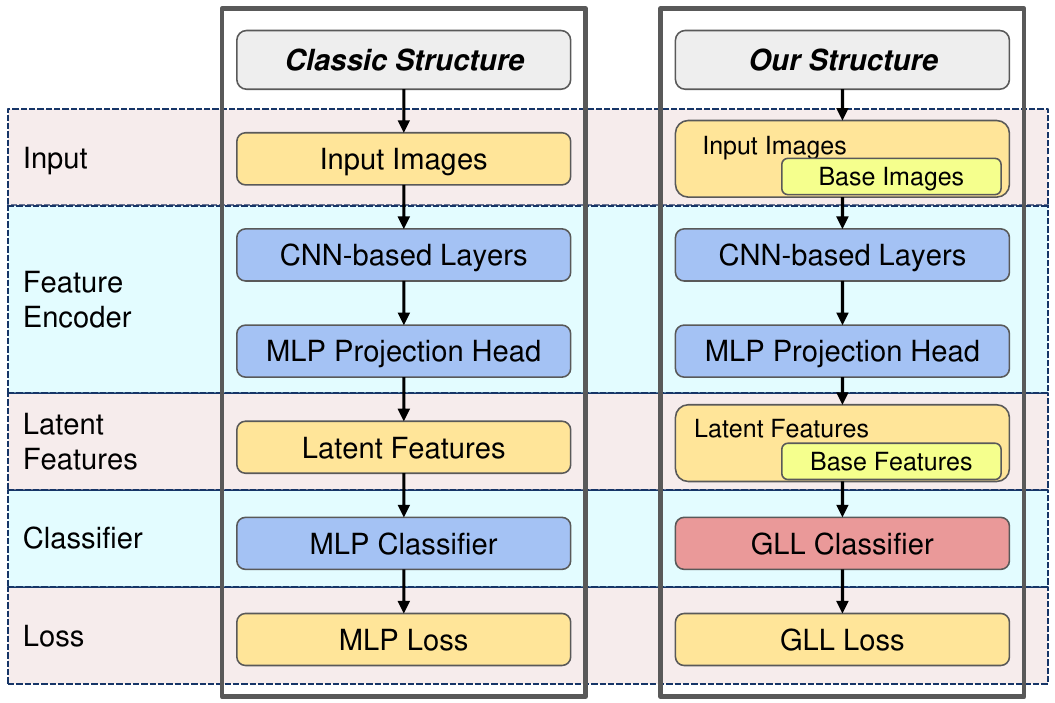}
\end{subfigure}
\\[1em]  
\begin{subfigure}{\textwidth}
  \centering
  \includegraphics[width=0.7\textwidth]{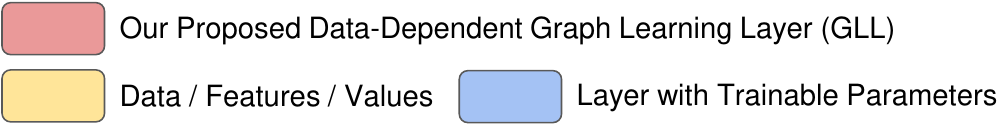}
\end{subfigure}
\caption{Comparison between the classic structure and our proposed graph learning layer (GLL) neural network structures for image classification. Both structures consist of a feature encoder and a classifier, but we replace the multilayer perceptron (MLP) classifier - consisting of one or more linear and activation layers and a final softmax classifier - with our data-dependent GLL with no trainable parameters. We reserve a subsample of the training data - denoted \emph{base} samples - to use as the labeled nodes for label propagation. The samples that are not from the base set are use to compute the loss.}
\label{fig:flowchart}
\end{figure}

\subsection{Our Contributions}

The most significant difference between our work and any of the preceding works is the derivation and implementation of the exact backpropagation equations through the Graph Learning Layer, including the similarity graph construction and the graph-based classifier. We also give a sparse, efficient implementation of these equations. %Rather than using an approximation, or detaching the graph learning algorithm from the gradient-based framework entirely, we derive the correct adjoint equations for backpropagation through the graph construction and graph learning algorithm. We also implement these equations in a sparse manner, minimizing memory usage, yielding efficient implementations, and allowing for larger batch sizes. 
This is a significant extension of \cite{enwright2023deep}, where the authors derive the backpropagation gradients for Laplace learning in the special case of a fully connected graph. In this work, we derive the gradients for any $k$-nearest neighbors-type graph construction and any nonlinear elliptic graph learning equation. In addition, we conduct detailed experiments on a variety of benchmark datasets and neural network architectures. Our analysis includes visualizing clustering effects in the latent space under different GLL hyperparameters, demonstrating improved training dynamics and generalization, and showcasing improved robustness to adversarial attacks. In the vast majority experiments - especially at low label rates - a GLL outperforms the standard projection head and softmax-based classifier.

The remainder of the paper is organized as follows. Section \ref{sec:background} details the necessary background on Laplace Learning, the seminal graph-based learning algorithm. Section \ref{sec:method} describes our novel mathematical contributions powering the GLL. Namely, we derive the gradient of the loss with respect to the similarity graph weight matrix and the associated $k$-Nearest Neighbor graph construction, which allows us to integrate graph learning into a neural network training context. We also demonstrate the effect this novel GLL has on the learned neural network embeddings. Section \ref{sec:experiments} contains illustrative experiments showcasing the benefits of the GLL, demonstrating improved generalization across network architectures and levels of supervision, and superior robustness to adversarial attacks compared to the softmax classifier.

% background
\section{Background}\label{sec:background}

In this section, we give some background on the construction similarity graphs, calculus on graphs, and graph learning using various types of graph Laplace equations.

\subsection{Similarity Graph Construction}\label{sec:graph_construction}

%Graph learning algorithms require constructing a similarity matrix $W$ over the data. The values in the \(i^{th}\) row and \(j^{th}\) column of \(W\) are the weights $w_{ij}$, which, in this work, are computed between the feature vectors of neighboring nodes $x_i$ and $x_j$ using the radial basis function

\blue
Let $\X$ be the vertex set of a graph with $n \in \N$ vertices and nonnegative edge weights $w_{xy} \geq 0$, for $x,y \in \X$. Let $W=(w_{xy})_{x,y\in \X}$ denote the $n\times n$ weight matrix for the graph.
%\jason{The notation makes sense, but would it be better to write $W_{xy}=w_{xy}$ or \(W=(w_{xy})\)} \jeff{The reason for the subscript is to denote the set that $x,y$ belong to. I made it more explicit now. It's similar to writing a matrix as $W=(w_{ij})_{i,j=1}^n$}.
We always assume the weight matrix $W$ is symmetric, so $w_{xy}= w_{yx}$. We usually think of $W$ as a similarity matrix, in which case the entry $w_{xy}$ encodes how similar the nodes $x$ and $y$ are, with large positive values indicating a high degree of similarity, and $w_{xy}=0$ indicating the nodes are not similar. Given we have a feature vector in $\R^d$ for each node $x\in \X$ in the graph, a common way to construct a similarity weight matrix, which we adopt in the numerical experiments in this paper, is a $k$-nearest neighbor graph based on similarity of feature vectors. In this case, we will denote the vertex set of the graph by $\X=\{\x_1,\dots,\x_n\}\in \R^d$, where $\x_i\in \R^d$ is the feature vector for the $i^{\rm th}$ node in the graph. The weight $w_{xy}=w_{\x_i\x_j}=w_{ij}$ between nodes $\x_i$ and $\x_j$ is then given by the expression
\begin{equation}\label{eq:weightform}
w_{ij} = \exp\left(\frac{-4\Vert \x_i - \x_j \Vert^2}{\eps_k(\x_i)\eps_k(\x_j)}\right).
\end{equation}
The parameter $k$ is the number of neighbors used in the graph construction, and we choose the bandwidth parameter $\eps_k(\x_i)$ to be the distance from $\x_i$ to its $k$th nearest neighbor in $\R^d$, which allows the weights to self-tune appropriately to areas of density and sparsity. If desired for applications, \(\eps_k(\x_i)\) may be set to a constant value for all \(\x_i\). To produce a sparse graph, we set the weight $w_{ij}=0$ if $\x_i$ is not among the $k$ nearest neighbors of $\x_j$, and $\x_j$ is not among the $k$ nearest neighbors of $\x_i$. While searching for the exact $k$ nearest neighbors in high dimensional data sets is a computational hard problem, there are many fast approximate nearest neighbor algorithms that can be used in practice.\footnote{We use the Annoy python package: \url{https://github.com/spotify/annoy}} 
While similarity matrices are often used in practice, we note that the theory in this section, and much of Section \ref{sec:method}, does not place any restrictions on $W$ aside from symmetry $w_{xy}=w_{yx}$. \nc

\subsection{Calculus on graphs}\label{sec:graphcalculus}

We let $\l2$ denote the Hilbert space of functions $u:\X\to \R^k$, equipped with the inner product
\begin{equation}\label{eq:graph_inner}
\ipg{u,v} = \sum_{x\in \X} u(x)\cdot v(x),
\end{equation}
and norm  $\ng2{u}^2 = \ipg{u,u}$.  For $p\geq 1$ we also define $p$-norm
\begin{equation}\label{eq:graph_pnorm}
\ng{p}{u}^p = \sum_{x\in \X} \|u(x)\|^p.
\end{equation}
The \emph{degree} is a function $\deg \in \l2$ with $k=1$ defined by
\[\deg(x) = \sum_{y \in \X} w_{xy}.\]

We let $\lx2$ denote the space of functions $\u:\X^2\to \R^k$. A \emph{vector field} on the graph is a function $\u\in \lx2$ that is skew symmetric, so $\u(x,y)=-\u(y,x)$. We use bold face for vector fields to distinguish from functions on the graph. The gradient $\nabla u\in \lx2$ of $u\in \l2$ defined by
\begin{equation}\label{eq:graph_gradient}
\nabla u(x,y) = u(x) - u(y)
\end{equation}
is an example of a vector field over the graph. \blue Note that $\nabla u$ is a vector field, so it is defined over the edges in the graph, i.e., pairs $(x,y)$ of vertices, while $u$ is a function of a single vertex $x$ on the graph. \nc For $\u,\v\in \lx2$ we define an inner product 
\begin{equation}\label{eq:vector_inner}
\ipv{\u,\v} =\frac{1}{2}\sum_{x,y \in \X} w_{xy}\u(x,y)\v(x,y),
\end{equation}
together with a norm $\nv{\u}^2 = \ipv{\u,\u}$. 

%removed as in 'is as an operator' (?) 
The \emph{graph divergence} is an operator taking vector fields to functions in $\l2$, and is defined as the adjoint of the gradient. Here, we  define the divergence for any function $\v\in \lx2$; that is, it may not be a vector field.  Hence, for $\v\in \lx2$, the graph divergence $\div\v\in \l2$ is defined so that the identity
\begin{equation}\label{eq:div_adjoint}
\ipv{\nabla u,\v} = \ipg{u,\div \v}
\end{equation} %\jason{Typo in the subscript for the inner products - the spaces are reversed - fixed now but leaving note}
holds for all $u\in \l2$. A straightforward computation shows that
\begin{equation}\label{eq:div_formula_general}
\div \v(x) =\frac{1}{2} \sum_{y\in \X}w_{xy}(\v(x,y) - \v(y,x)).
\end{equation}
If $\v$ is a vector field, then this can be simplified to
\begin{equation}\label{eq:div_formula}
\div \v(x) = \sum_{y\in \X}w_{xy}\v(x,y).
\end{equation}

The gradient is an operator $\nabla:\l2 \to \lx2$, and the divergence is an operator $\div: \lx2\to \l2$. The graph Laplacian is the composition of these two operators $\Delta = \div \circ \nabla$, which can be expressed as 
\begin{equation}\label{eq:graph_Laplacian}
\Delta u(x) = \div \nabla u(x) =\sum_{y\in \X}w_{xy}(u(x) - u(y)).
\end{equation}
By the definition $\Delta = \div \circ \nabla$ we have
\[\ipg{\Delta u,v} = \ipv{\nabla u,\nabla v},\]
for any $u,v\in \l2$. In particular, we have
\[\ipg{\Delta u,u} = \ipv{\nabla u,\nabla u} = \nv{\nabla u}^2 = \frac{1}{2}\sum_{x,y\in \X}w_{xy}(u(x)-u(y))^2,\]
which is the Dirichlet energy on the graph. For a connected graph with $k=1$, the kernel of the graph Laplacian $\Delta$ is $1$-dimensional and spanned by the vector of all ones. In general, the multiplicity of the eigenvalue $0$ indicates the number of connected components in the graph (\cite{von2007tutorial}).

\subsection{Classical Graph-based Learning}\label{sec:classic_gl}

Let $\X = \{\x_1, \x_2, \dots, \x_n\} \in \mathbb{R}^d$ be a set of $d$-dimensional feature vectors, \blue and let $W$ denote a corresponding weight matrix, which may be constructed as a similarity matrix, as discussed in Section \ref{sec:graph_construction}. This endows the data set with a graph structure. \nc Let $\mathcal{L} \subset \X$ be the set of labeled indices that identifies which feature vectors have labels $y_i \in \{1, 2, \dots, k\}$ (for a $k$-class classification problem). \blue The task of graph-based semi-supervised learning is to utilize the graph structure to infer the labels on the unlabeled index set $\mathcal{U} = \X\setminus \L$. \nc The labels $y_i$ are represented by one-hot vectors; that is, the vector $e_{y_i} \in \mathbb{R}^k$ of all zeros, except a 1 in the $y_i^{th}$ entry. \blue As in Section \ref{sec:graphcalculus}, generic elements of $\X$, i.e., graph vertices, are denoted as $x,y\in \X$, and the corresponding entry of the weight matrix $W$ is denoted $w_{xy}$.\nc

% \jeff{I commented out the text below, which repeats a lot about the graph construction.}
%We use the geometric structure of a similarity graph constructed on $\X$ to make predictions.  (feels like we repeat this wording a lot and it doesn't add much here
%Using a symmetric similarity function \(w: \X\times \X \to \mathbb{R}^+\), we construct the similarity graph's weight matrix $W$, where $w_{x_ix_j}$ or $w_{ij}$ is the edge weight between $x_i$ and $x_j$ ($w_{xy}$ denotes generic elements $x,y\in \X$). The weight matrix \(W\) is symmetric and the edge weights correspond to similarity between nodes, with larger values indicating higher similarity. Because pairwise similarities and dense matrix operations are expensive to compute, \(W\) is typically approximated using a sparse or low rank representation. Here we use a $k$-nearest neighbors search on $\X$ with $k\ll n$ so that $W$ is sparse (see Section \ref{sec:graph_construction}).

\textit{Laplace learning} (\cite{zhu2003semi}), also called \textit{label propagation}, is a seminal method in graph-based learning. The labels are inferred by solving the graph Laplace equation 
\begin{equation}\label{eq:graph_laplace}
\left\{
\begin{aligned}
\Delta u(x) &= 0,&& \text{if } x\in \U \\
u(x) &= g(x),&& \text{if } x\in \L,
\end{aligned}
\right.
\end{equation}
where $g:\L\to \R^k$ encodes the labeled data through their one hot vector representations. The solution $u\in \l2$ is a function $u:\X\to \R^k$ where the components of $u(x)\in \R^k$ can be interpreted as the probability that  $x$ belongs to each class. The label prediction is the class with the highest probability. Equation \eqref{eq:graph_laplace} is a system of $k$ linear equations, one for each class, which are all completely independent of each other. Thus, this formulation of Laplace learning is equivalent to applying the one-vs-rest method for converting a binary classifier to the multiclass setting. 

Laplace learning also has a variational interpretation as the solution to the problem
\begin{equation}\label{eq:laplace_variational}
\min_{u\in \l2}\frac{1}{2}\sum_{x,y\in \X}w_{xy}\|u(x) - u(y)\|^2  = \nv{\nabla u}^2,
\end{equation}
subject to the condition that $u=g$ on the labeled set $\L$. The necessary condition satisfied by any minimizer of \eqref{eq:laplace_variational} is exactly the Laplace equation \eqref{eq:graph_laplace}. Furthermore, as long as the graph is connected\footnote{In fact, a weaker condition that the graph is connected to the labeled data $\L$ is sufficient, which is equivalent to asking that every connected component of the graph has at least one labeled data point.} the solution of \eqref{eq:graph_laplace} exists and is unique, as is the minimizer of \eqref{eq:laplace_variational} (see \cite{calder2018game}).

There are variants of Laplace learning that employ soft constraints on the agreement with given labels, which can be useful when there is noise or corruption of the labels, typical in many real world settings. Fixing a parameter $\lambda>0$, the soft-constrained Laplace learning problem minimizes
\begin{equation}\label{eq:laplace_soft}
\min_{u\in \l2}\left\{\nv{\nabla u}^2 + \lambda\sum_{x\in \L}\|u(x) - g(x)\|^2\right\},
\end{equation}
where $\lambda>0$ tunes the tradeoff between label fidelity and label smoothness.  Provided the graph is connected, the unique minimizer of \eqref{eq:laplace_soft} is the solution of the graph Laplace equation
\begin{equation}\label{eq:graph_laplace_soft}
\Delta u + \lambda \xi (u - g) =0,
\end{equation}
where $\xi(x)=0$ if $x\in \U$ and $\xi(x)=1$ if $x\in \L$. We recover Laplace learning \eqref{eq:graph_laplace} in the limit as $\lambda\to \infty$. %\jason{I don't have Jeff's book accessible - but lambda would likely need to go to \(inf\) in equation  indicating a mismatch/error somewhere} \jeff{You are correct, it should be $\infty$.}

The graph Laplace equations \eqref{eq:graph_laplace} and \eqref{eq:graph_laplace_soft} are positive definite symmetric linear equations, and can be solved with any number of direct or indirect techniques. When the size of the problem is small, a direct matrix inversion may be used. For larger problems, especially for large sparse systems, an iterative method like the preconditioned conjugate gradient method is preferred. 

There are also a variety of graph-based learning techniques that do not require labeled information, such as the Ginzburg-Landau formulation \citep{garcia2014multiclass}
\begin{equation}\label{eq:ginzburg}
\min_{u\in \l2}\left\{\nv{\nabla u}^2 + \lambda\sum_{x\in \X}G(u(x))\right\},
\end{equation}
where $G$ is a many-well potential, with wells located at the one hot label vectors. For binary classification, the wells may be positioned at $\pm 1$ and $G(u) = (1 - u^2)^2$. The MBO methods pioneered by \cite{merkurjev_mbo_2013,garcia2014multiclass} are effective approximations of the Ginzburg-Landau energy \eqref{eq:ginzburg} when $\lambda \gg 1$. 
This method was first developed in a semi-supervised setting but also used for modularity optimization \cite{Boyd18} and a piecewise constant Mumford-Shah model in an unsupervised setting \cite{Mumfordgraph}.

\subsection{Diagonal Perturbation}\label{sec:tau}

A modification to Laplace learning (Section \ref{sec:classic_gl}) is to add a type of Tikhonov regularization in the energy (see e.g. \cite{miller2023poisson}), leading to the minimization problem 
\begin{equation}\label{eq:laplace_variational_tau}
\min_{u\in \l2}\left\{\nv{\nabla u}^2 + \tau \ng1{u}^2\right\},
\end{equation}
subject to the condition that $u=g$ on the labeled set $\L$. Tikhonov regularization can also be added to the soft constrained  Laplace learning \eqref{eq:laplace_soft}. The minimizer of \eqref{eq:laplace_variational_tau} is the solution of the graph Laplace equation \eqref{eq:graph_laplace} except that the graph Laplacian, $\Delta$, is replaced by the diagonal perturbed Laplacian $\Delta_\tau = \Delta + \tau I$ (likewise for the soft constrained problem \eqref{eq:graph_laplace_soft}). 

This regularization term increases the convergence speed of iterative algorithm by improving the condition number of the graph Laplacian. Moreover, it induces exponential decay away from labeled data points. This inhibits label propagation and weakens predictions further from labeled data in the latent space, which can limit overconfident predictions. In Section \ref{sec:toy}, we show that employing this regularization in the GLL causes denser clustering of similar samples in the latent space learned by the neural network, which may be desireable for classification or clustering tasks.

% We are interested in exploring this regularization term in the deep learning framework because minimizing a predictive loss with exponential decay will require denser clustering of similar samples in the latent space, which is a desirable outcome for classification purposes.

%It can be shown that solving the regularized problem gives the following solution 
%\begin{equation}\label{eq:LaplaceLearning}
%U = \begin{Bmatrix} Y \\ -(L_\tau)_{\mathcal{U},\mathcal{U}}^{-1}(L_\tau)_{\mathcal{U},\mathcal{L}} Y \end{Bmatrix}
%\end{equation}
%where 
%\[
%L_\tau = L + \tau I.
%\]
%This form is very simple and will prove useful for efficiently adding the additional Tikhonov regularization term.

\subsection{$p$-Laplace learning}

There are also graph-learning methods based on the nonlinear graph $p$-Laplacian (\cite{zhou2011semi,el2016asymptotic,flores2022analysis,slepcev2019analysis,calder2018game}). One such example is the variational $p$-Laplacian, which stems from solving the optimization problem 
\begin{equation}\label{eq:plaplace_variational}
\min_{u\in \l2}\frac{1}{2}\sum_{x,y\in \X}w_{xy}\|u(x) - u(y)\|^p,
\end{equation}
subject to $u=g$ on $\L$. Taking values of $p\gg 2$ was proposed in \citep{el2016asymptotic} as a method for improving Laplace learning in the setting of very few labeled data points, and analysis of the $p$-Laplacian was carried out in \citet{slepcev2019analysis,calder2018game}. The minimizer of the $p$-variational problem \eqref{eq:plaplace_variational} solves the graph $p$-Laplace equation
\begin{equation}\label{eq:}
\left\{
\begin{aligned}
\div (\|\nabla u\|^{p-2} \nabla u)(x) &= 0,&& \text{if } x\in \U \\
u(x) &= g(x),&& \text{if } x\in \L.
\end{aligned}
\right.
\end{equation}
Another form of the $p$-Laplacian---the game theoretic $p$-Laplacian---was proposed and studied in \citep{calder2018game,calder2024consistency}.

\subsection{Poisson learning}

Another recently proposed approach for low label rate graph-based semi-supervised learning is \emph{Poisson learning} \citep{calder2020poisson}. Instead of solving the graph Laplace equation \eqref{eq:graph_laplace}, Poisson learning solves the graph Poisson equation 
\[\Delta u = \sum_{x\in \L}(g - \bar{g})\delta_x\]
where $\bar{g} = \frac{1}{m}\sum_{x\in \L}g(x)$, $m$ is the number of labeled examples in $\L$, and $\delta_x(y)=0$ if $x\neq y$ and $\delta_x(x)=1$. In Poisson learning, the labels are encoded as point sources and sinks in a graph Poisson equation, which improves the performance at very low label rates \citep{calder2020poisson}. Poisson learning also has the variational interpretation
\begin{equation}\label{eq:poisson_variation}
\min_{u\in \l2}\left\{\nv{\nabla u}^2 - \sum_{x\in \L}(g(x) - \bar{g})\cdot u(x)\right\},
\end{equation}
which is similar to the soft constrained Laplace learning problem \eqref{eq:laplace_soft}, except that here we use a dot product fidelity instead of the $\l2$ fidelity.

% method
\section{Automatic differentiation through graph learning}\label{sec:method}

% \jeff{Add background on backpropgation through linear systems or convex optimization \cite{agrawal2019differentiable}}

% \jason{Do we have a subsection called overview or do we just do the overview under Method?}

In this section, we detail our main mathematical contributions, deriving the backpropagation equations necessary for incorporating graph-based learning algorithms into gradient-based learning settings, such as deep learning. Our results are general and apply to any nonlinear elliptic equation on a graph.

% The graph learning layer receives feature representations for a batch of data as well as one hot labels for the labeled data. Other than the batch size, additional hyperparameters include \(\tau\), the number of nearest neighbors $k$, and whether or not to fix \(d_k\) as a constant. 

We assume the following generic pipeline. In the forward pass, the graph learning layer receives feature representations (i.e. the output of a neural network-based encoder) for a batch of data. Then, a k-nearest neighbors search is conducted over the features and we construct a sparse weight matrix and corresponding graph Laplacian, as described in Section \ref{sec:graph_construction}. We then solve a graph Laplace equation, such as Equation \eqref{eq:graph_laplace}, to make label predictions at a set of unlabeled nodes. 

In the backward pass, the GLL receives the upstream gradients coming from a loss function and need to return the gradients of the loss with respect to the neural network's feature representations. The calculations for backpropagation can be divided into two distinct steps: tracking the gradients through the graph Laplace equation with respect to the entries in the weight matrix, which is addressed in Section \ref{sec:lapback}, and then tracking the gradients from the weight matrix to the feature vectors, which is addressed in Section \ref{sec:weightback}. In practice, we combine the implementation of these two terms within the same autodifferentiation function, which allows for more efficient calculations.

\subsection{Backpropagation through graph Laplace equations}\label{sec:lapback}

\blue 

In this section, we present the derivation of the backpropagation equations for computing the gradient of the solution of a graph Laplace equation for label propagation with respect to the ingredients of the equation, including the weight matrix used to construct the graph, and any source terms or boundary conditions that are present in the equation. For simplicity, we restrict our attention to binary classification, since the multi-class setting of any graph learning algorithm is equivalent to the one-vs-rest approach. 

\subsubsection{Linear graph Laplace equations}

In order to elucidate our approach in a simpler setting, we will first present our main results for the linear graph Laplace equation 
\begin{equation}\label{eq:simple_laplace}
\left\{
\begin{aligned}
\tau u(x) + \Delta u(x) &= f(x),&& \text{if } x\in \U \\
u(x) &= g(x),&& \text{if } x\in \L.
\end{aligned}
\right.
\end{equation}
Here, $\tau\geq 0$ is a parameter, $\Delta$ is the graph Laplacian, defined in \eqref{eq:graph_Laplacian}, $g$ is the boundary condition, and $f$ is the source term. Depending on the context, the label information can be encoded into the source term $f$ or the boundary condition $g$.

As discussed previously, we have in mind that a graph PDE like \eqref{eq:simple_laplace} is a layer 
%\jason{should this be layer or are we not referring to it as the GLL yet in the background?} \jeff{Yes, layer is good.}
within a neural network that is trained end to end. The inputs to the graph learning layer are any one of the following objects (or all of them):
\begin{itemize}
\item The weight matrix entries $w_{xy}$.
\item The source term $f$.
\item The boundary condition $g$. 
\end{itemize}
In all cases, the output of the graph learning layer is the solution $u\in \l2$, which feeds into the next layer of the network. Let us keep in mind that, while we have used functional analysis notation to illuminate the main ideas more clearly, the inputs and outputs of the block are all vectors in Euclidean space; i.e., a function $u\in \l2$ is simply a vector in $\R^n$ where $n$ is the number of nodes in the graph.  %removed 'can be identifed as a vector in Rn' 

The graph learning layer is part of an end-to-end network that is fed into a scalar loss function that we denote by $\Loss$. For $u\in \l2$ the solution of \eqref{eq:simple_laplace}, let $\nabla_u\Loss\in \l2$ be the gradient of the loss $\Loss$ with respect to the output $u$ of the graph learning layer, so 
\[\nabla_u \Loss(x) = \partial_{u(x)} \Loss,\]
where $\partial_z$ denotes a scalar partial derivative in the scalar variable $z$.  Backpropagating through a graph learning layer requires taking as input the gradient $\nabla_u \Loss$ and computing, via the chain rule, the gradients
\begin{align}
\nabla_{W} \Loss(x,y) &= \ipg{\partial_{w_{xy}}u,\nabla_u \Loss} \label{eq:gradw_pre} \\
\nabla_{f} \Loss(x) &= \ipg{\partial_{f(x)}u,\nabla_u \Loss} \label{eq:gradf_pre},\ \ \text{and} \\
\nabla_{g} \Loss(x) &= \ipg{\partial_{g(x)}u,\nabla_u \Loss}\label{eq:gradg_pre} ,
\end{align}
for all $x,y\in \X$. The quantities $\nabla_W \Loss$, $\nabla_f \Loss$, and $\nabla_g \Loss$ are the outputs of the graph learning autodifferentiation procedure and are then fed into the preceding autodifferentiation block. Letting $m$ be the number of labeled points, $\nabla_W \Loss$ is an $n\times n$ matrix, while $\nabla_f \Loss$ is a length $n-m$ vector and $\nabla_g \Loss$ is a length $m$ vector.

In the special case of the linear graph Laplace equation \eqref{eq:simple_laplace}, backpropgation through the graph learning layer requires solving the \emph{adjoint} equation
\begin{equation}\label{eq:simple_laplace_adjoint}
\left\{
\begin{aligned}
\tau v(x) + \Delta v(x) &= \nabla_u \Loss,&& \text{if } x\in \U \\
v(x) &= 0,&& \text{if } x\in \L
\end{aligned}
\right.
\end{equation}
for the unknown function $v\in \l2$. Since the left hand side of \eqref{eq:simple_laplace} is self-adjoint, the adjoint equation \eqref{eq:simple_laplace_adjoint} is nearly identical to the original equation \eqref{eq:simple_laplace} --- indeed, the main difference is that the boundary conditions have been set to zero, and the source term is replaced by the upstream gradient $\nabla_u \Loss$, which is the input to the automatic differentiation method. Hence, the computational complexity of the adjoint equation is similar to the original graph Laplace equation that was solved in the forward pass.

In this case, the adjoint equation \eqref{eq:simple_laplace_adjoint} is uniquely solvable for $v$, provided the graph is connected, and the downstream gradients we wish to compute can all be expressed in terms of $v$ as follows (see Theorem \ref{thm:backpropagation}):
\begin{align}
\nabla_{W} \Loss(x,y) &= -\nabla u(x,y)v(x) =  -(u(x) - u(y))v(x) \ \ \text{for} \ x,y\in \X, \\
\nabla_{f} \Loss(x) &= v(x)\ \ \text{for} \ x\in \X\setminus \L \\
\nabla_{g} \Loss(x) &= \nabla_u \Loss(x) -  \Delta v(x) \ \ \text{for} \ x\in \L.
\end{align}
We note that the last expression for $\nabla_{g} \Loss(x)$ cannot be simplified using the adjoint equation, since it is evaluated on the labeled nodes $\L$ where the adjoint equation does not hold. We also note that it may seem a little strange at first that $\nabla_{W} \Loss(x,y)$ is not symmetric in $x$ and $y$. However, as we show in the next section, since the weight matrix $W$ is symmetric, we can replace $\nabla_{W} \Loss(x,y)$ by the symmetrized expression
\[ -\nabla u(x,y)\nabla v(x,y) =  -(u(x) - u(y))(v(x)-v(y))\]
without affecting any downstream gradient computations. We use this symmetrized expression in practice to simplify computations.

The proofs of the claims in this section are all provided in a more general setting in Theorem \ref{thm:backpropagation} in the following section.

\nc 

\subsubsection{Nonlinear elliptic equations on graphs}

\blue
In order to make our results general - applicalbe to any of the graph-based learning algorithms discussed in Section \ref{sec:background}, including nonlinear graph-Laplace equations - we proceed in this section with generalizing the simple graph Laplace equation \eqref{eq:simple_laplace} to a general nonlinear elliptic equation on a graph. \nc Given a function $\phi:\R\times \X^2\to\R$ and a vector field $\v\in \lx2$, we define the function $\phi(\v,\cdot)\in \lx2$ by
\[\phi(\v,\cdot)(x,y)  = \phi(\v(x,y),x,y).\]
Let us write $\phi=\phi(q,x,y)$ and write $\phi_q$ for the partial derivative of $\phi$ in $q$, when it exists. In general, $\phi(\v,\cdot)$ may not be a vector field, without further assumptions on $\phi$. 
\begin{definition}\label{def:phi_vector}
We say $\phi:\R\times \X^2\to \R$ \emph{preserves vector fields} if
\[\phi(-q,x,y) = -\phi(q,y,x)\]
holds for all $q\in \R$ and $x,y\in \X$.
\end{definition}
Clearly, if $\phi$ preserves vector fields, then for any vector field $\v$, $\phi(\v,\cdot)$ is also a vector field. 

We also define ellipticity and symmetry.
\begin{definition}\label{def:elliptic}
We say that $\phi:\R\times \X^2\to \R$ is \emph{elliptic} if $q\mapsto \phi(q,x,y)$ is monotonically nondecreasing for all $x,y\in \X$. 
\end{definition}
\begin{definition}\label{def:symmetric}
We say that $\phi:\R\times \X^2\to \R$ is \emph{symmetric} if  
\[\phi_q(q,x,y) = \phi_q(-q,y,x)\]
holds for all $q\in \R$ and $x,y\in \X$ for which the derivatives $\phi_q(q,x,y)$ and $\phi_q(-q,y,x)$ exist.
\end{definition}
The notion of ellipticity is also called monotonicity in some references \citep{barles1991convergence,calder2022hamilton}, and has appeared previously in the study of graph partial differential equations \citep{manfredi2015nonlinear,calder2018game}.

We will consider here a nonlinear graph Laplace-type equation of the form
\begin{equation}\label{eq:elliptic_pde}
\left\{
\begin{aligned}
\zrd(u(x),x) + \div \phi(\nabla u,\cdot)(x) &= f(x),&& \text{if } x\in \X\setminus \L \\
u(x) &= g(x),&& \text{if } x\in \L,
\end{aligned}
\right.
\end{equation}
where $\zrd:\R\times \X\to\R$, $f\in \l2$, $\L\subset \X$, and $g:\X\to \R$. The graph Laplace equation generalizes all of the examples of graph Laplace-based semi-supervised learning algorithms given in Section \ref{sec:background}. In particular, we allow for the situation that $\L=\varnothing$. We will assume the equation \eqref{eq:elliptic_pde} is uniquely solvable; see \cite{calder2022hamilton} for conditions under which this holds. However, at the moment we do not place any assumptions on $\zrd$ and $\phi$; in particular, $\phi$ need not be elliptic, preserve vector fields, or be symmetric. 

\begin{example}\label{Ex:p_laplace}
An example is the graph $p$-Laplace equation for $p\geq 1$, where $\zrd(z,x)=\tu z$ for a nonnegative constant $\tu$, $f\in \l2$, and
\[\phi(q,x,y) = |q|^{p-2}q.\]
In this case, the equation in \eqref{eq:elliptic_pde} becomes
\[\tu u(x) + \div(|\nabla u|^{p-2}\nabla u)(x) = f(x)\]
or written out more explicitly 
\[\tu u(x) + \sum_{y\in \X}w_{xy}|u(x)-u(y)|^{p-2}(u(x)-u(y)) = f(x).\]
Since $\phi_q(q,x,y) = (p-1)|q|^{p-2}$, we have that $\phi$ preserves vector fields, and is both elliptic and symmetric for $p\geq 1$. Note that $p=2$ recovers {\normalfont Laplace learning}, which will be the focus of our experimental Section \ref{sec:experiments}.
\end{example}
\begin{remark}\label{rem:variational}
If $\phi$ is elliptic, then the divergence part of \eqref{eq:elliptic_pde} arises from the minimization of a convex energy on the graph of the form
\[\sum_{x,y\in \X}w_{xy}\psi(u(x)-u(y),x,y),\]
where $\psi$ is any antiderivative of $\phi$ in the $q$ variable. The ellipticity guarantees that $\psi_q$ is nondecreasing, and so $\psi$ is convex in $q$.  For example, with the usual graph Laplacian, $\phi(q) = \frac{1}{2}q$ and $\psi(q) = q^2$. 
\end{remark}

As we show below, efficient computation of the gradients \eqref{eq:gradw_pre}, \eqref{eq:gradf_pre}, and \eqref{eq:gradg_pre}---the essence of backpropagation---involves solving the \emph{adjoint} equation
\begin{equation}\label{eq:adjoint_pde}
\left\{
\begin{aligned}
\zrd_z(u(x),x)v(x) + \div\left( \phi_q(\nabla u,\cdot)\nabla v\right) &= \nabla_u \Loss(x) ,&& \text{if } x\in \X\setminus \L \\
v(x) &= 0,&& \text{if } x \in \L
\end{aligned}
\right.
\end{equation}
for the unknown function $v\in \l2$.  \blue Recall that $\zrd_z$ denotes the partial derivative of the function $(z,x) \mapsto \zrd(z,x)$ in the first variable $z$, i.e., if $\zrd(z,x) = z^2$ then $\zrd_z(z,x) = 2z$. We also note that using the definition of divergence given in \eqref{eq:div_formula}, the divergence term $\div\left( \phi_q(\nabla u,\cdot)\nabla v\right)$ evaluated at a vertex $x\in \X\setminus \L$ in \eqref{eq:adjoint_pde} can be expressed as
\begin{equation}\label{eq:divout}
\frac{1}{2} \sum_{y\in \X}w_{xy}(\phi_q(\nabla u(x,y),x,y)(v(x) - v(y)) - \phi_q(\nabla u(y,x),y,x)(v(y) - v(x))).
\end{equation}
If $\phi$ is symmetric, as per Definition \ref{def:symmetric}, then $\phi_q(-q,y,x) = \phi_q(q,x,y)$ and this expression can be simplified to read
\begin{equation}\label{eq:divout_simplified}
\div\left( \phi_q(\nabla u,\cdot)\nabla v\right) = \sum_{y\in \X}w_{xy}\,\phi_q(\nabla u(x,y),x,y)(v(x) - v(y)).
\end{equation}
This expression is a graph Laplacian applied to the function $v$, with weights that depend on the gradient of $u$. In fact, the first term $\zrd_z(u(x),x)v(x)$ in the adjoint equation \eqref{eq:adjoint_pde} also depends \nc on the solution $u$ of \eqref{eq:elliptic_pde}, which was computed in the forward pass of the neural network and can be saved for use in backpropagation. We also note that, even though the graph PDE \eqref{eq:elliptic_pde} is nonlinear, the adjoint equation \eqref{eq:adjoint_pde} is always a linear equation for $v$.

The following theorem is our main result for backpropagation through graph learning layers. 
\begin{theorem}\label{thm:backpropagation}
Assume $z\mapsto \zrd(z,x)$ and $q\mapsto \phi(q,x,y)$ are continuously differentiable for all $x,y\in \X$. Let $f\in \l2$, $g:\L\to \R$, and suppose $u\in \l2$ satisfies \eqref{eq:elliptic_pde}.  If the adjoint equation \eqref{eq:adjoint_pde} admits a unique solution $v\in \l2$, then the partial derivatives $\partial_{w_{xy}}u,\partial_{f(x)}u, \partial_{g(x)}u\in \l2$ exist and are continuous in the variables $w_{xy}, f(x)$ and $g(x)$, respectively. Furthermore, it holds that
\begin{align}
\nabla_{W} \Loss(x,y) &= -\frac{1}{2}\Big(\phi(\nabla u(x,y),x,y) - \phi(\nabla u(y,x),y,x)\Big)v(x) \ \ \text{for} \ x,y\in \X, \label{eq:gradW}\\
\nabla_{f} \Loss(x) &= v(x)\ \ \text{for} \ x\in \X\setminus \L\label{eq:gradf} \\
\nabla_{g} \Loss(x) &= \nabla_u \Loss(x) -  \div\left( \phi_q(\nabla u,\cdot)\nabla v\right)(x) \ \ \text{for} \ x\in \L \label{eq:gradg}.
\end{align}
\end{theorem}
\begin{proof}
The existence and continuity of the partial derivatives $\partial_{w_{xy}}u,\partial_{f(x)}u, \partial_{g(x)}u\in \l2$ follows from the unique solvability of the adjoint equation \eqref{eq:adjoint_pde} and the implicit function theorem. 

Now, the adjoint equation is utilized throughout the whole proof in a similar way. Given $h\in \l2$ with $h(x)=0$ for $x\in \L$, we can compute the inner product with $\nabla_u \Loss$ as follows:
\begin{equation}\label{eq:adjoint_formula}
\ipg{h,\nabla_u \Loss}=\ipg{\zrd_z(u,\cdot)h +\div\left(\phi_q(\nabla u,\cdot)\nabla h\right),v}.
\end{equation}
To see this we compute
\begin{align*}
\ipg{h,\nabla_u \Loss}&=\ipg{h,\zrd_z(u,\cdot)v + \div\left( \phi_q(\nabla u,\cdot)\nabla v\right)}\notag \\
&=\ipg{h,\zrd_z(u,\cdot)v}  + \ipg{h,\div\left( \phi_q(\nabla u,\cdot)\nabla v\right)}\notag \\
&=\ipg{\zrd_z(u,\cdot)h,v}  + \ipv{\nabla h,\phi_q(\nabla u,\cdot)\nabla v}\notag \\
&=\ipg{\zrd_z(u,\cdot)h,v}  + \ipv{\phi_q(\nabla u,\cdot)\nabla h,\nabla v}\notag \\
&=\ipg{\zrd_z(u,\cdot)h,v}  + \ipg{\div\left(\phi_q(\nabla u,\cdot)\nabla h\right),v}\notag \\
&=\ipg{\zrd_z(u,\cdot)h +\div\left(\phi_q(\nabla u,\cdot)\nabla h\right),v}.
\end{align*}
We will use the identity \eqref{eq:adjoint_formula} multiple times in what follows. 

We now prove \eqref{eq:gradW}. Given $r,s\in \X$, we let $h(x) = \partial_{w_{rs}}u(x)$. Then differentiating both sides of \eqref{eq:elliptic_pde} we find that 
\begin{equation}\label{eq:diff_wxy}
\zrd_z(u(x),x)h(x) + \div \left( \phi_q(\nabla u,\cdot)\nabla h\right) = -\frac{1}{2}\delta_r(x)(\phi(\nabla u(r,s),r,s) - \phi(\nabla u(s,r),s,r))
\end{equation}
for $x\in \X\setminus \L$, and $v(x)=0$ for $x\in \L$, where $\delta_r(x)=0$ if $r\neq x$, and $\delta_r(r)=1$. By \eqref{eq:gradw_pre}, \eqref{eq:adjoint_formula}, and \eqref{eq:diff_wxy} we have
\begin{align*}
\nabla_{W} \Loss(r,s) &= \ipg{\partial_{w_{rs}}u,\nabla_u \Loss}=\ipg{h,\nabla_u \Loss}\\
&=\ipg{\zrd_z(u,\cdot)h +\div\left(\phi_q(\nabla u,\cdot)\nabla h\right),v}\\
&=-\frac{1}{2}\ipg{\delta_r(\phi(\nabla u(r,s),r,s) - \phi(\nabla u(s,r),s,r)),v}\\
&=-\frac{1}{2}(\phi(\nabla u(r,s),r,s) - \phi(\nabla u(s,r),s,r))v(r).
\end{align*}

To prove \eqref{eq:gradf}, let $r\in \X\setminus \L$ and set $h(x) = \partial_{f(r)}u(x)$. Then differentiating both sides of \eqref{eq:elliptic_pde} we find that 
\begin{equation}\label{eq:diff_fx}
\zrd_z(u(x),x)h(x) + \div \left( \phi_q(\nabla u,\cdot)\nabla h\right) = \delta_r(x),
\end{equation}
for $x\in \X\setminus \L$, and $h(x)=0$ for $x\in \L$. By \eqref{eq:gradf_pre}, \eqref{eq:adjoint_formula}, and \eqref{eq:diff_fx} we have
\begin{align*}
\nabla_{f} \Loss(r) &= \ipg{h,\nabla_u \Loss}=\ipg{\zrd_z(u,\cdot)h +\div\left(\phi_q(\nabla u,\cdot)\nabla h\right),v}=\ipg{\delta_r,v} = v(r).
\end{align*}

Finally, to prove \eqref{eq:gradg}, let $r\in \L$ and set $h(x) = \partial_{g(r)}u(x)$. 
Then differentiating both sides of \eqref{eq:elliptic_pde} we find that 
\begin{equation}\label{eq:diff_gx}
\zrd_z(u(x),x)h(x) + \div \left( \phi_q(\nabla u,\cdot)\nabla h\right) = 0,
\end{equation}
for $x\in \X\setminus \L$ and $h(x)=\delta_r(x)$ for $x\in \L$. By \eqref{eq:gradg_pre}, \eqref{eq:adjoint_formula}, and \eqref{eq:diff_gx} we have
\begin{align*}
\nabla_{g} \Loss(r) &= \ipg{h,\nabla_u \Loss} = \ipg{h - \delta_r,\nabla_u \Loss} + \nabla_u \Loss(r)\\
&=\ipg{h - \delta_r,\zrd_z(u,\cdot)v + \div\left( \phi_q(\nabla u,\cdot)\nabla v\right)} + \nabla_u \Loss(r)\\
&=\ipg{h,\zrd_z(u,\cdot)v + \div\left( \phi_q(\nabla u,\cdot)\nabla v\right)} - \zrd_z(u(r),r)v(r) \\
&\hspace{2in}- \div\left( \phi_q(\nabla u,\cdot)\nabla v\right)(r) + \nabla_u \Loss(r)\\
&=\ipg{\zrd_z(u,\cdot)h + \div\left( \phi_q(\nabla u,\cdot)\nabla h\right),v} - \zrd_z(u(r),r)v(r) \\
&\hspace{2in}- \div\left( \phi_q(\nabla u,\cdot)\nabla v\right)(r) + \nabla_u \Loss(r)\\
&=-\zrd_z(u(r),r)v(r) -  \div\left( \phi_q(\nabla u,\cdot)\nabla v\right)(r) + \nabla_u \Loss(r).
\end{align*}
We finally use that $v(r)=0$ for $r\in \L$.
\end{proof}

\vspace{-0.5cm}
Recall that we have assumed the weight matrix $W$ is symmetric, so $w_{xy}=w_{yx}$. Thus, it may at first seem strange that the gradient $\nabla_W \Loss(x,y)$ is not symmetric in $x$ and $y$, but this is a natural consequence of the lack of any restrictions on the variations in $W$ when computing gradients. However, it turns out that we can symmetrize the expression for the gradient $\nabla_W \Loss(x,y)$ without affecting the backpropagation of gradients, provided that $W$ is guaranteed to be symmetric under perturbations in any variables that it depends upon.  This is encapsulated in the following lemma. 
\begin{lemma}\label{lem:symmetric_gradW}
Suppose that the weight matrix $W = (w_{xy})_{x,y\in \X}$ depends on a parameter vector $\theta \in \R^m$, so $W = W(\theta)$ and $w_{xy}=w_{xy}(\theta)$ for all $x,y\in \X$, and furthermore that the loss $\Loss$ depends on $\theta$ only through the weight matrix $W$. If $w_{xy}(\theta)=w_{yx}(\theta)$ for all $\theta \in \R^m$ and $x,y\in \X$ then it holds that
\begin{equation}\label{eq:symmetric_backpropagate}
\nabla_\theta \Loss = \sum_{x,y\in \X}\nabla_W \Loss(x,y) \nabla_\theta w_{xy} = \sum_{x,y\in \X}\frac{1}{2}\left(\nabla_W \Loss(x,y) + \nabla_W \Loss(y,x)\right) \nabla_\theta w_{xy}.
\end{equation}
\end{lemma}
\begin{proof}
By the chain rule we have
\[\nabla_\theta \Loss = \sum_{x,y\in \X}\nabla_W \Loss(x,y) \nabla_\theta w_{xy} = \sum_{x,y\in \X}\frac{1}{2}\left(\nabla_W \Loss(x,y)\nabla_\theta w_{xy} + \nabla_W \Loss(y,x)\nabla_\theta w_{yx}\right).\]
Since $w_{xy}(\theta)=w_{yx}(\theta)$ we have $\nabla_\theta w_{xy}=\nabla_\theta w_{yx}$, which upon substituting above completes the proof.
\end{proof}
Lemma \ref{lem:symmetric_gradW} allows us to replace the gradient $\nabla_W\Psi$ with a symmetrized version without affecting the correct computation of gradients. 
\begin{remark}\label{rem:symmetric_gradW}
By Lemma \ref{lem:symmetric_gradW}, we may replace $\nabla_W \Loss(x,y)$, given by \eqref{eq:gradW} with its symmetrized version 
\begin{equation}\label{eq:symmetrized_gradW}
\frac{1}{2}\left(\nabla_W \Loss(x,y) + \nabla_W \Loss(y,x)\right) = -\frac{1}{2}(\phi(\nabla u(x,y),x,y) - \phi(\nabla u(y,x),y,x))\nabla v(x,y), 
\end{equation}
without affecting the resultant backpropagation. The resulting expression is symmetric, in $x$ and $y$, which may be convenient in implementations. 
\end{remark}
\begin{example}\label{Ex:pLaplace_equations}
As in Example \ref{Ex:p_laplace}, we consider the graph $p$-Laplace equation with $p>1$,\footnote{The case $p=1$ does not satisfy the continuous differentiability conditions of Theorem \ref{thm:backpropagation}.} $\zrd(z,x)=\tu z$, where $\tu$ a nonnegative constant, and $f\in \l2$. In this case we have $\phi(q,x,y) = |q|^{p-2}q$ and $\phi_q(q,x,y) = (p-1)|q|^{p-2}$. The adjoint equation \eqref{eq:adjoint_pde} becomes 
\[\tu v + (p-1)\div(|\nabla u|^{p-2}\nabla v) = \nabla_u \Loss,\]
subject to $v=0$ on $\L$. The adjoint equation is uniquely solvable if $\tu>0$; otherwise the solvability depends crucially on the structure of $|\nabla u|^{p-2}$ when $p\neq 2$. For example, if $u$ is constant, so $\nabla u=0$, then the adjoint equation is \emph{not} solvable when $\tu = 0$. When $p=2$ and $\L\neq \varnothing$ (this is the setting of {\normalfont Laplace learning}) the adjoint equation is always uniquely solvable, and is simply given by the Laplace equation
\[\tu v + \Delta v = \nabla_u \Psi,\]
subject to $v=0$ on $\L$. 

Proceeding under the assumption that the adjoint equation is uniquely solvable, the expressions for the gradients in Theorem \ref{thm:backpropagation}, taking the symmetrized version for $\nabla_W \Loss$, become $\nabla_f \Loss(x)=v(x)$ for $x\in \X\setminus \L$,
\begin{equation}\label{eq:grad_w_plaplace}
\frac{1}{2}\left(\nabla_W \Loss(x,y) + \nabla_W \Loss(y,x)\right) = -|\nabla u(x,y)|^{p-2}\nabla u(x,y) \nabla v(x,y),
\end{equation}
for $x,y\in \X$, and
\begin{equation}\label{eq:grad_g_plaplace}
\nabla_{g} \Loss(x) = \nabla_u \Loss(x) - (p-1) \div\left( |\nabla u|^{p-2}\nabla v\right)(x) \ \ \text{for} \ x\in \L.
\end{equation}
When $p=2$, the symmetrized gradient $\nabla_W \Psi$ for Laplace learning has the especially simple form $-\nabla u(x,y) \nabla v(x,y)$. 
\end{example}

\subsection{Backpropagation through similarity matrices}\label{sec:weightback}

We now consider the backpropagation of gradients through the construction of the graph weight matrix. 

%While many of these gradients can be tracked by any standard software package that supports automatic differentiation, such as PyTorch (\cite{paszke2017automatic}), there are some subtleties with how we handle the adjacency structure for sparse matrices. We also find that the computations are far more efficient if the backpropagation through graph learning and the weight matrix construction are included in the same automatic differentiation block, which is largely due to implementation issues concerning sparse matrices in PyTorch.  Thus, in this section, we carefully compute the backpropgation equations for our construction of a similarity graph from feature vectors. 

As in Section \ref{sec:graph_construction}, we consider a symmetric k-nearest neighbor type graph construction for the weight matrix $W$, of the form
\begin{equation}\label{eq:Wij}
w_{ij} = a_{ij}\eta\left( \frac{\|\x_i -\x_j\|^2}{2\eps_i\eps_j}\right),
\end{equation}
where $\eta$ is differentiable and $a_{ij}$ is the adjacency matrix of the graph, so $a_{ij}=1$ if there is an edge between nodes $i$ and $j$, and $a_{ij}=0$ otherwise. We consider graphs without self-loops, so $w_{ii} = a_{ii} = 0$ for all $i$.

In deriving the backpropagation equations, we will assume that the adjacency matrix $a_{ij}$ does not depend on the features $\x_i$. This is an approximation that is true generically for $k$-nearest neighbor graphs, unless there is a tie for the kth neighbor of some $\x_i$ that allows the neighbor information to change under small perturbation. However, the self-tuning parameters $\eps_i$ and $\eps_j$ do depend on the feature vectors; in particular, they are defined by 
\begin{equation}\label{eq:epsi}
\eps_i = \|\x_i - \x_{k_i}\|,
\end{equation}
where $k_i$ is the index of the $k^{\rm th}$ nearest neighbor of the vertex $i$.  However, we note that the derivation below does not place any assumptions on the indices $k_1,\dots,k_n$ except that $k_i\neq i$ and $a_{ik_i}=1$. This, in particular, allows our analysis to handle approximate nearest neighbor searches, where some neighbors may be incorrect, as long as we use the correct distances. We will also assume that the indices $k_i$ do not depend on (or rather, are constant with respect to) any of the features $\x_j$, which is again true under small perturbations of the features, as long as there is no tie for the \kth { }nearest neighbor of any given point. This is generically true for e.g. random data on the unit sphere in $\mathbb{R}^d$.

In backpropagation, we are given the gradients $g_{ij} = \partial_{w_{ij}}\Loss$ computed in Section \ref{sec:lapback}, which we can assemble into a matrix $G=(g_{ij})_{ij}$, and we need to compute $\nabla_{\x_\ell}\mathcal{J}$ in terms of the matrix $G$, where $\x_\ell$ is an arbitrary feature vector. That is to say, we backpropagate the gradients through the graph construction \eqref{eq:Wij}.  Our main result is the following lemma.
\begin{lemma}\label{lem:backw}
We have
\begin{equation}\label{eq:final_grad1}
\nabla_{\x_i}\mathcal{J} = \sum_{j=1}^n\frac{\tilde g_{i j}a_{i j}\vv_{i j}}{\eps_i\eps_j}(\x_i - \x_j) + \sum_{j=1}^n b_j(\x_j - \x_{k_j})\delta_{i k_j} - b_i(\x_i - \x_{k_i}),
\end{equation}
where $\tilde g_{ij} = g_{ij} + g_{ji}$, and the vector $b=(b_i)_i$ and matrix $\V=(\vv_{ij})_{ij}$ are defined by
\begin{equation}\label{eq:bi}
b_i = \sum_{j=1}^n\frac{\tilde g_{ij}a_{ij}\vv_{ij}}{2\eps_i^3\eps_j}\|\x_i-\x_j\|^2, \ \ \text{and} \ \ \vv_{ij} = \eta'\left(\frac{\|\x_i -\x_j\|^2}{2\eps_i\eps_j} \right),
\end{equation}
\blue where $\eta'$ is the derivative of $\eta$. \nc
\end{lemma}
\begin{remark}\label{rem:sparsity}
Notice in \eqref{eq:final_grad1} and \eqref{eq:bi}, the gradients $\tilde{g}_{ij}$ always appear next to the corresponding adjacency matrix entry $a_{ij}$. Thus, in an end to end graph learning block, where the features vectors $\x_i$ are taken as input, we only need to compute the gradients $g_{ij}$ along \emph{edges} in the graph. This allows us to preserve the sparsity structure of the graph during backpropagation, drastically accelerating computations. 
\end{remark}
\begin{proof}
We begin by using the chain rule to compute
\begin{equation}\label{eq:gradL}
\nabla_{\x_\ell}\mathcal{J} = \sum_{i,j=1}^n \frac{\partial \mathcal{J}}{\partial w_{ij}}\nabla_{\x_\ell}w_{ij} =\sum_{i,j=1}^n g_{ij}\nabla_{\x_\ell}w_{ij}=\sum_{i,j=1}^n g_{ij}a_{ij}\nabla_{\x_\ell}\eta\left( \frac{\|\x_i -\x_j\|^2}{2\eps_i\eps_j}\right).
\end{equation}
Notice that the sum above is only over edges in the graph where $a_{ij}=1$, and we may also omit $i=j$, since the gradient term is zero there. Thus, we need to compute
\[b^{\ell}_{ij}:=\nabla_{\x_\ell}\eta\left( \frac{\|\x_i -\x_j\|^2}{2\eps_i\eps_j}\right)\]
for $i\neq j$ with $a_{ij}=1$. Recalling \eqref{eq:bi} we have
\[b^{\ell}_{ij} = \frac{\vv_{ij}}{\eps_i\eps_j}(\x_i - \x_j)(\delta_{i\ell} - \delta_{j\ell}) - \frac{\vv_{ij}}{2\eps_i^2\eps_j^2}\|\x_i-\x_j\|^2 \nabla_{\x_\ell}(\eps_i\eps_j).\]
We now note that
\[\nabla_{\x_\ell}\eps_i = \frac{1}{\epsilon_i}(\x_i - \x_{k_i})(\delta_{\ell i} -  \delta_{\ell k_i}) \quad \text{and} \quad \nabla_{\x_\ell}\eps_j = \frac{1}{\epsilon_j}(\x_j - \x_{k_j})(\delta_{\ell j} -  \delta_{\ell k_j}), \]
so
\[\nabla_{\x_\ell}(\eps_i\eps_j) = \frac{\eps_i}{\eps_j}(\x_j - \x_{k_j})(\delta_{\ell j} -  \delta_{\ell k_j})  + \frac{\eps_j}{\eps_i} (\x_i - \x_{k_i})(\delta_{\ell i} -  \delta_{\ell k_i}).\]
Plugging this in above we have
\begin{align*}
b^{\ell}_{ij} = \frac{\vv_{ij}}{\eps_i\eps_j}(\x_i - \x_j)(\delta_{i\ell} - \delta_{j\ell}) &- \frac{\vv_{ij}}{2\eps_i\eps_j^3}\|\x_i-\x_j\|^2(\x_j - \x_{k_j})(\delta_{\ell j} -  \delta_{\ell k_j}) \\
& - \frac{\vv_{ij}}{2\eps_i^3\eps_j}\|\x_i-\x_j\|^2 (\x_i - \x_{k_i})(\delta_{\ell i} -  \delta_{\ell k_i}).
\end{align*}
Plugging this into \eqref{eq:gradL} we have
\begin{align*}
\nabla_{\x_\ell}\mathcal{J} &=\sum_{i,j=1}^n g_{ij}a_{ij}b^{\ell}_{ij}\\
&=\sum_{i,j=1}^ng_{ij}a_{ij}\frac{\vv_{ij}}{\eps_i\eps_j}(\x_i - \x_j)(\delta_{i\ell}- \delta_{j\ell})\\
&\hspace{1cm} - \sum_{i,j=1}^n g_{ij}a_{ij}\frac{\vv_{ij}}{2\eps_i\eps_j^3}\|\x_i-\x_j\|^2(\x_j - \x_{k_j})(\delta_{\ell j} -  \delta_{\ell k_j})  \\
&\hspace{1cm} - \sum_{i,j=1}^n g_{ij}a_{ij}\frac{\vv_{ij}}{2\eps_i^3\eps_j}\|\x_i-\x_j\|^2 (\x_i - \x_{k_i})(\delta_{\ell i} -  \delta_{\ell k_i}) \\
%&=\sum_{i=1}^n(g_{\ell i} + g_{i\ell})\frac{a_{\ell i}\vv_{\ell i}}{\eps_\ell\eps_i}(\x_\ell - \x_i)\\
%&\hspace{1cm} - \sum_{i=1}^n g_{i\ell}a_{i\ell}\frac{\vv_{i\ell}}{2\eps_i\eps_\ell^3}\|\x_i-\x_\ell\|^2(\x_\ell - \x_{k_\ell}) - \sum_{i,j=1}^n g_{ij}a_{ij}\frac{\vv_{ij}}{2\eps_i\eps_j^3}\|\x_i-\x_j\|^2(\x_\ell - \x_j)\delta_{\ell k_j}  \\
%&\hspace{1cm} - \sum_{j=1}^n g_{\ell j}a_{\ell j}\frac{\vv_{\ell j}}{2\eps_\ell ^3\eps_j}\|\x_\ell -\x_j\|^2 (\x_\ell  - \x_{k_\ell }) - \sum_{i,j=1}^n g_{ij}a_{ij}\frac{\vv_{ij}}{2\eps_i^3\eps_j}\|\x_i-\x_j\|^2 (\x_\ell - \x_i)\delta_{\ell k_i} \\
%&=\sum_{i=1}^n(g_{\ell i} + g_{i\ell})\frac{a_{\ell i}\vv_{\ell i}}{\eps_\ell\eps_i}(\x_\ell - \x_i)\\
%&\hspace{1cm} - \sum_{i=1}^n g_{i\ell}a_{\ell i}\frac{\vv_{\ell i}}{2\eps_\ell^3\eps_i}\|\x_\ell-\x_i\|^2(\x_\ell - \x_{k_\ell}) - \sum_{i,j=1}^n g_{ji}a_{ij}\frac{\vv_{ij}}{2\eps_j\eps_i^3}\|\x_i-\x_j\|^2(\x_\ell - \x_i)\delta_{\ell k_i}  \\
%&\hspace{1cm} - \sum_{i=1}^n g_{\ell i}a_{\ell i}\frac{\vv_{\ell i}}{2\eps_\ell ^3\eps_i}\|\x_\ell -\x_i\|^2 (\x_\ell  - \x_{k_\ell }) - \sum_{i,j=1}^n g_{ij}a_{ij}\frac{\vv_{ij}}{2\eps_i^3\eps_j}\|\x_i-\x_j\|^2 (\x_\ell - \x_i)\delta_{\ell k_i} \\
&=\sum_{i=1}^n\frac{\tilde g_{\ell i}a_{\ell i}\vv_{\ell i}}{\eps_\ell\eps_i}(\x_\ell - \x_i) - \left[\sum_{i=1}^n \frac{\tilde g_{\ell i}a_{\ell i}\vv_{\ell i}}{2\eps_\ell^3\eps_i}\|\x_\ell-\x_i\|^2\right](\x_\ell - \x_{k_\ell}) \\
&\hspace{2.5in}- \sum_{i=1}^n \left[\sum_{j=1}^n\frac{\tilde g_{ij}a_{ij}\vv_{ij}}{2\eps_i^3\eps_j}\|\x_i-\x_j\|^2\right](\x_{k_i} - \x_i)\delta_{\ell k_i},  
\end{align*}
where $\tilde g_{ij} = g_{ij} + g_{ji}$. Recalling \eqref{eq:bi}, the gradient above can be expressed as
\begin{equation}\label{eq:final_grad}
\nabla_{\x_\ell}\mathcal{J} = \sum_{i=1}^n\frac{\tilde g_{\ell i}a_{\ell i}\vv_{\ell i}}{\eps_\ell\eps_i}(\x_\ell - \x_i) + \sum_{i=1}^n b_i(\x_i - \x_{k_i})\delta_{\ell k_i} - b_\ell(\x_\ell - \x_{k_\ell}).
\end{equation}
The proof is completed by relabelling indices.
\end{proof}

\begin{remark}\label{rem:grad_form}
For computational purposes, we can further simplify \eqref{eq:final_grad1} by encoding the $k$-nearest neighbor information as follows. Let $C=(c_{ij})_{ij}$ be the $n\times n$ matrix with $c_{ij}=1$ if node $i$ is the \kth nearest neighbor of node $j$, $c_{ii}=-1$, and $c_{ij}=0$ otherwise. Then the gradient given in the expression in Eq.~\eqref{eq:final_grad1} can be expressed as
\begin{equation}\label{eq:final_grad2}
\nabla_{\x_i}\mathcal{J} = \sum_{j=1}^n\frac{\tilde g_{i j}a_{i j}\vv_{i j}}{\eps_i\eps_j}(\x_i - \x_j) + \sum_{j=1}^n c_{ij}b_j(\x_j - \x_{k_j}).
\end{equation}
This is our final equation for the gradient of the loss with respect to the feature representations $\x_1,\dots,\x_n$; since these features are the output from a neural network, the remainder of backpropagation can be done as normal. 
\end{remark}

\subsubsection{Sparse and efficient gradient computations}\label{sec:final_grad2}

We discuss here how to efficiently compute the backpropagation gradient formula \eqref{eq:final_grad2} in a computationally efficient way in the setting of a sparse adjacency matrix $A$. 

We use the expression \eqref{eq:final_grad2} from Remark \ref{rem:grad_form}. The matrices $A$ (adjacency matrix), $\V$ (from Lemma \ref{lem:backw}), and $C$ (from Remark \ref{rem:grad_form}) require minimal additional computational overhead, as their structure and values all come from the $k$-NN search and weight matrix construction required in the forward pass to construct the weight matrix $W$. In our experiments, for example, we use $\eta(z) = \exp(-4z)$, so $\V = -4W$. The matrix $G$ is computed as described in Section \ref{sec:lapback}.

It is more efficient in implementations to express \eqref{eq:final_grad2} as sparse matrix multiplications. The first term already has this form; indeeed, if we define the matrix $H=(h_{ij})_{ij}$ by
\[h_{ij} = \frac{\tilde g_{i j}a_{i j}\vv_{i j}}{\eps_i\eps_j}\]
then the first term is given by
\[\sum_{j=1}^nH_{ij}(\x_i-\x_j)^T = L_H X\]
which is the multiplication of the sparse graph Laplacian matrix $L_H$ corresponding to the symmetric matrix $H$ with the data matrix $X$ whose rows are the feature vectors $\x_i^T$. 

Analogously, for the second term, define the matrix $M = (m_{ij})_{ij}$ as
\[ m_{ij} = c_{ij}b_j + c_{ji}b_i. \]
We then have that
\begin{align*}
\sum_{j=1}^n m_{ij}(\x_i-\x_j) &= \sum_j m_{ij} \x_i - \sum_j m_{ij}\x_j \\
&= \sum_j c_{ij}b_j(\x_i - \x_j) + \sum_j c_{ji}b_i(\x_i - \x_j).
\end{align*}
Now note that, for a fixed $i$, the column $c_{ji}$ ($j \in \{1,...,n\}$) is 1 when $\x_j$ is the kth nearest neighbor of $\x_i$. A node only has 1 kth nearest neighbor, so the column $c_{ji}$ is 0 everywhere except when $j = k_i$, $c_{k_ii} = 1$. The term $c_{ii}=-1$ is irrelevant, since $\x_i-\x_i=0$. Hence we have
\begin{align*}
\sum_{j=1}^n m_{ij}(\x_i-\x_j)&= \sum_j c_{ij}b_j(\x_i - \x_j) + b_i(\x_i - \x_{k_i})\\
&= \sum_j b_j(\x_{k_j}-\x_j)\delta_{ik_j} + b_i(\x_i - \x_{k_i})\\
&=-\sum_{j=1}^nc_{ij}b_j(\x_j - \x_{k_j}),
\end{align*}
where we used the definition of the matrix $C$ again. This is exactly the negative of the second term in \eqref{eq:final_grad2}. The expression on the left is again a graph Laplacian $L_M$ for the symmetric matrix $M$, so we can write 
\[\sum_{j=1}^nc_{ij}b_j(\x_j - \x_{k_j})^T = L_M X.\]
Finally, this allow us to express \eqref{eq:final_grad2} as
\begin{equation}\label{eq:final_grad_matrix}
\nabla_{X}\mathcal{J}^T = (L_H - L_M)X = L_{H-M}X.
%\nabla_{\x_i}\mathcal{J} = \sum_{j=1}^nH_{ij}(\x_i-\x_j) - \sum_{j=1}^n m_{ij}(\x_i-\x_j)
\end{equation}
Compared to a softmax classifier, the only additional computational overhead are sparse matrices, which can be efficiently stored and manipulated.

%https://www.overleaf.com/project/64b9a7c5217786728e5a9002
%If the matrix $W$ is sparse, with at most $k$ nonzero entries in each row, then the sum about only has $k$ nonzero terms, which correspond to the neighbors $j$ of node $i$. Thus, computing $\nabla_{\x_i}L$ should take only $O(k)$ operations for each $i$, so $O(nk)$ operations in total. Of course the complexity scales with the dimensionality of $\x_i$ as well.

%Your autograd can return $\nabla_{\x_i}L$ for $i=1,\dots,n$, instead of $\frac{\partial L}{\partial w_{ij}}$, and then all the sparse matrix stuff can go in the custom autograd routine, where tracking gradient is not needed. 

% experiments
\section{Implementation \& Numerical Experiments}\label{sec:experiments}

% {\blue
% Section~\ref{sec:implementation} introduces a practical implementation of our GLL to classification tasks. 
% }
% Afterwards, we showcase the strong performance of the GLL in a variety of tasks. In \ref{sec:toy}, we show the geometric effect of varying $\tau$ in Equation \eqref{eq:laplace_variational_tau} in a small neural network on a toy dataset. The remaining sections use the standard Laplace learning algorithm (Equation \eqref{eq:laplace_variational}). 
% Section \ref{sec:comparison_mlp_gll} compares the generalization, training dynamics \tcb{and computational efficiency} of neural networks equipped with MLP or GLL classification heads. 
% Section \ref{sec:large_experiments} presents a suite of experiments on the CIFAR-10 (\cite{krizhevsky2009learning}) and EMNIST (\cite{cohen2017emnist}) datasets at varying levels of supervision, comparing MLP, WNLL (\cite{wang2021graph}), and GLL classification heads with several different deep neural network backbones. GLL-based neural networks consistently meet or surpass the performance of MLP and WNLL-based networks, especially at low label rates.
% Section \ref{sec:adversarial} demonstrates the improved adversarial robustness of the GLL compared to MLP and WNLL across datasets, training strategies, and attack strengths. 
% While our theoretical results are for any graph Laplacian-based learning algorithm, we leave numerical experiments on methods like Poisson learning and $p$-Laplace learning for future work.

{\blue
This section details the practical use of the proposed GLL classifier and showcases its strong performance through a variety of numerical experiments. Section~\ref{sec:implementation} explains our implementation of the GLL as a classifier on supervised and semi-supervised tasks, including batching, contrastive pretraining, and data augmentation. Section~\ref{sec:toy} visualizes the embeddings learned by a GLL on a toy example and illustrates the effect of the Tikhonov parameter $\tau$ in~\eqref{eq:laplace_variational_tau}; all remaining experiments use the standard Laplace learning algorithm in~\eqref{eq:laplace_variational}. Section~\ref{sec:comparison_mlp_gll} considers a deliberately over-parameterized CNN on FashionMNIST to compare a GLL head to an MLP head in terms of generalization, training dynamics, and computational efficiency. Section~\ref{sec:large_experiments} presents large-scale experiments on the CIFAR-10 (\cite{krizhevsky2009learning}) and EMNIST (\cite{cohen2017emnist}) datasets at varying levels of supervision, comparing MLP, WNLL (\cite{wang2021graph}), and GLL classification heads across several deep neural network backbones; GLL-based networks consistently match or surpass the performance of MLP- and WNLL-based networks, especially at low label rates. Section~\ref{sec:training_inference_times} quantifies the computational overhead of GLL by reporting training and inference times as a function of batch size and network size. Finally, Section~\ref{sec:adversarial} provides our most significant empirical result: across datasets, training strategies, and attack strengths, a GLL yields consistently and substantially stronger robustness to adversarial attacks than both MLP- and WNLL-based classifiers. While our theoretical analysis applies to general graph Laplacian-based learning algorithms, we leave numerical experiments on variants such as Poisson learning and $p$-Laplace learning to future work.
}

% \jason{Need to fix these sentences to talk about the timing features. Maybe timing goes at the end and we switch 4.4 and 4.5}
% \harris{I added reference to timing details. I agree the timing stuff could possibly be moved to the end or to the appendix.}
% In our numerical experiments with our GLL, we use the diagonally perturbed graph Laplace equation \eqref{eq:laplace_variational_tau}, and we explore the role of the parameter $\tau$ in many of the results. 

{\blue
\subsection{Implementation}\label{sec:implementation}

% {\red
% Jason: Thinking we can relax the notation - would it be easier to just say that the batches use N samples from unlabeled set and M of each class sampled from the labeled set?
% Bohan: we need to clarify the base dataset, the test method, and the pretraining here. We assume the full training dataset includes labeled (L) and unlabeled data (U). }

In this subsection, we present implementation details for the GLL in our experiments. Although our GLL can be plugged into any neural network as a classifier, replacing the original head in this way leads to a batch setting that differs slightly from the classical MLP-based training paradigm. In addition, to accelerate training, we adopt a contrastive-learning–based pretraining stage, after which we perform supervised training for the neural network with the GLL.

\subsubsection{Sampling and Training on Batches}\label{sec:sampling_and_training_on_batches}
We describe a slight modification of batching needed for a GLL classifier. Recalling the notation of Section \ref{sec:classic_gl}, we let $\X$ denote the set of samples in the training set, \(\mathcal{L}\) be the subset of base samples with known labels, \(\mathcal{U}\) the unlabeled subset. Then we consider training on a batch $\mathcal{B}$ with size $B$ for the classic classifier (e.g. an MLP) or for our GLL classifier:

\paragraph{Classic Training: } Since each data is classified individually, the batch is selected only from the labeled dataset $\mathcal{L}$, i.e. $\mathcal{B}\subset \mathcal{L}$. The training loss $\Loss$ is computed between the predicted labels and their ground-truth labels over all samples in $\mathcal{B}$, which are available for every element of $\mathcal{B}$ since $\mathcal{B}\subset \mathcal{L}$.

\paragraph{GLL Training: } The data is classified by label propagation on a graph built on a batch $\mathcal{B}$. 
Because the GLL is a transductive classifier, it requires constructing each batch to include an additional set of labeled data - which we call the \textit{base dataset} $\mathcal{L}_b$  - to propagate the labels to the rest of the batch. 
% The classification is processed on a subset of nodes, called the \textit{base dataset} 
$\mathcal{L}_b$ is a subset sampled from the labeled dataset $\mathcal{L}$ with a size $N_b$. The remainder of the batch includes standard training samples $\mathcal{B}_l$ - whose labels we predict and compute a loss over - and (optionally) some unlabeled samples $\mathcal{B}_u$ in (e.g.) the semi-supervised setting. The batch can thus be written as

\begin{equation}
\mathcal{B} = \mathcal{L}_b \cup \mathcal{B}_l \cup \mathcal{B}_u,
\end{equation}
where $\mathcal{B}_l = (\mathcal{B}\setminus \mathcal{L}_b) \cap \mathcal{L}$ and $\mathcal{B}_u = (\mathcal{B}\setminus \mathcal{L}_b) \cap \mathcal{U}$. Let $N_l = |\mathcal{B}_l|$, $N_u = |\mathcal{B}_u|$, and $N_b + N_u = B$. We further require that, after fixing the base dataset $\mathcal{L}_b$, the proportion of labeled and unlabeled samples in each batch reflects their proportion in the remaining training set. Excluding the base samples, we sample $\mathcal{B}_l$ from $\mathcal{L} \setminus \mathcal{L}_b$ and $\mathcal{B}_u$ from $\mathcal{U}$ such that
\begin{equation}
\frac{N_l}{N_l + N_u} = \frac{\left|\mathcal{L} \setminus \mathcal{L}_b\right|}{\left|\mathcal{L} \setminus \mathcal{L}_b\right| + \left|\mathcal{U}\right|}.
\end{equation}
Equivalently, for a fixed batch size $B$ and the base dataset $\mathcal{L}_b$, we choose $N_l$ and $N_u$ to satisfy $N_l + N_u = B$ and
\begin{equation}\label{eq:N_l_and_N_u}
N_l = \frac{\left|\mathcal{L} \setminus \mathcal{L}_b\right|}{\left|\mathcal{L} \setminus \mathcal{L}_b\right| + \left|\mathcal{U}\right|}B, \quad
N_u = \frac{\left|\mathcal{U}\right|}{\left|\mathcal{L} \setminus \mathcal{L}_b\right| + \left|\mathcal{U}\right|}B.
\end{equation}
In practice, we choose a suitable batch size $B$ and base dataset size $N_b$ such that the resulting $N_l$ and $N_u$ computed from \eqref{eq:N_l_and_N_u} are both integers. For GLL, the training loss $\Loss$ is computed between the predicted labels and ground-truth labels over all samples in $\mathcal{B}_l$. Note that although the base dataset $\mathcal{L}_b$ is labeled, those labels are already used for graph-based label propagation and are therefore not included in the loss computation. 

\begin{remark}[Practical usage of GLL]
The proposed GLL-based batching scheme leads to several useful regimes in practice:
\begin{enumerate}
    \item \textbf{Semi-supervised learning.}
    In the standard semi-supervised setting, the training set decomposes as $\X = \mathcal{L} \cup \mathcal{U}$ with $\mathcal{L}$ labeled and $\mathcal{U}$ unlabeled. In each batch, the nodes in $\mathcal{B}_u$ are treated as unlabeled on the graph: they do not contribute directly to the loss, but they participate in graph construction and label propagation, and hence influence the predictions on $\mathcal{B}_l$. In this way the GLL can fully exploit the unlabeled data to improve classification on the labeled subset.

    \item \textbf{Choice of base dataset during training.}
    During training, the base dataset \(\mathcal{L}_b\) is selected at the beginning of each epoch, and all batches within the same epoch share the same base set \(\mathcal{L}_b\). We consider two strategies for constructing \(\mathcal{L}_b\), including a simple stratified random resampling scheme that preserves the class proportions of the full labeled set \(\mathcal{L}\) and an uncertainty-based selection method. The precise sampling procedures are described in Appendix~\ref{sec:base_dataset_sampling}.

    \item \textbf{Transductive inference at test time.}
    At test time, the GLL still requires a base set to construct the graph, but no backpropagation through the GLL classifier is needed. We can therefore first apply the feature encoder (the network backbone preceding the GLL classifier) to all available samples, including both training and test data, using standard mini-batch processing. Based on the resulting feature representations, we then construct a single large graph and solve the graph learning problem described in Section~\ref{sec:classic_gl} on this graph to obtain predictions for the test nodes and compute the test accuracy. Since this stage only involves linear algebra on fixed features and does not require gradients, it can be moved to the CPU to reduce GPU memory usage.
\end{enumerate}
\end{remark}

\subsubsection{Pretraining}
Due to its transductivity, GLL accuracy and learning behavior may degenerate if the feature embeddings are poor, which is often the case for an untrained network. To remedy this, we warm the model up with a pretext task that allows the feature extractor to learn basic features.
Prior works (\cite{brown2023utilizing, chen2025cgap}) show that feature representations learned from contrastive learning pre-training exhibit more pronounced clustering structure, leading to improved performance for graph-based learning methods such as GLL. 
We follow suit and warm-up the models using the contrastive methods SimCLR (\cite{chen2020simple}) or SupCon (\cite{khosla2020supervised}), depending on the availability of labeled data. 
%as they share similarities with the GLL. 
Contrasive learning and the GLL are a natural pairing, as the loss functions for the contrastive methods use a similarity function that is effectively the same as the one used in the similarity graph construction in the GLL. Warming up the models with contrastive learning will thus help the model capture features and structures similar to the GLL and accelerate learning. 
The contrastive methods are discussed further in Appendix~\ref{sec:appendix_simclr}. 
Warm-ups were also used in previous work on graph-based classification heads (\cite{wang2021graph}), where the authors alternate between 400 epochs of MLP-based training and 5 epochs of WNLL-based training in order to stabilize learning. Conversely, we warm-up with contrastive learning, and then proceed with GLL training, with no alternating stages.

In practice, we consider a mixed training target with the SimCLR loss $\mathcal{L}_\mathrm{sim}$ \eqref{eq: SimCLR_loss} and the SupCon loss $\mathcal{L}_\mathrm{sup}$ \eqref{eq: SupCon_loss}. We only pretrain the feature encoder before the classifier in the neural network architecture since both losses do not require predicted labels. The pretraining loss is defined by:
\begin{equation}\label{eq:pretrain_loss}
    \mathcal{L}_\mathrm{pretrain} = \gamma \mathcal{L}_\mathrm{sim} + (1 - \gamma) \mathcal{L}_\mathrm{sup}, \gamma\in [0,1].
\end{equation}
The contrastive pretraining described above is agnostic to the choice of classification head and can therefore be used both with classical MLP classifiers and with our GLL classifier. For each pretraining batch, we draw a mixture of labeled and unlabeled samples from the full training set $\mathcal{X} = \mathcal{L} \cup \mathcal{U}$, in the same spirit as the batch construction in Section~\ref{sec:sampling_and_training_on_batches}. The SimCLR loss $\mathcal{L}_\mathrm{sim}$ is evaluated on all samples in the batch, regardless of whether they are labeled or unlabeled, since it does not require label information. By contrast, the SupCon loss $\mathcal{L}_\mathrm{sup}$ is computed only on those samples in the pretraining batch that have labels, using their ground-truth class assignments. In this way, the pretraining stage can exploit all available data while still benefiting from supervised signal when labels are present, and the resulting encoder can be paired with any downstream classifier, including GLL.

\begin{remark}[Choice of mixing parameter $\gamma$]
In practice, we choose the mixing parameter $\gamma$ in $\mathcal{L}_\mathrm{pretrain}$ \eqref{eq:pretrain_loss} from the discrete candidate set $\{0.01, 0.25, 0.5, 0.75, 0.99\}$. For each candidate value, we perform the contrastive pretraining stage to obtain a feature encoder and then feed the resulting feature vectors into a parameter-free graph learning classifier. We compute the resulting classification accuracy on the labeled subset $\mathcal{L} \subset \mathcal{X}$ of the training data and select the value of $\gamma$ that achieves the highest accuracy on this labeled subset.
\end{remark}

% We also conduct experiments that use an MLP projection head initially during training before switching to a GLL. This essentially serves the same purpose as warming up with contrastive learning but demonstrates the ability of the GLL to further improve on the MLP. \harris{clarify which experiments/sections use contrastive warm up vs MLP warm-up?}  

%We calculate SupCon loss on $L$ only [cite]. And calculate SimCLR loss on $\X$ [cite]. Need to write the loss here and clarify the connection to similarity graph.
%\subsubsection*{Fully Supervised Training}
%Denote $L_\mathrm{batch} = \X_\mathrm{batch}\cap L$, $U_\mathrm{batch} = \X_\mathrm{batch}\cap U$. We construct a graph on $\X_\mathrm{base}\cup X_\mathrm{batch}$ and then use $\X_\mathrm{base}$ to predict the labels of $X_\mathrm{batch}$. But the cross entropy loss is only calculated on $L_\mathrm{batch}\subset X_\mathrm{batch}$. Mention that this is different from classical fully supervised training -- unlabeled data are still involved in the training. 

\subsubsection{Data Augmentation}\label{sec:data_augmentation}
Data augmentation is applied throughout both the contrastive pretraining stage without a classification head and the subsequent supervised training of the full network with different classifiers. This follows the standard practice in contrastive representation learning (\cite{chen2020simple,chen2020big,khosla2020supervised}), where multiple views of the same input are generated via stochastic transformations to form positive pairs. During evaluation on all benchmarks, we disable data augmentation and feed the original inputs to the encoder and classifier.

Concretely, at each training step we first sample a mini-batch of images from the training set and then apply a stochastic augmentation pipeline to each image. During contrastive pretraining, every sample in the batch is independently transformed twice to produce two augmented views, which are then used as positive pairs for the SimCLR and SupCon losses. In the fully supervised training phase, each data sample is augmented once on-the-fly to increase the diversity of training examples. 
% Because the augmentation pipeline is stochastic, the exact sequence and magnitudes of transformations applied to a given data point change across epochs and across different mini-batches. 

The augmentation pipeline itself is based on a RandAugment-style (\cite{cubuk2019randaugment}) policy: for each image, a small number of transformations is drawn at random from a pool that includes common photometric and geometric operations, such as autocontrast, histogram equalization, brightness/contrast adjustments, rotations, shears, translations, and solarization, followed by a random cutout operation that masks out a square region of the image. For grayscale datasets we use a slightly modified variant of this policy that is tailored to single-channel images, e.g., removing color-only operations and using grayscale cutout. The full list of transformations and the exact sampling procedure are provided in Appendix~\ref{sec:appendix_augmentation}. RandAugment is applied to all samples in the batch, including those belonging to the base dataset $\mathcal{L}_b$.

}

\begin{figure}[!t]
    \centering
    \includegraphics[width=0.4\textwidth]{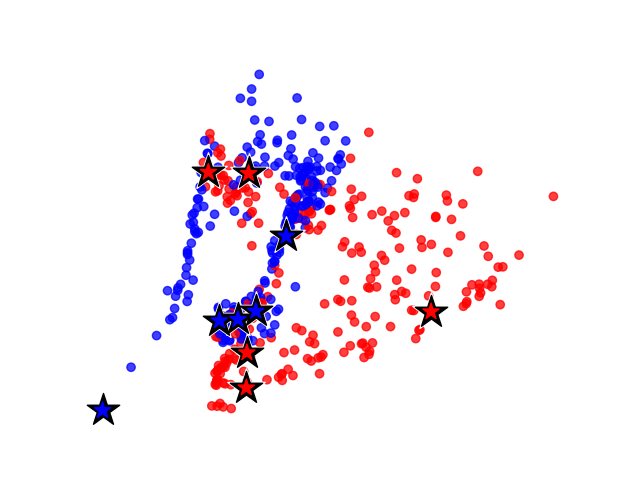}
    \caption{The initial embedding of the two moons dataset from the random seed. The blue and red correspond to the different classes, with the starred samples corresponding to the base samples for the respective classes.}
    \label{fig:tauinit}
\end{figure}

\subsection{Understanding Embeddings and the Effect of Tikhonov Regularization via Two Moons}\label{sec:toy}

In this subsection, we highlight the differences between the learned embeddings of neural networks when trained with a GLL compared to a softmax layer. Moreover, we demonstrate the effect on the embedding by modifying the decay parameter, \(\tau\), from Section \ref{sec:tau}. 

The softmax classifier uses no relational information and thus will create an embedding that leads to a straightforward separation based on the features alone. The GLL, using relational information, will learn to separate the data such that the generated similarity graph will be nearly disjoint by class (a totally disjoint graph can prevent learning). Therefore, we may expect the GLL to produce a more relaxed embedding, as a disjoint similarity graph does not necessarily require the data to be separated to a probability simplex. Increasing the Tikhonov regularization, $\tau$, inhibits label propagation and weakens predictions, meaning higher confidence predictions and lower loss will require a similarity graph that is even more separated. Specifically, increasing \(\tau\) will increase the learning rate and convergence speed of the process and should lead to more locally dense learned embeddings. We give a brief example of these effects in two dimensions with the simple toy dataset, two moons (\cite{scikit-learn})).

In this experiment, we perform supervised classification on a two-dimensional generated dataset of 500 samples from the two moons distribution with a noise of 0.1 and random state of one. The network was trained with a stochastic gradient descent optimizer with a learning rate of 0.01. 

The architecture backbone for this experiment is a simple multi-layer perceptron neural network with three fully connected (FC) layers and two ReLU activation layers. The hidden layer has a size of 64 and the output size is two for two dimensional visualization purposes. We choose an initialization for the neural network such that the initial embedding has sufficiently mixed samples (Figure \ref{fig:tauinit}). As the dataset is simplistic, there will be no projection head in this comparison, meaning the only difference between the standard MLP model and the GLL model will be purely the softmax and GLL classifiers. 

We arbitrarily fix the first ten samples, five from each class, to act as the base samples which propagate the labels. As the dataset is small, we load all the samples in each batch so there is no concern about partial class representation. To make the results more interpretable, we fix the \(\eps_k\) term, from the similarity equation, Equation \ref{eq:weightform}, to be fixed at one. 

The embeddings for this ablation study are shown in Figure \ref{fig:grid}. We compare five values of \(\tau\): 0, 0.001, 0.01, 0.1, and 0.5. The figure shows the learned embedding of the dataset from the neural network across lengthening training epochs. 

We note that this regime is deliberately constrained to two dimensions, meaning that the network may struggle to 'move' the embeddings from one class beyond the other class's base samples, as approaching the opposing base samples will increase the loss. The seed for these samples was chosen to be difficult and we note that the problem is easier in a higher dimensional space. Despite this, after sufficient training, the model successfully separates the data. %The transductive nature of the GLL layer indicates that the label predictions for the samples are relative to the locality of the samples. Therefore if samples are blocked from their respective class by a differing class, the gradients from that sample may be very small since there is very low relative connection. 

\begin{figure}[!t]
    \centering
    \renewcommand{\thesubfigure}{\arabic{subfigure}} % Numbering the subfigures
    \setlength{\tabcolsep}{0pt} % Adjust column spacing
    \renewcommand{\arraystretch}{0} % Adjust row spacing
    % https://tex.stackexchange.com/questions/7208/how-to-vertically-center-the-text-of-the-cells - m{} sets the width and will vertically center text. used for the tau values. the text then becomes left aligned so centering\arraybackslash is used to center it
    \begin{tabular}{ >{\centering\arraybackslash} m{2cm} >{\centering\arraybackslash} m{2cm} >{\centering\arraybackslash} m{2cm} >{\centering\arraybackslash} m{2cm} >{\centering\arraybackslash} m{2cm} >{\centering\arraybackslash} m{2cm} >{\centering\arraybackslash} m{2cm}}
        % Column labels
        & Ep. 0 & Ep. 500 & Ep. 2e3 & Ep. 1e4 & Ep. 5e4 & Ep. 1.5e5 \\
        
        % Row 1
        \(\tau=0\)&
        \begin{subfigure}[b]{0.15\textwidth}
            \includegraphics[width=\textwidth]{figures/tau/0/1_Emb.png}
        \end{subfigure} &
        \begin{subfigure}[b]{0.15\textwidth}
            \includegraphics[width=\textwidth]{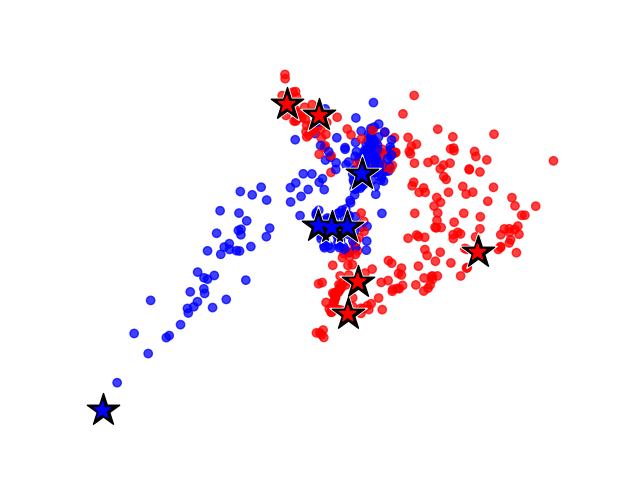}
        \end{subfigure} & 
        \begin{subfigure}[b]{0.15\textwidth}
            \includegraphics[width=\textwidth]{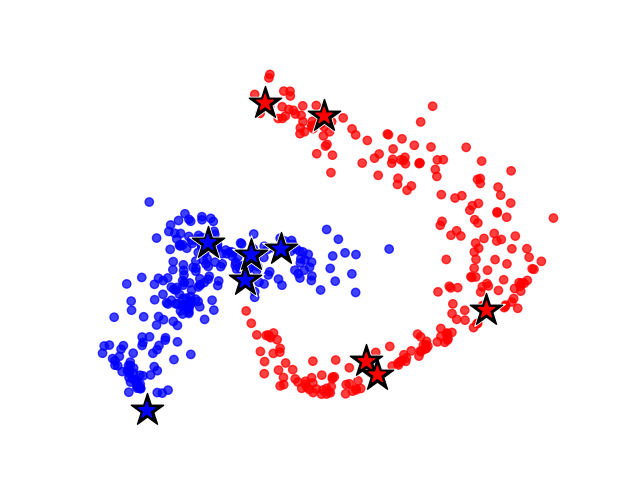}
        \end{subfigure} &
        \begin{subfigure}[b]{0.15\textwidth}
            \includegraphics[width=\textwidth]{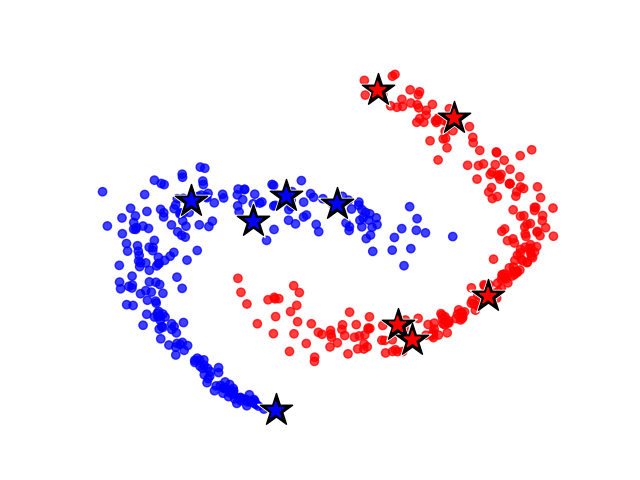}
        \end{subfigure} &
        \begin{subfigure}[b]{0.15\textwidth}
            \includegraphics[width=\textwidth]{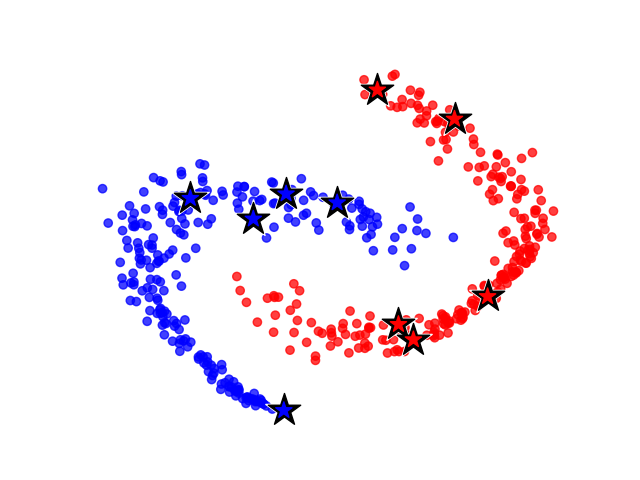}
        \end{subfigure} &
        \begin{subfigure}[b]{0.15\textwidth}
            \includegraphics[width=\textwidth]{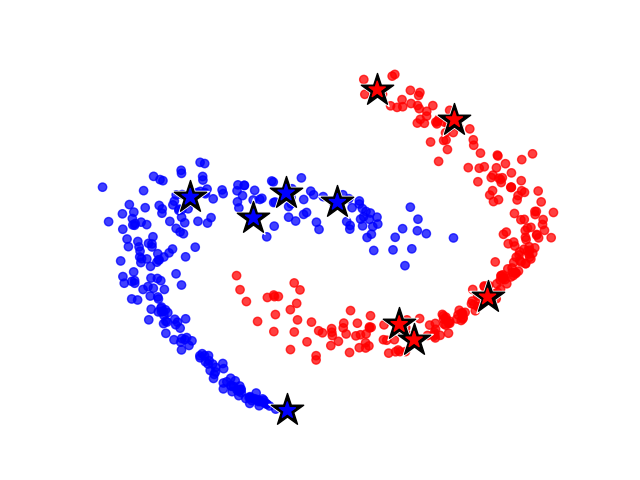}
        \end{subfigure} \\

        % Row 2
        \(\tau=0.001\) &
                \begin{subfigure}[b]{0.15\textwidth}
            \includegraphics[width=\textwidth]{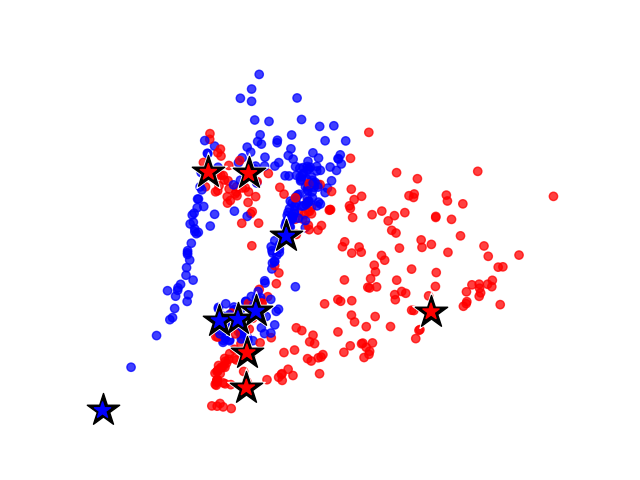}
        \end{subfigure} &
        \begin{subfigure}[b]{0.15\textwidth}
            \includegraphics[width=\textwidth]{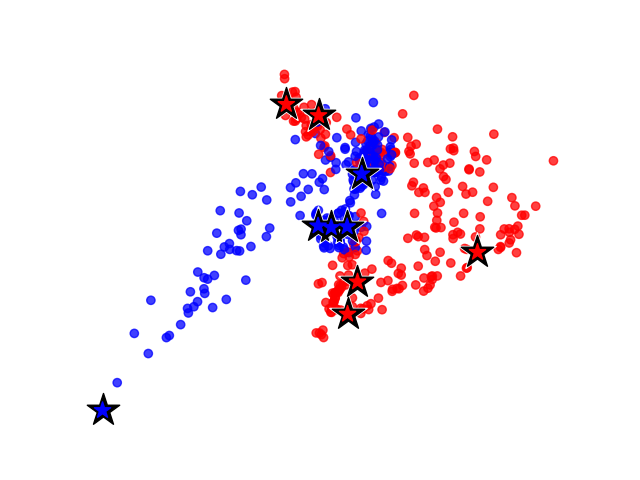}
        \end{subfigure} & 
        \begin{subfigure}[b]{0.15\textwidth}
            \includegraphics[width=\textwidth]{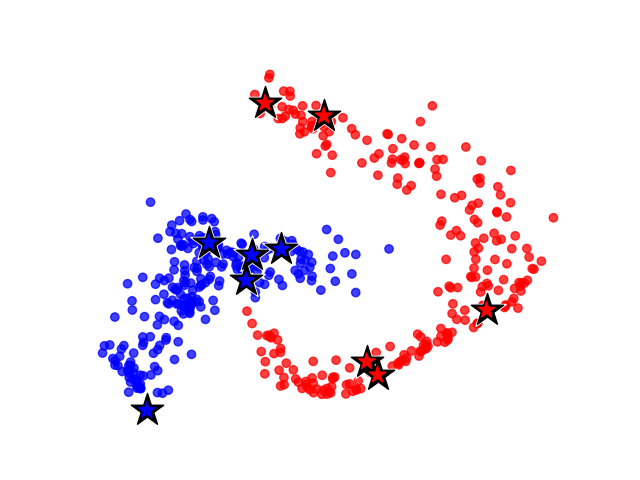}
        \end{subfigure} &
        \begin{subfigure}[b]{0.15\textwidth}
            \includegraphics[width=\textwidth]{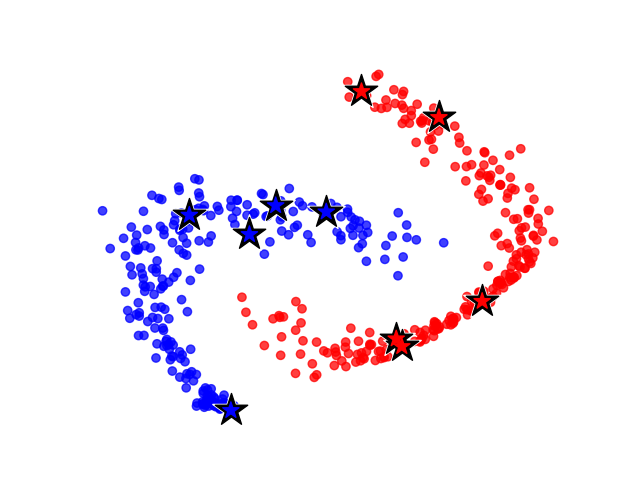}
        \end{subfigure} &
        \begin{subfigure}[b]{0.15\textwidth}
            \includegraphics[width=\textwidth]{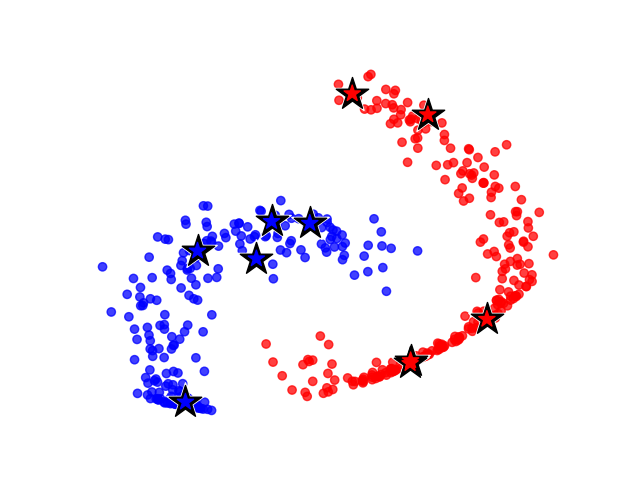}
        \end{subfigure} &
        \begin{subfigure}[b]{0.15\textwidth}
            \includegraphics[width=\textwidth]{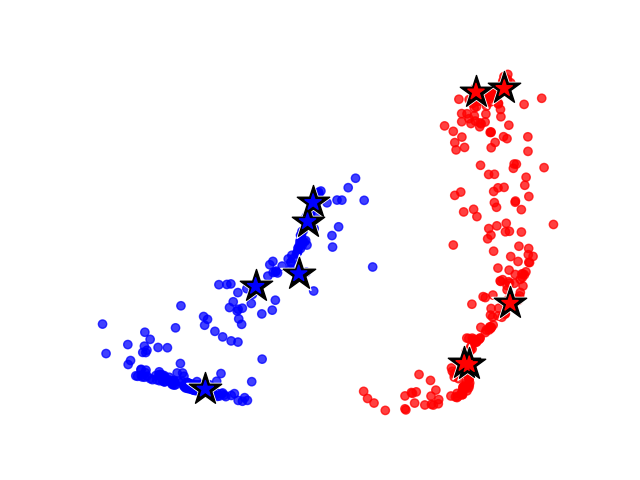}
        \end{subfigure} \\

        % Row 3
        \(\tau=0.01\) &
                \begin{subfigure}[b]{0.15\textwidth}
            \includegraphics[width=\textwidth]{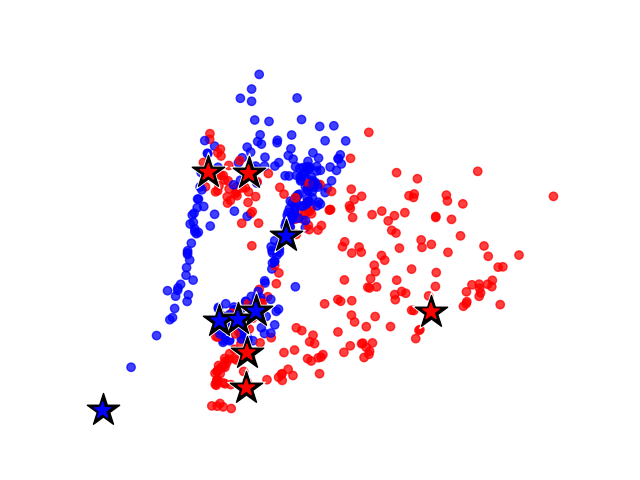}
        \end{subfigure} &
        \begin{subfigure}[b]{0.15\textwidth}
            \includegraphics[width=\textwidth]{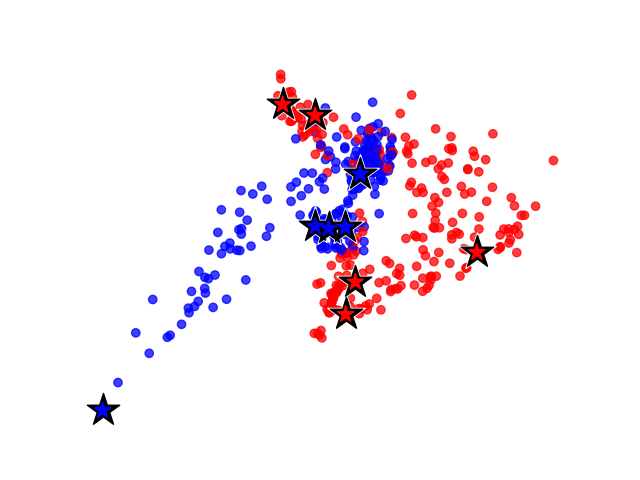}
        \end{subfigure} & 
        \begin{subfigure}[b]{0.15\textwidth}
            \includegraphics[width=\textwidth]{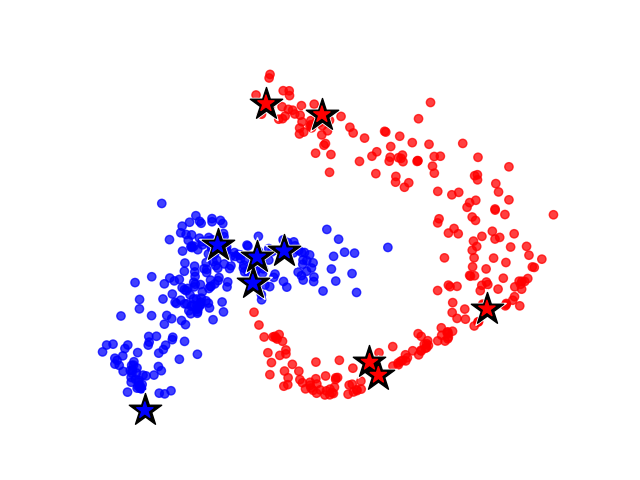}
        \end{subfigure} &
        \begin{subfigure}[b]{0.15\textwidth}
            \includegraphics[width=\textwidth]{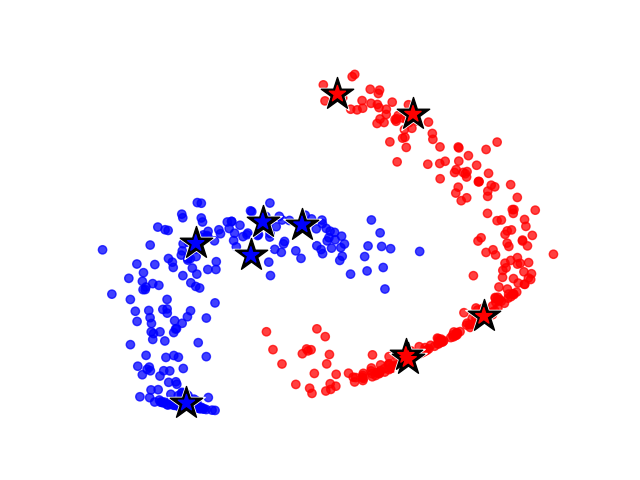}
        \end{subfigure} &
        \begin{subfigure}[b]{0.15\textwidth}
            \includegraphics[width=\textwidth]{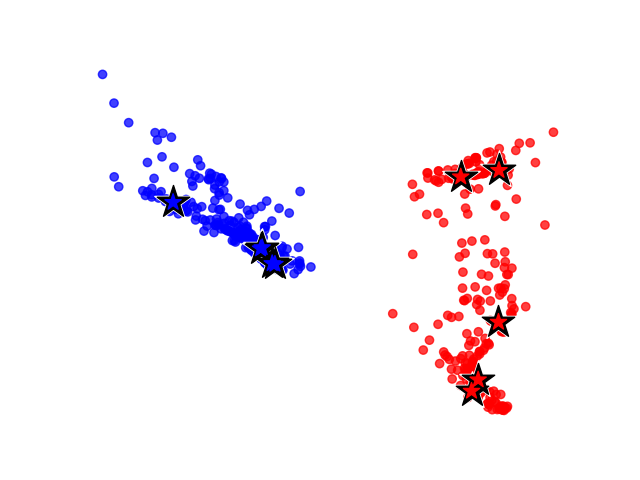}
        \end{subfigure} &
        \begin{subfigure}[b]{0.15\textwidth}
            \includegraphics[width=\textwidth]{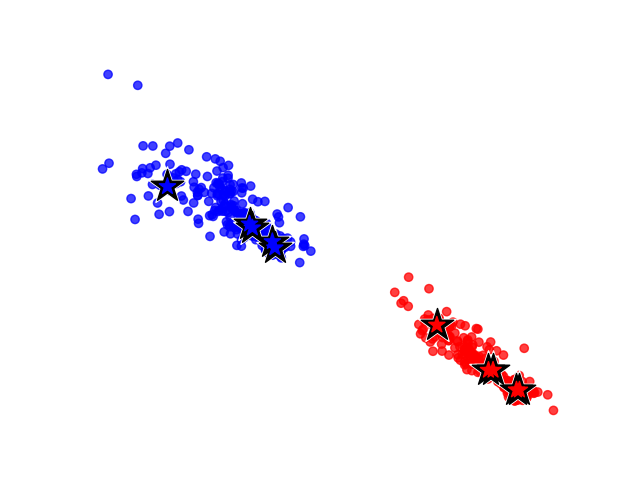}
        \end{subfigure} \\

        % Row 4
        \(\tau=0.1\) &
                \begin{subfigure}[b]{0.15\textwidth}
            \includegraphics[width=\textwidth]{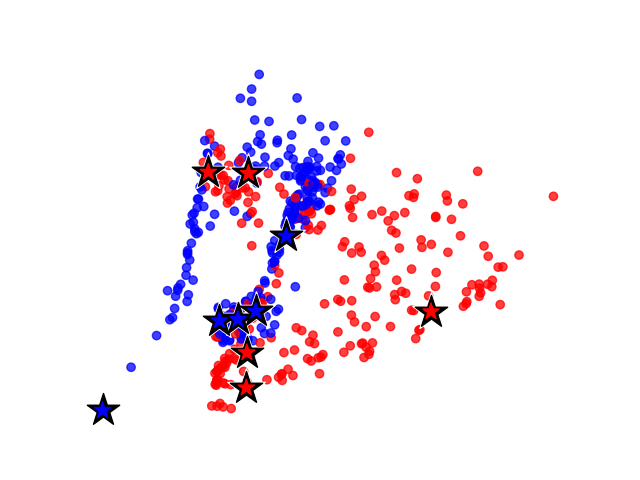}
        \end{subfigure} &
        \begin{subfigure}[b]{0.15\textwidth}
            \includegraphics[width=\textwidth]{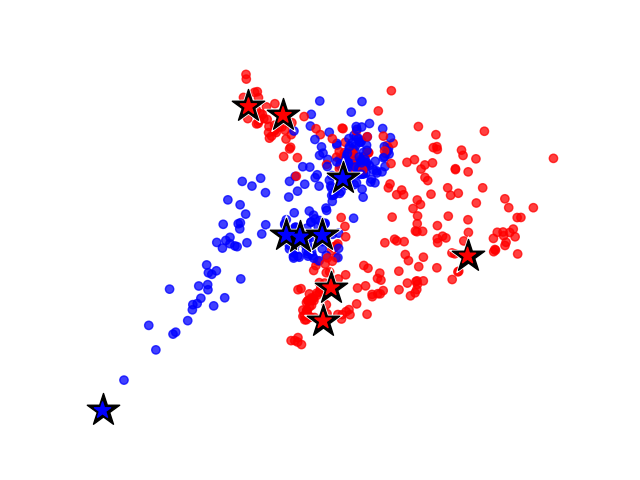}
        \end{subfigure} & 
        \begin{subfigure}[b]{0.15\textwidth}
            \includegraphics[width=\textwidth]{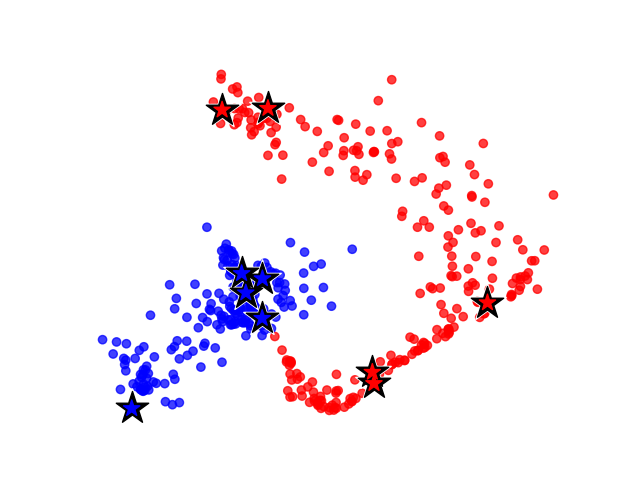}
        \end{subfigure} &
        \begin{subfigure}[b]{0.15\textwidth}
            \includegraphics[width=\textwidth]{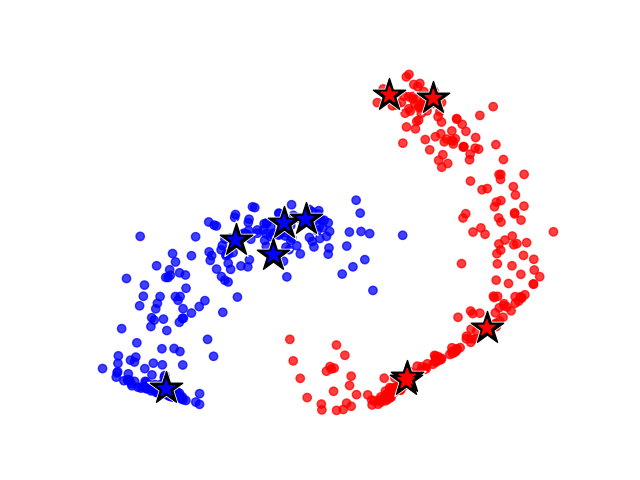}
        \end{subfigure} &
        \begin{subfigure}[b]{0.15\textwidth}
            \includegraphics[width=\textwidth]{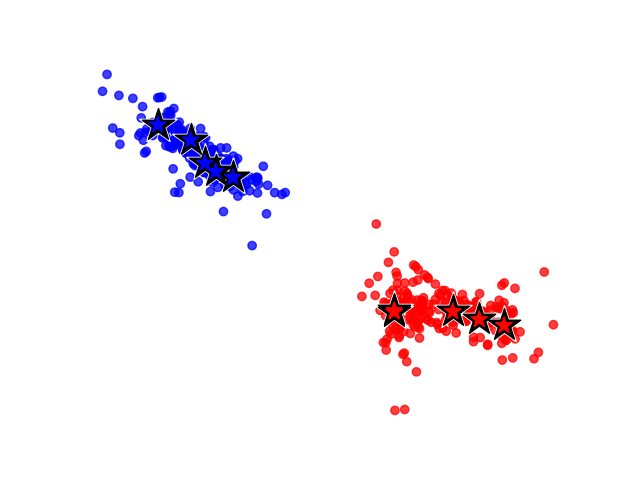}
        \end{subfigure} &
        \begin{subfigure}[b]{0.15\textwidth}
            \includegraphics[width=\textwidth]{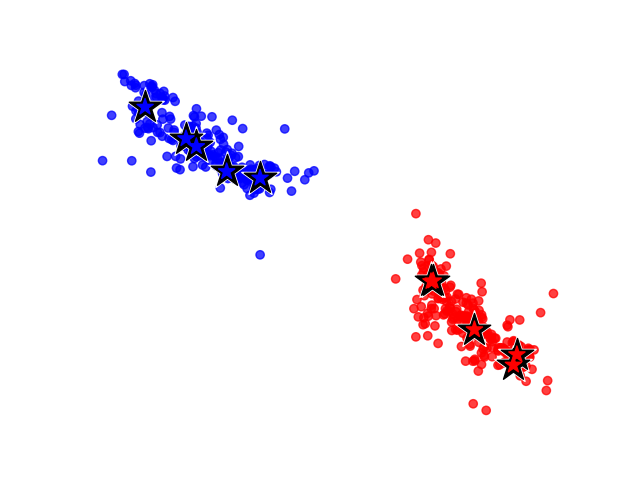}
        \end{subfigure} \\

        % Row 5
        \(\tau=0.5\) &
                \begin{subfigure}[b]{0.15\textwidth}
            \includegraphics[width=\textwidth]{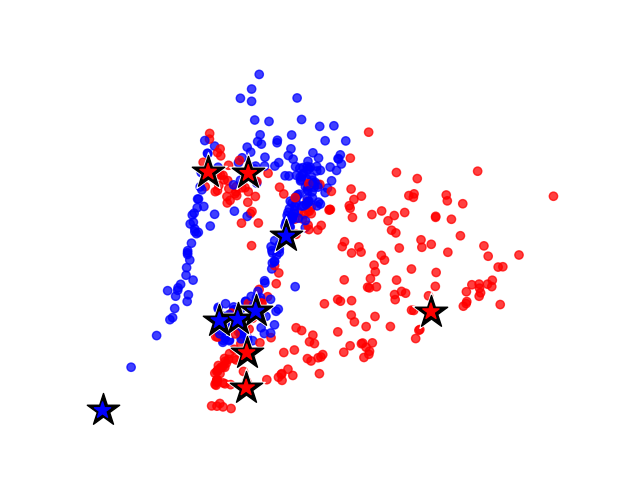}
        \end{subfigure} &
        \begin{subfigure}[b]{0.15\textwidth}
            \includegraphics[width=\textwidth]{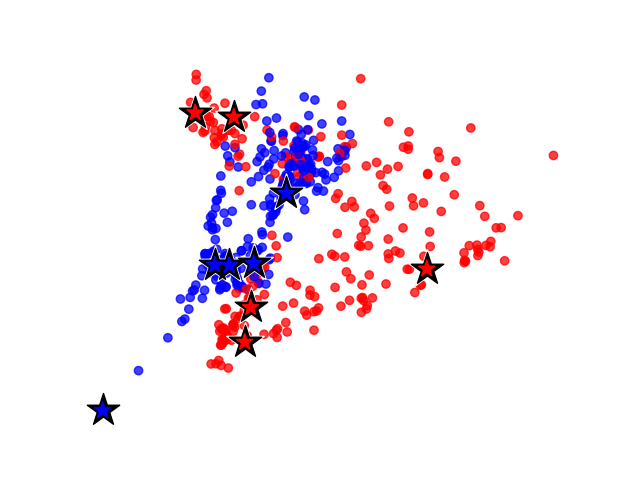}
        \end{subfigure} & 
        \begin{subfigure}[b]{0.15\textwidth}
            \includegraphics[width=\textwidth]{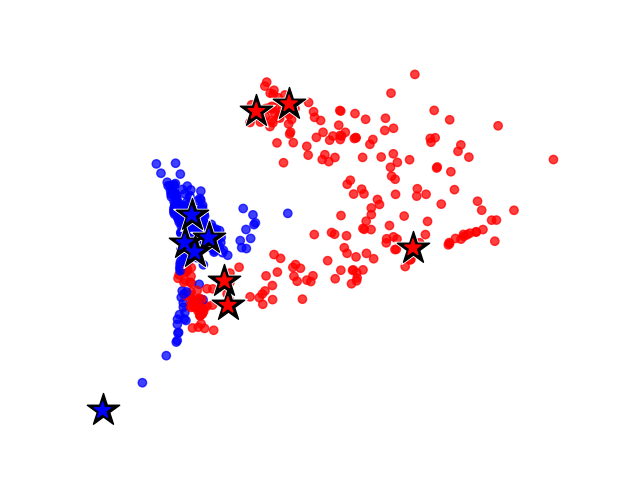}
        \end{subfigure} &
        \begin{subfigure}[b]{0.15\textwidth}
            \includegraphics[width=\textwidth]{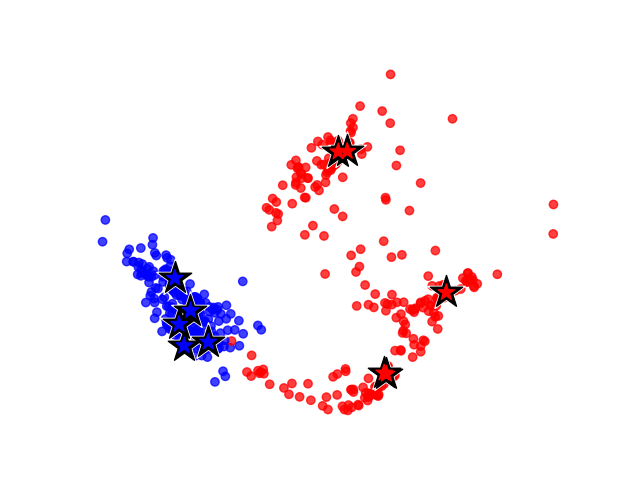}
        \end{subfigure} &
        \begin{subfigure}[b]{0.15\textwidth}
            \includegraphics[width=\textwidth]{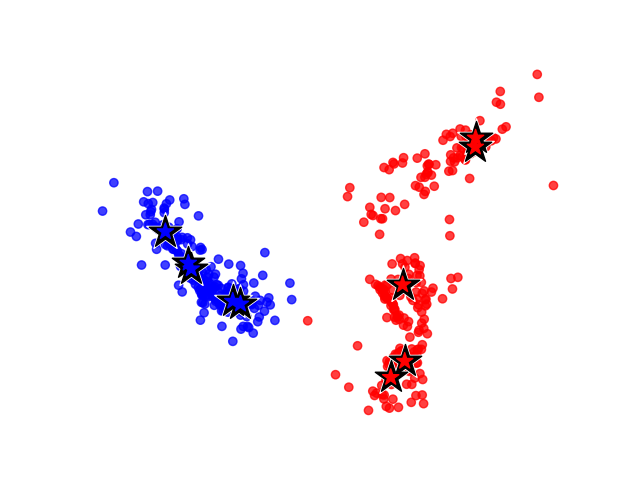}
        \end{subfigure} &
        \begin{subfigure}[b]{0.15\textwidth}
            \includegraphics[width=\textwidth]{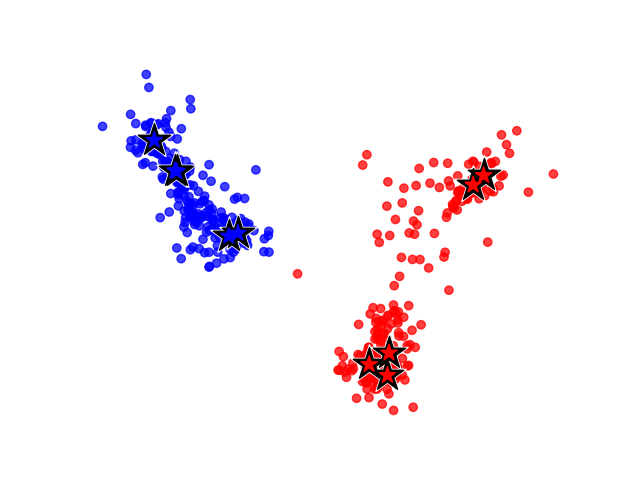}
        \end{subfigure} \\

        % Row 6
        Softmax &
                \begin{subfigure}[b]{0.15\textwidth}
            \includegraphics[width=\textwidth]{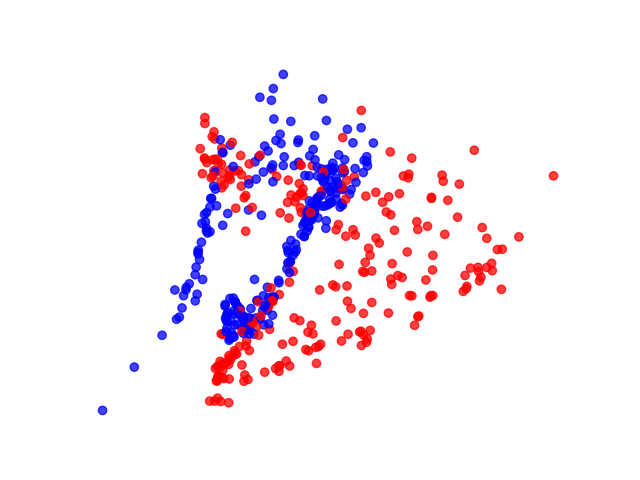}
        \end{subfigure} &
        \begin{subfigure}[b]{0.15\textwidth}
            \includegraphics[width=\textwidth]{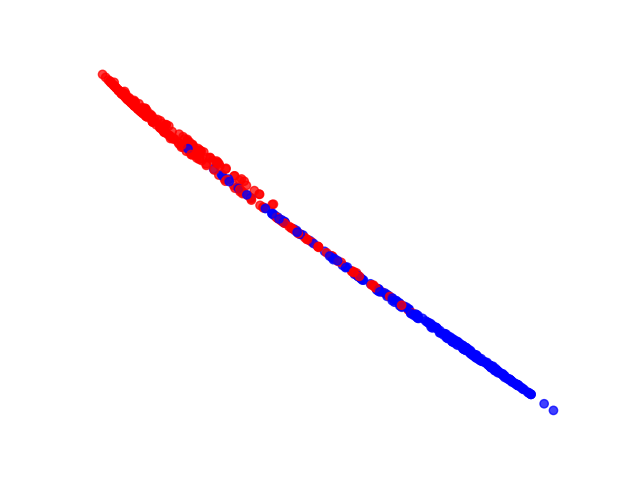}
        \end{subfigure} & 
        \begin{subfigure}[b]{0.15\textwidth}
            \includegraphics[width=\textwidth]{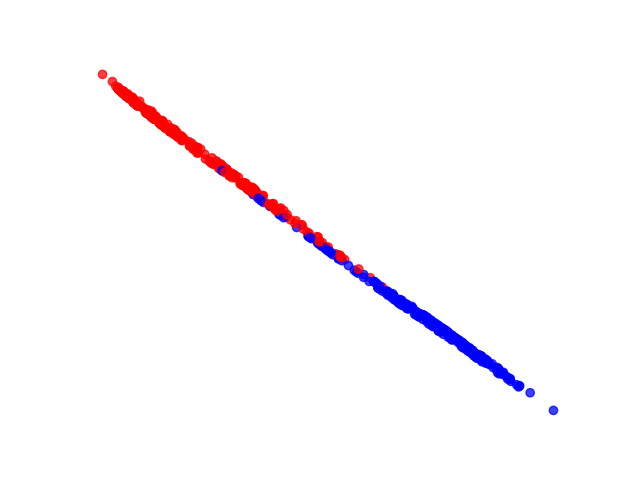}
        \end{subfigure} &
        \begin{subfigure}[b]{0.15\textwidth}
            \includegraphics[width=\textwidth]{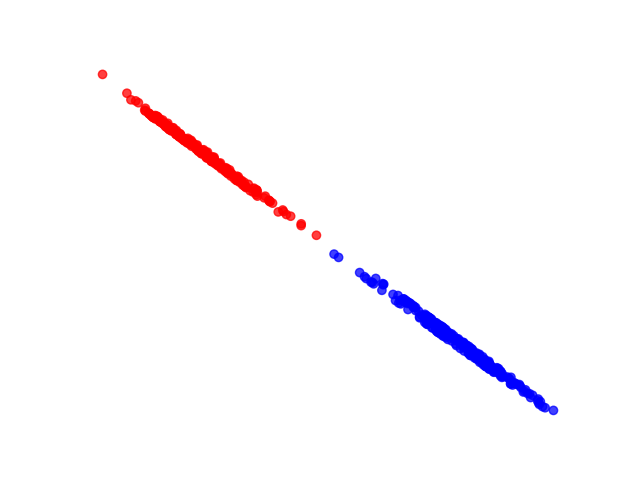}
        \end{subfigure} &
        \begin{subfigure}[b]{0.15\textwidth}
            \includegraphics[width=\textwidth]{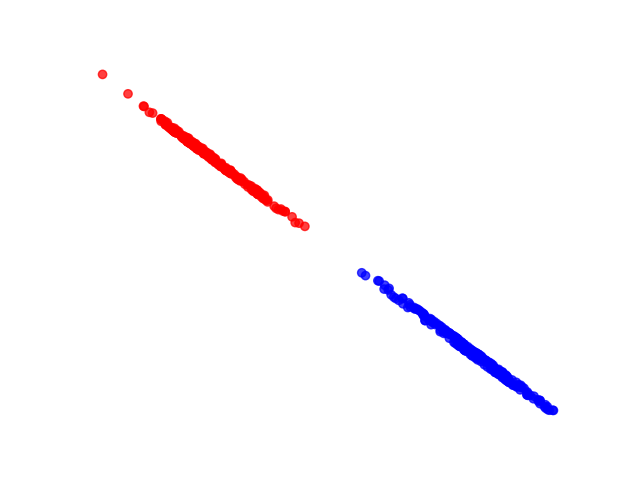}
        \end{subfigure} &
        \begin{subfigure}[b]{0.15\textwidth}
            \includegraphics[width=\textwidth]{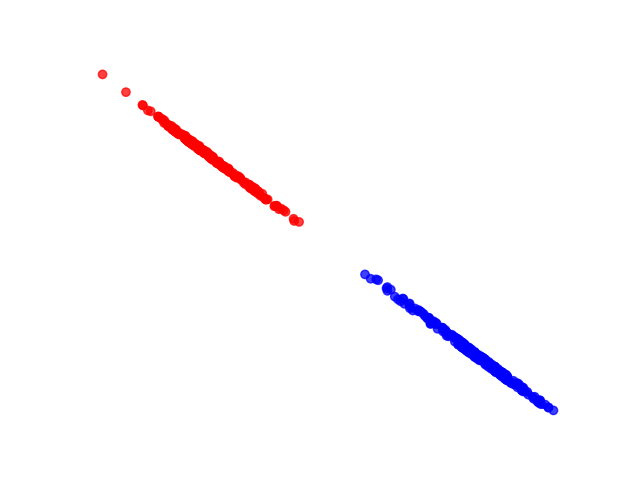}
        \end{subfigure} \\
    \end{tabular}
    \caption{Visualization of effect of \(\tau\) on spatial embeddings. Each row corresponds to a graph learning model with a specified value of \(\tau\) or a softmax classification head. The columns correspond to the epoch of the training. With low values of \(\tau\), the model separates the data by class and resembles the original two moons shape. Increasing \(\tau\) forces tighter clustering and localization. The softmax layer simply learns a linear separation with no regard to relational information. Further training with the softmax layer increases the scale, which is not seen without the axes. }
    \label{fig:grid}
\end{figure}

Examining the results, the baseline MLP encoder trained with the softmax layer simply separated the data approximately along a line. However, the same encoder architecture trained with the GLL will recover the original two moons shape, implying that the two moons shape induces a strongly separated $k$-nearest neighbors similarity graph. 

The addition of the \(\tau\) decay parameter will encourage a tighter clustering of the classes around the base labels. For low values of \(\tau\), we still see initial convergence to the two moons before diverging into tighter clusters around the base samples. If the \(\tau\) value is large, the divergence is more immediate and is seen splintering one of the classes into multiple clusters localized about different base samples. This can be rationalized as the base samples pulling the unlabeled samples closer to themselves, strengthening their own influence and dampening the influence and learning from distant base samples. This is because increasing \(\tau\) acts as an exponential dampening term which localizes the influence of the base samples.  

{
\subsection{Toy Example: Over-parameterized CNN on FashionMNIST}\label{sec:comparison_mlp_gll}

In this subsection, we focus on a toy example on the FashionMNIST dataset to compare our GLL classifier with a MLP classifier. The goal is to study a setting where the backbone network is intentionally over-parameterized and challenging to optimize, so that the effect of replacing the MLP classifier by GLL becomes more evident.

\paragraph{Experimental setup.}
We use the FashionMNIST dataset (\cite{xiao2017fashion}) and construct a simple convolutional neural network as the feature encoder. The architecture is summarized in Table~\ref{tab:structure_mlp_gll}. The feature encoder consists of three convolutional layers followed by max-pooling layers and a fully connected (FC) layer. Each convolutional and FC layer is followed by a ReLU activation function, which is omitted from the table for brevity. The resulting encoder is purposely over-parameterized relative to FashionMNIST, and we will see that it is difficult to train effectively when combined with a standard MLP classifier.

\begin{table}[htbp]
\centering
\begin{tabular}{|c|c|c|c|}
\hline
\textbf{Structure} & \textbf{Layer} & \textbf{Details} & \textbf{Feature Size}
\\\hline
Input Image & - & From FashionMNIST Dataset & 28 $\times$ 28 $\times$ 1 \\
\hline
\multirow{6}{*}{\centering\parbox{2.5cm}{\centering Feature\\Encoder}} & Conv1 & 1 $\rightarrow$ 64, kernel size: 3, padding: 1 & 28 $\times$ 28 $\times$ 64 \\
& Conv2 & 64 $\rightarrow$ 128, kernel size: 3, padding: 1 & 28 $\times$ 28 $\times$ 128 \\
& MaxPool & kernel size: 2, stride: 2 & 14 $\times$ 14 $\times$ 128 \\
& Conv3 & 128 $\rightarrow$ 256, kernel size: 3, padding: 1 & 14 $\times$ 14 $\times$ 256 \\
& MaxPool & kernel size: 2, stride: 2 & 7 $\times$ 7 $\times$ 256 \\
& FC & 256 $\times$ 7 $\times$ 7 $\rightarrow$ 128 & 128 \\
\hline
\multirow{3}{*}{\parbox{2.5cm}{\centering MLP\\Classifier}} & FC1 & 128 $\rightarrow$ 1024 & 1024 \\
& FC2 & 1024 $\rightarrow$ 10 & 10 \\
& Softmax & dim: 1 & 10 \\
\hline
\end{tabular}
\caption{Network structure used in the FashionMNIST experiments (Section~\ref{sec:comparison_mlp_gll}). The feature encoder is deliberately over-parameterized to create a challenging training scenario, and the model contains 1,975,424 parameters in the feature encoder and 142,346 parameters in the MLP classifier.}
\label{tab:structure_mlp_gll}
\end{table}

The MLP classifier consists of two fully connected layers followed by a softmax layer, as shown in Table~\ref{tab:structure_mlp_gll}. To compare with our GLL classifier, we keep the feature encoder architecture fixed and replace the MLP classifier by GLL, which has no trainable parameters. In all cases, the final outputs are used to compute the cross-entropy loss with the ground-truth labels. All experiments in this subsection are fully supervised and use the entire FashionMNIST training set.

\paragraph{Training strategies.}
We consider four training strategies on FashionMNIST, all sharing the same feature encoder:

\begin{itemize}
\item \textbf{MLP only (100 MLP).} We train the network with the MLP classifier from scratch for 100 epochs. During this process, the encoder output is also passed to a GLL classifier in parallel. Since GLL has no trainable parameters, this does not affect the training dynamics. Throughout MLP training, we use the GLL outputs to compute loss (without backpropagation) and test error, allowing a direct comparison between MLP and GLL under the same encoder.

\item \textbf{GLL-0 (50 GLL).} We train the network from scratch using GLL as the classifier for 50 epochs. In the plots we refer to this strategy as GLL-0.

\item \textbf{GLL-50 (50 MLP + 50 GLL).} We first train the network with the MLP classifier for 50 epochs. Starting from epoch 51, we switch the classifier to GLL and continue training for another 50 epochs, using the gradients induced by the GLL loss.

\item \textbf{GLL-75 (75 MLP + 25 GLL).} We first train with the MLP classifier for 75 epochs, and then switch to GLL for the remaining 25 epochs.
\end{itemize}

\paragraph{Test error comparison.}
Table~\ref{tab:network_comparison_mlp_vs_gll} reports the final test classification errors in \%. For the MLP-only strategy, we report the test error of the MLP classifier, which is used for training, and the test error of the GLL classifier evaluated on the same encoder. For the remaining strategies, we report the test error of the GLL classifier at the end of training.

\begin{table}[t]
\centering
\small
% \resizebox{\linewidth}{!}{%
\begin{tabular}{lllc}
\toprule
Dataset & Trial & Epochs & Test error (\%) \\
\midrule
FashionMNIST & MLP only (MLP) & 100 MLP & 45.42 \\
FashionMNIST & MLP only (GLL) & 100 MLP & 14.66 \\
FashionMNIST & GLL only (GLL-0) & 50 GLL & \textbf{8.90} \\
FashionMNIST & GLL-50 & 50 MLP + 50 GLL & \textbf{8.68} \\
FashionMNIST & GLL-75 & 75 MLP + 25 GLL & \textbf{8.91} \\
\bottomrule
\end{tabular}
% }
\caption{Test classification error rates (\%, lower is better) on FashionMNIST for the customized network in Table~\ref{tab:structure_mlp_gll} under different training strategies. For the ``MLP only'' strategy we report both the error of the MLP classifier and that of a GLL classifier evaluated on the same encoder (GLL). For the remaining strategies, GLL is used as the classifier during training and evaluation.}
\label{tab:network_comparison_mlp_vs_gll}
\end{table}

Figure~\ref{fig:mlp_vs_gl} shows the evolution of the training loss and test accuracy over epochs. Figure~\ref{fig:gll_vs_mlp_loss} presents the training loss curves for all four strategies, while Figure~\ref{fig:gll_vs_mlp_acc} presents the corresponding test accuracies. The curves for GLL-50 and GLL-75 start at epochs 50 and 75, respectively, since these strategies switch to GLL after an initial MLP training phase.

\begin{figure}[t]
\centering
\begin{subfigure}{0.48\textwidth}
  \centering
  \includegraphics[width=\textwidth]{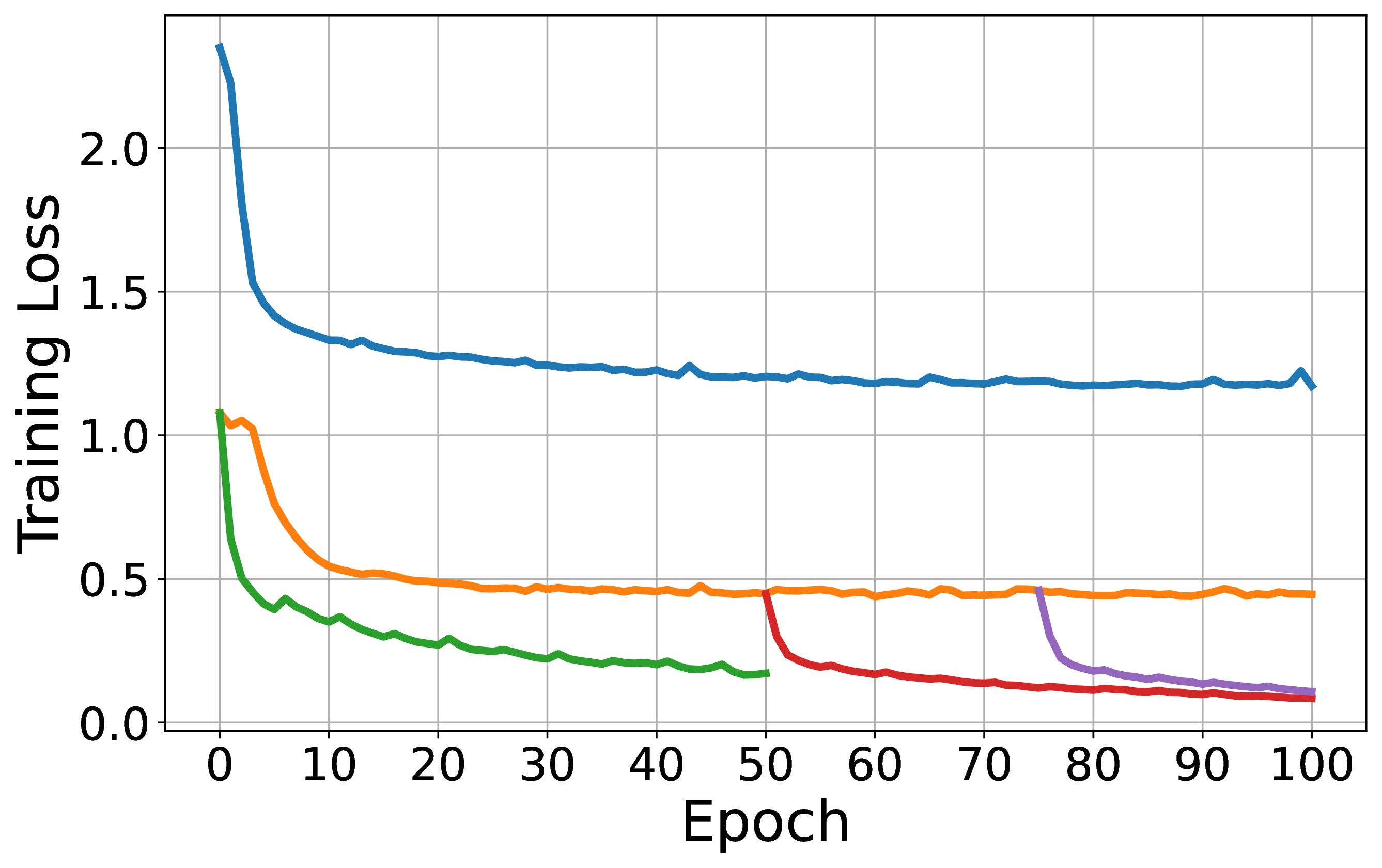}
  \caption{Training loss curves}
  \label{fig:gll_vs_mlp_loss}
\end{subfigure}
\begin{subfigure}{0.48\textwidth}
  \centering
  \includegraphics[width=\textwidth]{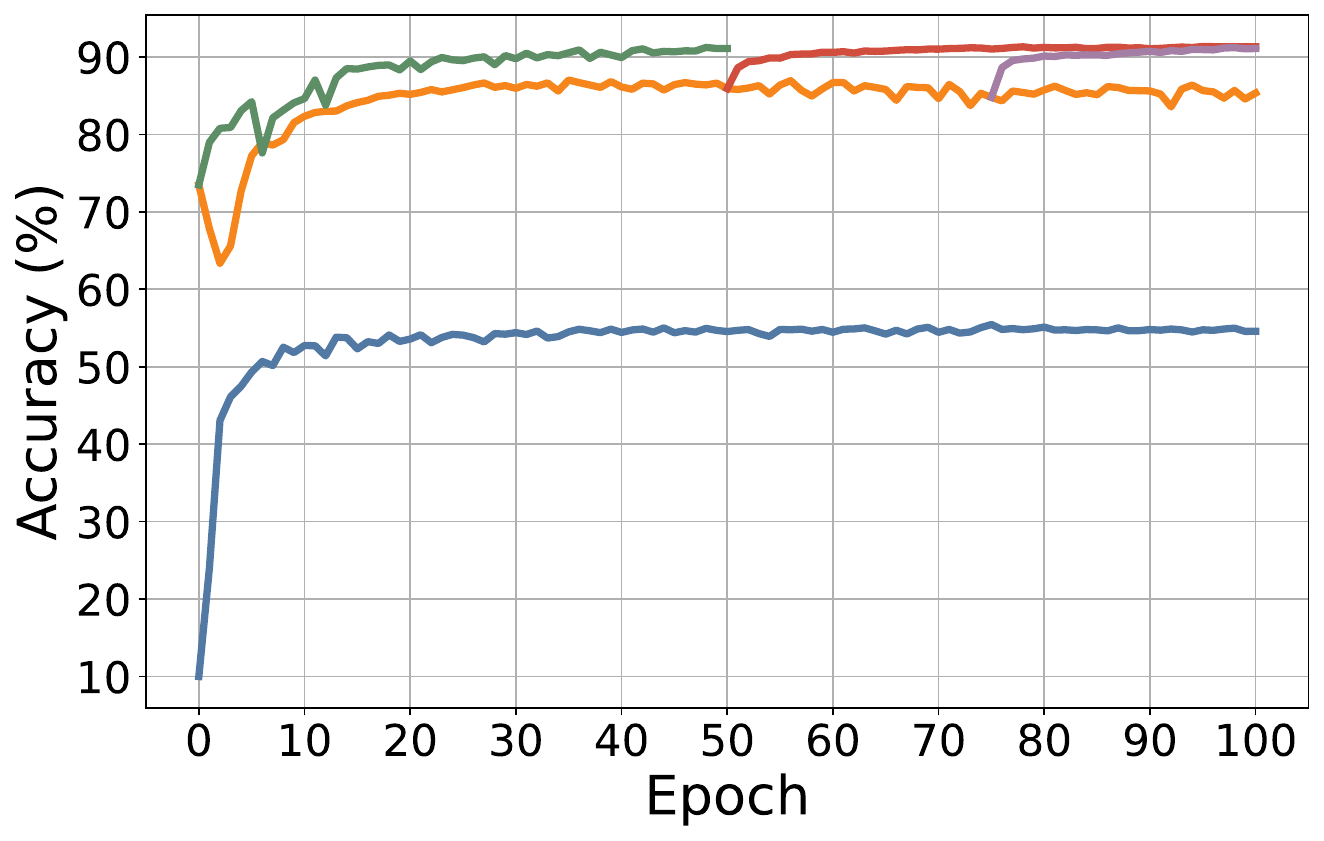}
  \caption{Test accuracy curves}
  \label{fig:gll_vs_mlp_acc}
\end{subfigure}  
\begin{subfigure}{\textwidth}
  \centering
  \includegraphics[width=\textwidth, trim={0 2cm 0 0},clip]{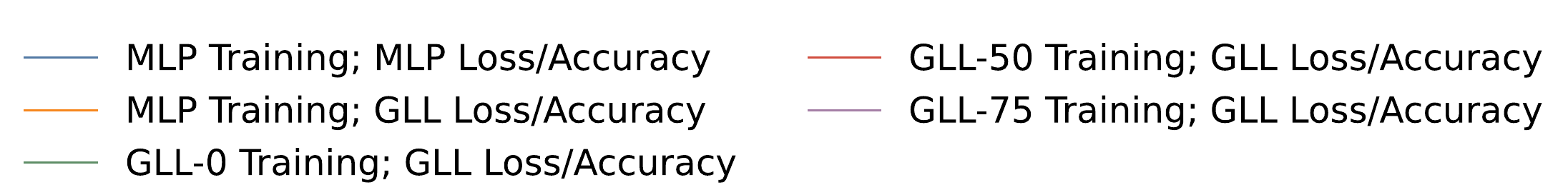}
\end{subfigure}
\caption{Training loss and test accuracy of four training strategies on the \textbf{FashionMNIST} dataset with the customized feature encoder in Table~\ref{tab:structure_mlp_gll}. The strategies are: (1) MLP trained from scratch, (2) GLL trained from scratch (GLL-0), (3) GLL-50 (GLL after 50 epochs of MLP training), and (4) GLL-75 (GLL after 75 epochs of MLP training). For the MLP-only strategy, both the MLP and GLL losses and accuracies are recorded during training; for the remaining strategies, only the GLL classifier is used.}
\label{fig:mlp_vs_gl}
\end{figure}

From Table~\ref{tab:network_comparison_mlp_vs_gll}, the MLP classifier trained from scratch performs poorly: after 100 epochs it still incurs a test error of 45.42\%, even though the GLL classifier reaches a much lower test error of 14.66\% when evaluated on the same encoder. This suggests that the bottleneck is not the representation learned by the encoder, but rather the optimization of the over-parameterized MLP classifier. In contrast, when we replace the MLP with the parameter-free GLL classifier and train the same backbone (GLL-0), the test error decreases to 8.90\%.

The mixed strategies GLL-50 and GLL-75 behave similarly. Although the MLP training tends to stagnate, switching to GLL mid-training yields stable decreases in the loss and further improvements in test error, reaching 8.68\% and 8.91\%, respectively. This indicates that once the encoder has learned a reasonable representation, GLL can immediately exploit it more effectively than the MLP classifier, and can also continue to drive smooth optimization of the encoder.

We summarize the main observations from this toy example:

\begin{enumerate}
\item \textit{Error under a fixed encoder.} Even when the encoder is trained only with the MLP classifier, the GLL achieves substantially lower loss and lower test error than the MLP on the same features. Simply replacing the MLP classifier by GLL can therefore yield a large performance gain without modifying the encoder architecture.

\item \textit{Robustness to over-parameterization of the classifier.} The over-parameterized MLP classifier is difficult to optimize and can suffer from issues such as vanishing gradients or dying ReLUs, leading to training stagnation. The GLL classifier, having no trainable parameters, avoids these optimization pathologies and works well even when the backbone is not carefully designed.

\item \textit{Smooth training dynamics.} When training starts with an MLP classifier and then switches to GLL (GLL-50 and GLL-75), the loss immediately resumes decreasing and the test accuracy improves further. This shows that GLL can effectively unlock stalled training and provide a smoother optimization trajectory for the shared encoder.
\end{enumerate}
}
%%%%%%%%%%%%%%%%%%%%%%%%%

\subsection{Large Scale Experiments}\label{sec:large_experiments}

\begin{table}[htbp]
  \centering
  \footnotesize
  \begin{tabular}{lrrr}
    \hline
    Network & \multicolumn{3}{c}{Params} \\
    \cline{2-4}
            & Encoder & MLP head & Total \\
    \hline
    VGG11 & 9,223,232 & 332,778 & 9,556,010 \\
    VGG13 & 9,407,936 & 332,778 & 9,740,714 \\
    % VGG16 & 14,718,912 & 332,778 & 15,051,690 \\
    % VGG19 & 20,029,888 & 332,778 & 20,362,666 \\
    ResNet20 & 271,824 & 16,938 & 288,762 \\
    ResNet32 & 466,256 & 16,938 & 483,194 \\
    ResNet44 & 660,688 & 16,938 & 677,626 \\
    ResNet56 & 855,120 & 16,938 & 872,058 \\
    ResNet110 & 1,730,064 & 16,938 & 1,747,002 \\
    ResNet18 & 11,168,832 & 332,778 & 11,501,610 \\
    % ResNet34 & 21,276,992 & 332,778 & 21,609,770 \\
    % ResNet50 & 23,500,352 & 4,463,082 & 27,963,434 \\
    PreActResNet18 & 11,171,146 & 332,778 & 11,503,924 \\
    % PreActResNet32 & 21,612,084 & 332,778 & 21,279,306 \\
    % PreActResNet50 & 23,509,066 & 332,778 & 27,972,148 \\
    % WRN-28-2 & 1,468,352 & 37,482 & 1,505,834 \\
    % WRN-28-8 & 23,357,792 & 332,778 & 23,690,570 \\
    \hline
  \end{tabular}
  \caption{Parameter counts by network}
  \label{tab:network_params_simple}
\end{table}

\begin{table}[t]
\centering
\small
\resizebox{\linewidth}{!}{%
\begin{tabular}{lrrrrrrrrrr}
\toprule
Train Num & \multicolumn{3}{c}{Whole (50000)} & \multicolumn{3}{c}{20\% (10000)} & \multicolumn{3}{c}{2\% (1000)} \\
\cmidrule(lr){1-1}\cmidrule(lr){2-4}\cmidrule(lr){5-7}\cmidrule(lr){8-10}
Network & MLP & WNLL & GLL & MLP & WNLL & GLL & MLP & WNLL & GLL \\
\midrule
VGG11 & \underline{7.77{\scriptsize$\,\pm\,$}{\scriptsize 0.22}} & 
         \underline{7.71{\scriptsize$\,\pm\,$}{\scriptsize 0.10}} & 
         \textbf{7.48{\scriptsize$\,\pm\,$}{\scriptsize 0.19}} & 
         15.16{\scriptsize$\,\pm\,$}{\scriptsize 0.58} & 
         14.74{\scriptsize$\,\pm\,$}{\scriptsize 0.27} & 
         \textbf{13.14{\scriptsize$\,\pm\,$}{\scriptsize 0.09}} & 
         34.76{\scriptsize$\,\pm\,$}{\scriptsize 1.38} & 
         34.97{\scriptsize$\,\pm\,$}{\scriptsize 1.57} & 
         \textbf{29.59{\scriptsize$\,\pm\,$}{\scriptsize 0.12}} \\
VGG13 & 6.16{\scriptsize$\,\pm\,$}{\scriptsize 0.31} & 
        6.20{\scriptsize$\,\pm\,$}{\scriptsize 0.23} & 
        \textbf{5.80{\scriptsize$\,\pm\,$}{\scriptsize 0.03}} & 
        12.68{\scriptsize$\,\pm\,$}{\scriptsize 0.48} & 
        12.27{\scriptsize$\,\pm\,$}{\scriptsize 0.19} & 
        \textbf{11.03{\scriptsize$\,\pm\,$}{\scriptsize 0.04}} & 
        31.61{\scriptsize$\,\pm\,$}{\scriptsize 1.17} & 
        31.31{\scriptsize$\,\pm\,$}{\scriptsize 0.49} & 
        \textbf{26.90{\scriptsize$\,\pm\,$}{\scriptsize 0.18}} \\
ResNet20 & 9.25{\scriptsize$\,\pm\,$}{\scriptsize 0.96} & 
           \underline{7.99{\scriptsize$\,\pm\,$}{\scriptsize 0.33}} & 
           \textbf{7.49{\scriptsize$\,\pm\,$}{\scriptsize 0.24}} & 
           15.60{\scriptsize$\,\pm\,$}{\scriptsize 0.74} & 
           \underline{14.76{\scriptsize$\,\pm\,$}{\scriptsize 0.68}} & 
           \textbf{14.00{\scriptsize$\,\pm\,$}{\scriptsize 0.31}} & 
           33.35{\scriptsize$\,\pm\,$}{\scriptsize 1.81} & 
           32.04{\scriptsize$\,\pm\,$}{\scriptsize 0.58} & 
           \textbf{31.25{\scriptsize$\,\pm\,$}{\scriptsize 0.19}} \\
ResNet32 & 8.45{\scriptsize$\,\pm\,$}{\scriptsize 0.96} & 
           7.26{\scriptsize$\,\pm\,$}{\scriptsize 0.29} & 
           \textbf{6.61{\scriptsize$\,\pm\,$}{\scriptsize 0.15}} & 
           14.78{\scriptsize$\,\pm\,$}{\scriptsize 0.58} & 
           13.63{\scriptsize$\,\pm\,$}{\scriptsize 0.24} & 
           \textbf{13.04{\scriptsize$\,\pm\,$}{\scriptsize 0.13}} & 
           33.93{\scriptsize$\,\pm\,$}{\scriptsize 1.76} & 
           \textbf{30.82{\scriptsize$\,\pm\,$}{\scriptsize 0.79}} & 
           \underline{32.96{\scriptsize$\,\pm\,$}{\scriptsize 1.44}} \\
ResNet44 & 8.07{\scriptsize$\,\pm\,$}{\scriptsize 0.78} & 
           6.57{\scriptsize$\,\pm\,$}{\scriptsize 0.11} & 
           \textbf{6.12{\scriptsize$\,\pm\,$}{\scriptsize 0.12}} & 
           14.38{\scriptsize$\,\pm\,$}{\scriptsize 0.82} & 
           13.23{\scriptsize$\,\pm\,$}{\scriptsize 0.33} & 
           \textbf{12.51{\scriptsize$\,\pm\,$}{\scriptsize 0.34}} & 
           \underline{33.76{\scriptsize$\,\pm\,$}{\scriptsize 2.07}} & 
           \underline{31.85{\scriptsize$\,\pm\,$}{\scriptsize 1.46}} & 
           \textbf{31.69{\scriptsize$\,\pm\,$}{\scriptsize 0.70}} \\
ResNet56 & 7.69{\scriptsize$\,\pm\,$}{\scriptsize 0.61} & 
           6.17{\scriptsize$\,\pm\,$}{\scriptsize 0.31} & 
           \textbf{5.62{\scriptsize$\,\pm\,$}{\scriptsize 0.07}} & 
           13.92{\scriptsize$\,\pm\,$}{\scriptsize 0.74} & 
           12.85{\scriptsize$\,\pm\,$}{\scriptsize 0.35} & 
           \textbf{12.02{\scriptsize$\,\pm\,$}{\scriptsize 0.09}} & 
           \underline{32.28{\scriptsize$\,\pm\,$}{\scriptsize 1.40}} & 
           \textbf{30.22{\scriptsize$\,\pm\,$}{\scriptsize 1.40}} & 
           \underline{30.37{\scriptsize$\,\pm\,$}{\scriptsize 0.04}} \\
ResNet110 & 6.89{\scriptsize$\,\pm\,$}{\scriptsize 0.46} & 
            5.71{\scriptsize$\,\pm\,$}{\scriptsize 0.18} & 
            \textbf{4.62{\scriptsize$\,\pm\,$}{\scriptsize 0.01}} & 
            14.11{\scriptsize$\,\pm\,$}{\scriptsize 1.21} & 
            \textbf{12.07{\scriptsize$\,\pm\,$}{\scriptsize 0.32}} & 
            \underline{12.15{\scriptsize$\,\pm\,$}{\scriptsize 0.11}} & 
            31.67{\scriptsize$\,\pm\,$}{\scriptsize 1.72} & 
            \underline{29.67{\scriptsize$\,\pm\,$}{\scriptsize 0.93}} & 
            \textbf{28.73{\scriptsize$\,\pm\,$}{\scriptsize 0.70}} \\
ResNet18 & 5.44{\scriptsize$\,\pm\,$}{\scriptsize 0.28} & 
           5.16{\scriptsize$\,\pm\,$}{\scriptsize 0.21} & 
           \textbf{4.57{\scriptsize$\,\pm\,$}{\scriptsize 0.13}} & 
           11.51{\scriptsize$\,\pm\,$}{\scriptsize 0.44} & 
           10.53{\scriptsize$\,\pm\,$}{\scriptsize 0.18} & 
           \textbf{9.44{\scriptsize$\,\pm\,$}{\scriptsize 0.06}} & 
           30.27{\scriptsize$\,\pm\,$}{\scriptsize 1.17} & 
           28.83{\scriptsize$\,\pm\,$}{\scriptsize 0.76} & 
           \textbf{25.05{\scriptsize$\,\pm\,$}{\scriptsize 0.35}} \\
PreResNet18 & 5.42{\scriptsize$\,\pm\,$}{\scriptsize 0.32} & 
              5.02{\scriptsize$\,\pm\,$}{\scriptsize 0.18} & 
              \textbf{4.48{\scriptsize$\,\pm\,$}{\scriptsize 0.01}} & 
              11.37{\scriptsize$\,\pm\,$}{\scriptsize 0.46} & 
              10.94{\scriptsize$\,\pm\,$}{\scriptsize 0.32} & 
              \textbf{9.89{\scriptsize$\,\pm\,$}{\scriptsize 0.01}} & 
              29.75{\scriptsize$\,\pm\,$}{\scriptsize 1.05} & 
              29.24{\scriptsize$\,\pm\,$}{\scriptsize 1.31} & 
              \textbf{25.23{\scriptsize$\,\pm\,$}{\scriptsize 0.07}} \\
\bottomrule
\end{tabular}
}
\caption{\blue CIFAR-10 test error rates (\%, lower is better) comparing MLP, WNLL, and GLL output heads across architectures, evaluated under three training-set sizes: “Whole” (50000 images), “20\%” (10000), and “2\%” (1000). Entries report mean $\pm$ one standard deviation of the test error. Best results for each backbone and label rate are \textbf{bolded} and results within a standard deviation of the best result are \underline{underlined}. Network parameter counts follow Table~\ref{tab:network_params_simple}.}
\label{tab:cifar10_largescale}
\end{table}

\begin{table}[t]
\centering
\small
\resizebox{\linewidth}{!}{%

\begin{tabular}{lrrrrrrrrrr}
\toprule
Train Num & \multicolumn{3}{c}{Whole (112800)} & \multicolumn{3}{c}{20\% (22560)} & \multicolumn{3}{c}{2\% (2256)} \\
\cmidrule(lr){1-1}\cmidrule(lr){2-4}\cmidrule(lr){5-7}\cmidrule(lr){8-10}
Network & MLP & WNLL & GLL & MLP & WNLL & GLL & MLP & WNLL & GLL \\
\midrule
VGG11 & \textbf{10.72{\scriptsize$\,\pm\,$}{\scriptsize 0.14}} & 
        12.18{\scriptsize$\,\pm\,$}{\scriptsize 0.79} & 
        \underline{10.76{\scriptsize$\,\pm\,$}{\scriptsize 0.09}} & 
        \underline{12.06{\scriptsize$\,\pm\,$}{\scriptsize 0.13}} & 
        13.46{\scriptsize$\,\pm\,$}{\scriptsize 0.69} & 
        \textbf{11.96{\scriptsize$\,\pm\,$}{\scriptsize 0.08}} & 
        15.86{\scriptsize$\,\pm\,$}{\scriptsize 0.35} & 
        16.64{\scriptsize$\,\pm\,$}{\scriptsize 0.78} & 
        \textbf{14.98{\scriptsize$\,\pm\,$}{\scriptsize 0.13}} \\
VGG13 & \textbf{10.50{\scriptsize$\,\pm\,$}{\scriptsize 0.06}} & 
        10.73{\scriptsize$\,\pm\,$}{\scriptsize 0.15} & 
        \underline{10.60{\scriptsize$\,\pm\,$}{\scriptsize 0.08}} & 
        \textbf{11.85{\scriptsize$\,\pm\,$}{\scriptsize 0.08}} & 
        15.43{\scriptsize$\,\pm\,$}{\scriptsize 1.15} & 
        \underline{11.94{\scriptsize$\,\pm\,$}{\scriptsize 0.10}} & 
        15.91{\scriptsize$\,\pm\,$}{\scriptsize 0.38} & 
        16.48{\scriptsize$\,\pm\,$}{\scriptsize 0.82} & 
        \textbf{15.08{\scriptsize$\,\pm\,$}{\scriptsize 0.20}} \\
ResNet20 & \textbf{9.72{\scriptsize$\,\pm\,$}{\scriptsize 0.07}} & 
           10.77{\scriptsize$\,\pm\,$}{\scriptsize 0.29} & 
           10.34{\scriptsize$\,\pm\,$}{\scriptsize 0.05} & 
           \textbf{11.45{\scriptsize$\,\pm\,$}{\scriptsize 0.06}} & 
           12.33{\scriptsize$\,\pm\,$}{\scriptsize 0.21} & 
           11.60{\scriptsize$\,\pm\,$}{\scriptsize 0.07} & 
           15.79{\scriptsize$\,\pm\,$}{\scriptsize 0.33} & 
           17.10{\scriptsize$\,\pm\,$}{\scriptsize 1.01} & 
           \textbf{14.63{\scriptsize$\,\pm\,$}{\scriptsize 0.12}} \\
ResNet32 & \textbf{9.73{\scriptsize$\,\pm\,$}{\scriptsize 0.06}} & 
           10.78{\scriptsize$\,\pm\,$}{\scriptsize 0.37} & 
           10.15{\scriptsize$\,\pm\,$}{\scriptsize 0.14} & 
           \textbf{11.41{\scriptsize$\,\pm\,$}{\scriptsize 0.05}} & 
           12.26{\scriptsize$\,\pm\,$}{\scriptsize 0.38} & 
           \underline{11.54{\scriptsize$\,\pm\,$}{\scriptsize 0.11}} & 
           15.85{\scriptsize$\,\pm\,$}{\scriptsize 0.46} & 
           17.52{\scriptsize$\,\pm\,$}{\scriptsize 0.90} & 
           \textbf{14.30{\scriptsize$\,\pm\,$}{\scriptsize 0.08}} \\
ResNet44 & \textbf{9.92{\scriptsize$\,\pm\,$}{\scriptsize 0.06}} & 
           11.30{\scriptsize$\,\pm\,$}{\scriptsize 0.36} & 
           10.10{\scriptsize$\,\pm\,$}{\scriptsize 0.09} & 
           \underline{11.55{\scriptsize$\,\pm\,$}{\scriptsize 0.16}} & 
           12.04{\scriptsize$\,\pm\,$}{\scriptsize 0.28} & 
           \textbf{11.43{\scriptsize$\,\pm\,$}{\scriptsize 0.14}} & 
           15.87{\scriptsize$\,\pm\,$}{\scriptsize 0.35} & 
           17.81{\scriptsize$\,\pm\,$}{\scriptsize 0.98} & 
           \textbf{14.98{\scriptsize$\,\pm\,$}{\scriptsize 0.18}} \\
ResNet56 & \underline{9.98{\scriptsize$\,\pm\,$}{\scriptsize 0.12}} & 
           10.80{\scriptsize$\,\pm\,$}{\scriptsize 0.44} & 
           \textbf{9.90{\scriptsize$\,\pm\,$}{\scriptsize 0.12}} & 
           \underline{11.76{\scriptsize$\,\pm\,$}{\scriptsize 0.14}} & 
           12.24{\scriptsize$\,\pm\,$}{\scriptsize 0.29} & 
           \textbf{11.60{\scriptsize$\,\pm\,$}{\scriptsize 0.08}} & 
           15.74{\scriptsize$\,\pm\,$}{\scriptsize 0.30} & 
           17.81{\scriptsize$\,\pm\,$}{\scriptsize 0.97} & 
           \textbf{14.61{\scriptsize$\,\pm\,$}{\scriptsize 0.21}} \\
ResNet110 & 10.47{\scriptsize$\,\pm\,$}{\scriptsize 0.09} & 
            10.74{\scriptsize$\,\pm\,$}{\scriptsize 0.24} & 
            \textbf{10.19{\scriptsize$\,\pm\,$}{\scriptsize 0.11}} & 
            \textbf{11.83{\scriptsize$\,\pm\,$}{\scriptsize 0.12}} & 
            13.49{\scriptsize$\,\pm\,$}{\scriptsize 0.54} & 
            \underline{11.99{\scriptsize$\,\pm\,$}{\scriptsize 0.07}} & 
            15.86{\scriptsize$\,\pm\,$}{\scriptsize 0.44} & 
            16.81{\scriptsize$\,\pm\,$}{\scriptsize 0.89} & 
            \textbf{14.34{\scriptsize$\,\pm\,$}{\scriptsize 0.19}} \\
ResNet18 & \underline{10.57{\scriptsize$\,\pm\,$}{\scriptsize 0.15}} & 
           11.04{\scriptsize$\,\pm\,$}{\scriptsize 0.18} & 
           \textbf{10.32{\scriptsize$\,\pm\,$}{\scriptsize 0.12}} & 
           \underline{11.76{\scriptsize$\,\pm\,$}{\scriptsize 0.10}} & 
           12.57{\scriptsize$\,\pm\,$}{\scriptsize 0.43} & 
           \textbf{11.66{\scriptsize$\,\pm\,$}{\scriptsize 0.07}} & 
           16.21{\scriptsize$\,\pm\,$}{\scriptsize 0.11} & 
           17.66{\scriptsize$\,\pm\,$}{\scriptsize 0.83} & 
           \textbf{14.75{\scriptsize$\,\pm\,$}{\scriptsize 0.09}} \\
PreResNet18 & 10.57{\scriptsize$\,\pm\,$}{\scriptsize 0.09} & 
              10.79{\scriptsize$\,\pm\,$}{\scriptsize 0.17} & 
              \textbf{10.37{\scriptsize$\,\pm\,$}{\scriptsize 0.06}} & 
              \underline{11.80{\scriptsize$\,\pm\,$}{\scriptsize 0.12}} & 
              12.24{\scriptsize$\,\pm\,$}{\scriptsize 0.14} & 
              \textbf{11.68{\scriptsize$\,\pm\,$}{\scriptsize 0.11}} & 
              15.72{\scriptsize$\,\pm\,$}{\scriptsize 0.40} & 
              16.77{\scriptsize$\,\pm\,$}{\scriptsize 0.64} & 
              \textbf{14.71{\scriptsize$\,\pm\,$}{\scriptsize 0.13}} \\
\bottomrule
\end{tabular}

}
\caption{\blue EMNIST test error rates (\%, lower is better) comparing MLP, WNLL, and GLL output heads across architectures, evaluated under three training-set sizes: “Whole” (112800 images), “20\%” (22560), and “2\%” (2256). “Whole” uses the full standard training set. Best results for each backbone and label rate are \textbf{bolded} and results within a standard deviation of the best result are \underline{underlined}. Entries report mean $\pm$ one standard deviation of the test error. Network parameter counts follow Table~\ref{tab:network_params_simple}. \nc}
\label{tab:emnist_largescale}
\end{table}

\begin{table}[t]
\centering
\small
\resizebox{\linewidth}{!}{%

\begin{tabular}{lrrrrrrrrrr}
\toprule
Train Num & \multicolumn{3}{c}{Whole (112800)} & \multicolumn{3}{c}{20\% (22560)} & \multicolumn{3}{c}{2\% (2256)} \\
\cmidrule(lr){1-1}\cmidrule(lr){2-4}\cmidrule(lr){5-7}\cmidrule(lr){8-10}
Network & MLP & WNLL & GLL & MLP & WNLL & GLL & MLP & WNLL & GLL \\
\midrule
VGG11 & \textbf{35.88{\scriptsize$\,\pm\,$}{\scriptsize 0.48}} & 
        \underline{36.92{\scriptsize$\,\pm\,$}{\scriptsize 1.15}} & 
        \underline{35.91{\scriptsize$\,\pm\,$}{\scriptsize 0.28}} & 
        38.51{\scriptsize$\,\pm\,$}{\scriptsize 0.42} & 
        \textbf{36.57{\scriptsize$\,\pm\,$}{\scriptsize 0.55}} & 
        38.10{\scriptsize$\,\pm\,$}{\scriptsize 0.18} & 
        \underline{42.91{\scriptsize$\,\pm\,$}{\scriptsize 0.76}} & 
        45.05{\scriptsize$\,\pm\,$}{\scriptsize 0.33} & 
        \textbf{42.14{\scriptsize$\,\pm\,$}{\scriptsize 0.33}} \\
VGG13 & \textbf{34.88{\scriptsize$\,\pm\,$}{\scriptsize 0.32}} & 
        36.05{\scriptsize$\,\pm\,$}{\scriptsize 0.37} & 
        \underline{35.38{\scriptsize$\,\pm\,$}{\scriptsize 0.34}} & 
        \textbf{37.90{\scriptsize$\,\pm\,$}{\scriptsize 0.23}} & 
        \underline{38.68{\scriptsize$\,\pm\,$}{\scriptsize 0.74}} & 
        \underline{38.11{\scriptsize$\,\pm\,$}{\scriptsize 0.33}} & 
        44.51{\scriptsize$\,\pm\,$}{\scriptsize 1.48} & 
        45.28{\scriptsize$\,\pm\,$}{\scriptsize 0.53} & 
        \textbf{42.29{\scriptsize$\,\pm\,$}{\scriptsize 0.36}} \\
ResNet20 & \textbf{31.66{\scriptsize$\,\pm\,$}{\scriptsize 0.21}} & 
           35.76{\scriptsize$\,\pm\,$}{\scriptsize 0.59} & 
           33.56{\scriptsize$\,\pm\,$}{\scriptsize 0.38} & 
           \underline{35.80{\scriptsize$\,\pm\,$}{\scriptsize 0.59}} & 
           \underline{36.00{\scriptsize$\,\pm\,$}{\scriptsize 0.36}} & 
           \textbf{35.53{\scriptsize$\,\pm\,$}{\scriptsize 0.23}} & 
           43.17{\scriptsize$\,\pm\,$}{\scriptsize 1.29} & 
           46.61{\scriptsize$\,\pm\,$}{\scriptsize 0.96} & 
           \textbf{41.45{\scriptsize$\,\pm\,$}{\scriptsize 0.41}} \\
ResNet32 & \textbf{31.69{\scriptsize$\,\pm\,$}{\scriptsize 0.29}} & 
           34.47{\scriptsize$\,\pm\,$}{\scriptsize 0.81} & 
           32.87{\scriptsize$\,\pm\,$}{\scriptsize 0.49} & 
           \textbf{36.09{\scriptsize$\,\pm\,$}{\scriptsize 0.32}} & 
           \underline{36.72{\scriptsize$\,\pm\,$}{\scriptsize 0.60}} & 
           \underline{36.09{\scriptsize$\,\pm\,$}{\scriptsize 0.42}} & 
           43.66{\scriptsize$\,\pm\,$}{\scriptsize 1.35} & 
           45.80{\scriptsize$\,\pm\,$}{\scriptsize 0.62} & 
           \textbf{39.94{\scriptsize$\,\pm\,$}{\scriptsize 0.33}} \\
ResNet44 & \textbf{31.84{\scriptsize$\,\pm\,$}{\scriptsize 0.27}} & 
           35.61{\scriptsize$\,\pm\,$}{\scriptsize 0.56} & 
           32.83{\scriptsize$\,\pm\,$}{\scriptsize 0.25} & 
           \underline{37.01{\scriptsize$\,\pm\,$}{\scriptsize 0.56}} & 
           \underline{36.47{\scriptsize$\,\pm\,$}{\scriptsize 0.44}} & 
           \textbf{35.98{\scriptsize$\,\pm\,$}{\scriptsize 0.57}} & 
           43.70{\scriptsize$\,\pm\,$}{\scriptsize 1.44} & 
           46.68{\scriptsize$\,\pm\,$}{\scriptsize 0.39} & 
           \textbf{41.51{\scriptsize$\,\pm\,$}{\scriptsize 0.51}} \\
ResNet56 & \underline{32.44{\scriptsize$\,\pm\,$}{\scriptsize 0.45}} & 
           34.62{\scriptsize$\,\pm\,$}{\scriptsize 1.07} & 
           \textbf{32.31{\scriptsize$\,\pm\,$}{\scriptsize 0.46}} & 
           \underline{37.41{\scriptsize$\,\pm\,$}{\scriptsize 0.41}} & 
           \textbf{36.89{\scriptsize$\,\pm\,$}{\scriptsize 0.72}} & 
           \underline{36.93{\scriptsize$\,\pm\,$}{\scriptsize 0.38}} & 
           43.99{\scriptsize$\,\pm\,$}{\scriptsize 1.25} & 
           45.65{\scriptsize$\,\pm\,$}{\scriptsize 0.79} & 
           \textbf{41.28{\scriptsize$\,\pm\,$}{\scriptsize 0.70}} \\
ResNet110 & \underline{34.03{\scriptsize$\,\pm\,$}{\scriptsize 0.53}} & 
            35.03{\scriptsize$\,\pm\,$}{\scriptsize 0.81} & 
            \textbf{33.22{\scriptsize$\,\pm\,$}{\scriptsize 0.29}} & 
            \underline{37.48{\scriptsize$\,\pm\,$}{\scriptsize 0.36}} & 
            \textbf{36.77{\scriptsize$\,\pm\,$}{\scriptsize 0.57}} & 
            38.25{\scriptsize$\,\pm\,$}{\scriptsize 0.30} & 
            44.98{\scriptsize$\,\pm\,$}{\scriptsize 1.22} & 
            43.70{\scriptsize$\,\pm\,$}{\scriptsize 0.65} & 
            \textbf{40.66{\scriptsize$\,\pm\,$}{\scriptsize 0.70}} \\
ResNet18 & \underline{35.51{\scriptsize$\,\pm\,$}{\scriptsize 0.46}} & 
           \underline{35.90{\scriptsize$\,\pm\,$}{\scriptsize 0.54}} & 
           \textbf{35.04{\scriptsize$\,\pm\,$}{\scriptsize 0.44}} & 
           38.11{\scriptsize$\,\pm\,$}{\scriptsize 0.23} & 
           \textbf{36.59{\scriptsize$\,\pm\,$}{\scriptsize 0.56}} & 
           37.71{\scriptsize$\,\pm\,$}{\scriptsize 0.22} & 
           44.77{\scriptsize$\,\pm\,$}{\scriptsize 0.54} & 
           47.10{\scriptsize$\,\pm\,$}{\scriptsize 1.95} & 
           \textbf{41.32{\scriptsize$\,\pm\,$}{\scriptsize 0.19}} \\
PreResNet18 & \underline{35.70{\scriptsize$\,\pm\,$}{\scriptsize 0.38}} & 
              35.95{\scriptsize$\,\pm\,$}{\scriptsize 0.50} & 
              \textbf{35.23{\scriptsize$\,\pm\,$}{\scriptsize 0.12}} & 
              \underline{38.30{\scriptsize$\,\pm\,$}{\scriptsize 0.57}} & 
              \underline{38.48{\scriptsize$\,\pm\,$}{\scriptsize 0.35}} & 
              \textbf{38.00{\scriptsize$\,\pm\,$}{\scriptsize 0.23}} & 
              44.27{\scriptsize$\,\pm\,$}{\scriptsize 1.25} & 
              45.02{\scriptsize$\,\pm\,$}{\scriptsize 0.80} & 
              \textbf{40.87{\scriptsize$\,\pm\,$}{\scriptsize 0.64}} \\
\bottomrule
\end{tabular}

}
\caption{\blue EMNIST test error rates (\%, lower is better) for MLP, WNLL, and GLL output heads across architectures, evaluated under three training-set sizes: “Whole” (112800 images), “20\%” (22560), and “2\%” (2256). “Whole” uses the full standard training set. Error rates are computed on a class subindex of the EMNIST (balanced) label space: although the dataset contains 47 classes in total, we evaluate only the 10 hardest classes: 0, 1, 9, F, I, L, O, f, g, and q.
%with indices \{0, 1, 9, 15, 18, 21, 24, 40, 41, 44\}. 
Entries report mean $\pm$ one standard deviation of the test error. Best results for each backbone and label rate are \textbf{bolded} and results within a standard deviation of the best result are \underline{underlined}. Network parameter counts follow Table~\ref{tab:network_params_simple}.}
\label{tab:emnist_subindex}
\end{table}

%%%%%%%%%%%%%%%

{\blue This section presents large scale experiments that compare three interchangeable classifiers under the same backbone feature encoder: (1) a two-layer MLP with softmax, (2) WNLL (\cite{wang2021graph}), and (3) the proposed GLL. We chose (1) as a comparison point because GLL is intended to be a drop-in replacement for the standard MLP-based classification head, while (2) serves as a good intermediary comparison between MLP and GLL because it uses a graph-based classifier in testing. However, WNLL differs significantly from GLL in training. First, it uses an alternating two-stage training protocol, where Stage 1 consists of normal MLP training, and Stage 2 is WNLL training. During Stage 2, the graph classification head computes a loss, but backpropagation is handled by approximating the gradients with a linear layer. Second, Stage 1 lasts for 400 epochs at a time, while Stage 2 lasts for 5 epochs at a time. Finally, Stage 2 freezes the parameters of all but two layers of the network. In this sense, WNLL is closer to MLP in training than GLL. 

% We follow a two-stage training protocol: first, contrastive pretraining on the backbone without any classification head; second, supervised end-to-end training with the cross-entropy loss after attaching one of the three heads. 
As discussed in Section \ref{sec:implementation}, all models are first warmed up with contrastive pretraining on the backbone without any classification head. Then, we proceed with supervised end-to-end training after attaching one of the three heads.  
We report the test error on CIFAR-10 and EMNIST-Balanced (\cite{cohen2017emnist}) datasets. EMNIST extends MNIST by adding handwritten upper and lowercase letters. The Balanced version contains 131,600 samples and 47 classes\footnote{Some of the upper and lowercase letters are merged into one class because they are the same shape, for example ``s'' and ``S''.}. Following a similar experimental protocol as \cite{wang2021graph}, we evaluate multiple feature encoders including variants of VGG (\cite{simonyan2014very}), ResNet (\cite{he2016deep}), and PreActResNet (\cite{he2016identity}) and three label fractions: full training set, 20 percent, and 2 percent based on class-balanced sampling, since graph-based classifiers are a standard choice in semi-supervised settings. Parameter counts for these networks are listed in Table \ref{tab:network_params_simple}. 

For CIFAR-10, we train for 1000 epochs. Recalling the definitions from Section \ref{sec:implementation}, we use a batch of $1500$ points in training and testing, which includes 250 base labeled points $\mathcal{L}_b$ for GLL. The remaining 1250 points are constructed with labeled ($\mathcal{B}_l$) and unlabeled ($\mathcal{B}_u$) points according to Equation \ref{eq:N_l_and_N_u}. This batch partition also applies to WNLL during WNLL's Stage 2 ($\mathcal{L}_b$ is referred to as ``template'' labels in \cite{wang2021graph}). 
We use SGD with a learning rate of $0.001$, momentum of $0.9$, and weight decay of $0.0005$, with cross entropy loss. We use the standard implementations of all the neural network backbones. On EMNIST, we use similar training settings, except we train for 200 epochs and use a learning rate of $0.01$, and the base sample size is increased to 470 to account for the additional number of classes. All experiments were run on a single A100 GPU. For both datasets, we use the data augmentation strategy described in Section \ref{sec:implementation}.

% \harris{maybe Bohan ran it on a different kind of machine? Bohan: this is fine. Do not need to clarify}
 
% base label size of 
% batch size of
% lr
% sgd optimizer with weight decay and momentum

% 200 epochs for emnist

Test error rates for CIFAR-10 and EMNIST across architectures and supervision levels are given in Tables \ref{tab:cifar10_largescale} and \ref{tab:emnist_largescale}. On CIFAR-10, the GLL-based models are consistently the best performing of the three classification heads, especially as training set size decreases. This empirically validates the advantage of the relational information encoded and exploited via a GLL, as well as the precise tracking of gradients. This becomes even more essential for lower label rates. 

On EMNIST, GLL is better than WNLL across label rates and MLP for low label rates, and on par with MLP in the other training settings. As expected, the GLL-based networks see the largest advantage in the semi-supervsied setting (2\%). 

During our experiments on EMNIST, we found that the majority of the classes were easily classified by all methods, while a handful of other classes were consistently more challenging. By isolating results on only these trickier classes, we can better elucidate differences between models. To this end, we further report EMNIST test error on \textit{only the ten hardest classes} in Table \ref{tab:emnist_subindex}. These correspond to handwritten digits and letters that are difficult to distinguish without context\footnote{Specifically, the classes ``0'' and ``O''; ``1'', ``I'', and ``L'';  ``9'', ``g'', and ``q''; and ``F'' and ``f''.}. Here, the advantage of GLL is more pronounced in the 2\% training case.

Overall, the results show that using a GLL classification head consistently performs on par with or better than MLP or WNLL-based classifiers. The benefit of a GLL is particularly evident at low label rates. To complement these results, Section \ref{sec:adversarial} shows that a GLL offers substantial robustness gains under adversarial perturbations.

}

% \begin{table}[t]
% \centering
% \small
% \begin{tabular}{lrrrrrrrrrr}
% \toprule
% & \multicolumn{3}{c}{Whole} & \multicolumn{3}{c}{10,000} & \multicolumn{3}{c}{1000} \\
% \cmidrule(lr){2-4}\cmidrule(lr){5-7}\cmidrule(lr){8-10}
% Network & MLP & WNLL & GLL & MLP & WNLL & GLL & MLP & WNLL & GLL \\
% \midrule
% VGG11      &  &  & 7.29 &  &  & 13.21 &  &  & 29.71  \\
% VGG13      &  &  & 5.84 &  &  & 10.99 &  &  & 27.08 \\
% % VGG16      & 6.72  & \textbf{5.69} & 7.29 & 9.01  & \textbf{7.54} & 25.41 & \textbf{22.23} \\
% % VGG19      & 6.95  & \textbf{5.92} & 7.99 & 9.62  & \textbf{8.09} & 25.70 & \textbf{22.87} \\
% ResNet20   & 8.64 &  & 7.25 & 14.38 &  & 13.69 & 33.07 &  & 31.06 \\
% ResNet32   &  &  & 6.76 &  &  & 12.91 & 33.34 &  & 31.52 \\
% ResNet44   &  &  & 6.00 &  &  & 12.86 &  &  & 30.99 \\
% ResNet56   &  &  & 5.69 &  &  & 11.93 &  &  & 30.42 \\
% ResNet110  &  &  & 4.63 &  &  & 12.04 &  &  & 28.03 \\
% ResNet18   & 5.46 & & 3.90 & 11.09 &  &  & 32.81 &  & \\
% PreResNet18 & 5.65 &  & 4.50 & 10.99 &  & 9.89 & 29.25 &  & 25.11 \\
% % ResNet34   & 5.93  & \textbf{4.26} & 6.32 & 8.31  & \textbf{6.11} & 26.47 & \textbf{20.27} \\
% % ResNet50   & 6.24  & \textbf{4.17} & 6.63 & 9.64  & \textbf{6.49} & 29.69 & \textbf{20.19} \\
% % PreResNet34& 6.08  & \textbf{4.40} & 5.88 & 8.52  & \textbf{6.34} & 23.56 & \textbf{19.02} \\
% % PreResNet50& 6.05  & \textbf{4.27} & 5.91 & 9.18  & \textbf{6.05} & 25.05 & \textbf{18.61} \\
% \bottomrule
% \end{tabular}
% \caption{Error rates (\%) comparing Vanilla, WNLL, and GLL heads across architectures and training set sizes.}
% \end{table}

%%%%%%%%%%%%%%%%%%%%%%

\subsection{Training and Inference Times}\label{sec:training_inference_times}

% \harris{move to after large scale experiments}

\begin{table}[!ht]
\centering
\begin{tabular}{llcc}
\toprule
Batch Size & Metric & GLL & MLP \\
\midrule
\multirow{3}{*}{128} 
  & Training time / batch (s) & $0.0457 {\pm 0.0093}$ & $0.0044 {\pm 0.0006}$ \\
  & Testing time / batch (s)  & $0.0308 {\pm 0.0085}$ & $0.0024 {\pm 0.0027}$ \\
  & Testing error (\%)        & 0.76                  & 2.11                  \\
\midrule
\multirow{3}{*}{256} 
  & Training time / batch (s) & $0.0685 {\pm 0.0146}$ & $0.0061 {\pm 0.0005}$ \\
  & Testing time / batch (s)  & $0.0457 {\pm 0.0108}$ & $0.0039 {\pm 0.0005}$ \\
  & Testing error (\%)        & 1.33                  & 1.69                  \\
\midrule
\multirow{3}{*}{512} 
  & Training time / batch (s) & $0.1244 {\pm 0.0280}$ & $0.0094 {\pm 0.0007}$ \\
  & Testing time / batch (s)  & $0.0881 {\pm 0.0206}$ & $0.0072 {\pm 0.0006}$ \\
  & Testing error (\%)        & 1.27                  & 1.53                  \\
\midrule
\multirow{3}{*}{1024} 
  & Training time / batch (s) & $0.2447 {\pm 0.0606}$ & $0.0167 {\pm 0.0012}$ \\
  & Testing time / batch (s)  & $0.1556 {\pm 0.0310}$ & $0.0139 {\pm 0.0009}$ \\
  & Testing error (\%)        & 1.26                  & 4.65                  \\
\bottomrule
\end{tabular}
\caption{Comparison of GLL and MLP performance across different batch sizes on MNIST. The architecture backbone is a four-layer CNN with two linear layers. We use identical training settings for both models. For time metrics, each entry is reported as mean $\pm$ standard deviation over five epochs of training. Testing error is reported in percentage points. The experiments were run on a T4 GPU.}
\label{tab:gll_mlp_compute_cnn}
\end{table}

\begin{table}[!ht]
\centering
\begin{tabular}{llcc}
\toprule
Batch Size & Metric & GLL & MLP \\
\midrule
\multirow{3}{*}{128} 
  & Training time / batch (s) & $0.1537 {\pm 0.0121}$ & $0.0597 {\pm 0.0083}$ \\
  & Testing time / batch (s)  & $0.0659 {\pm 0.0068}$ & $0.0220 {\pm 0.0026}$ \\
  & Testing error (\%)        & 0.84                  & 0.86                  \\
\midrule
\multirow{3}{*}{256} 
  & Training time / batch (s) & $0.2295 {\pm 0.0193}$ & $0.1157 {\pm 0.0168}$ \\
  & Testing time / batch (s)  & $0.1045 {\pm 0.0162}$ & $0.0426 {\pm 0.0060}$ \\
  & Testing error (\%)        & 0.94                  & 1.04                  \\
\midrule
\multirow{3}{*}{512} 
  & Training time / batch (s) & $0.3851 {\pm 0.0362}$ & $0.2292 {\pm 0.0323}$ \\
  & Testing time / batch (s)  & $0.1834 {\pm 0.0236}$ & $0.0924 {\pm 0.0102}$ \\
  & Testing error (\%)        & 0.97                  & 1.34                  \\
\midrule
\multirow{3}{*}{1024} 
  & Training time / batch (s) & $0.7129 {\pm 0.0773}$ & $0.4676 {\pm 0.0734}$ \\
  & Testing time / batch (s)  & $0.3404 {\pm 0.0352}$ & $0.2032 {\pm 0.0340}$ \\
  & Testing error (\%)        & 0.94                  & 2.46                  \\
\bottomrule
\end{tabular}
\caption{Comparison of GLL and MLP performance across different batch sizes on MNIST. The architecture backbone is ResNet-18. We use identical training settings for both models. For time metrics, each entry is reported as mean $\pm$ standard deviation over five epochs of training. Testing error is reported in percentage points. The experiments were run on a T4 GPU.}
\label{tab:gll_mlp_compute_resnet18}
\end{table}

\begin{table}[!ht]
\centering
\begin{tabular}{llcc}
\toprule
Batch Size & Metric & GLL & MLP \\
\midrule
\multirow{3}{*}{128} 
  & Training time / batch (s) & $0.1257 {\pm 0.0044}$ & $0.0471 {\pm 0.0024}$ \\
  & Testing time / batch (s)  & $0.0629 {\pm 0.0035}$ & $0.0224 {\pm 0.0037}$ \\
  & Testing error (\%)        & 1.65                  & 2.39                  \\
\midrule
\multirow{3}{*}{256} 
  & Training time / batch (s) & $0.1772 {\pm 0.0065}$ & $0.0725 {\pm 0.0028}$ \\
  & Testing time / batch (s)  & $0.0950 {\pm 0.0103}$ & $0.0404 {\pm 0.0043}$ \\
  & Testing error (\%)        & 2.10                  & 1.18                  \\
\midrule
\multirow{3}{*}{512} 
  & Training time / batch (s) & $0.2950 {\pm 0.0218}$ & $0.1480 {\pm 0.0637}$ \\
  & Testing time / batch (s)  & $0.1605 {\pm 0.0143}$ & $0.0774 {\pm 0.0074}$ \\
  & Testing error (\%)        & 1.14                  & 1.13                  \\
\midrule
\multirow{3}{*}{1024} 
  & Training time / batch (s) & $0.5214 {\pm 0.0605}$ & $0.2802 {\pm 0.0224}$ \\
  & Testing time / batch (s)  & $0.2919 {\pm 0.0208}$ & $0.1498 {\pm 0.0102}$ \\
  & Testing error (\%)        & 1.93                  & 1.57                  \\
\bottomrule
\end{tabular}
\caption{Comparison of GLL and MLP performance across different batch sizes on MNIST. The architecture backbone is ResNet-110. We use identical training settings for both models. For time metrics, each entry is reported as mean $\pm$ standard deviation over five epochs of training. Testing error is reported in percentage points. The experiments were run on an A100 GPU.}
\label{tab:gll_mlp_compute_resnet110}
\end{table}

{\blue 
To emphasize the practicality of our method, we present a side-by-side comparison of training and inference times of several GLL- and MLP-based models on the MNIST dataset in Tables \ref{tab:gll_mlp_compute_cnn}, \ref{tab:gll_mlp_compute_resnet18}, and \ref{tab:gll_mlp_compute_resnet110}. We train each model for 5 epochs with varying batch sizes and report the average time per batch in training and testing. For GLL, we use a base labeled set of 100 points (10 per class). All other training hyperparameters are identical throughout. To ensure a fair comparison, we also report test error alongside training and inference time. Across all architectures and batch sizes, GLL attains test errors that are comparable to or better than those of the corresponding MLP baselines. For the small CNN backbone, GLL consistently achieves noticeably lower error than MLP, particularly at large batch sizes. For ResNet-18 and ResNet-110, the two methods are typically within a fraction of a percentage point of each other, with GLL often slightly better and only occasionally slightly worse. Note that this setting is simplistic -  with no pretraining or hyperparameter optimization - and only meant to illustrate a computational efficiency comparison between two methods with similar test performance.

% While GLL will naturally be slower due to the extra linear solve in the forward and backward pass, the overhead is modest. In particular, the cost of GLL is independent of the number of parameters in the architecture, so GLL-based models approach the runtime of standard architectures as the model size grows.
GLL is naturally slower due to the additional linear solve in the forward and backward passes: depending on the architecture and batch size, the per-batch training time increases by roughly a factor of two to an order of magnitude relative to MLP, with similar trends for testing time. Importantly, however, the extra cost of GLL is independent of the number of parameters in the backbone; its main overhead scales with the size of the base labeled set rather than network width or depth. Consequently, as the architecture becomes larger, from a small CNN to ResNet-18 and ResNet-110, GLL-based models progressively approach the runtime of their standard MLP-based counterparts, while retaining competitive or improved test error. 
}

\nc

%%%%%%%%%%%%%%%

\subsection{Adversarial Robustness}\label{sec:adversarial}

In this subsection, we demonstrate our method's superior robustness to adversarial attacks compared to the typical projection head and softmax activation layer in a neural network. Adversarial attacks seek to generate \textit{adversarial examples} (\cite{szegedy2013intriguing}): input data with small perturbations that cause the model to misclassify the data. Many highly performant deep neural networks can be fooled by very small changes to the input, even when the change is imperceptible to a human (\cite{madry2017towards, goodfellow2014explaining}). Hence, training models robust (i.e. resistant) to such attacks is important for both security and performance. Our experiments show that in both naturally and robustly-trained models, using the graph learning layer improves adversarial robustness compared to a softmax classifier on MNIST (\cite{mnist}), FashionMNIST (\cite{xiao2017fashion}), and CIFAR-10 (\cite{krizhevsky2009learning}) in both relatively shallow and deep networks. Moreover, the GLL is more robust to adversarial attacks \textit{without sacrificing performance on natural images}.

\subsubsection{Adversarial Attacks}

We consider three types of adversarial attacks: the fast gradient sign method (FGSM) (\cite{goodfellow2014explaining}) in the $\ell_\infty$ norm, the iterative fast gradient sign method (IFGSM) (\cite{kurakin2018adversarial}) in the $\ell_\infty$ norm, and the Carlini-Wagner attack (\cite{carlini2017towards}) in the $\ell_2$ norm. 

The FGSM and IFGSM attacks consider the setting where an attacker has an $\ell_\infty$ ``budget" to perturb an image $x$. More formally, FGSM creates an adversarial image $x'$ by maximizing the loss $\Loss(x',y)$ subject to a maximum allowed perturbation $\Vert x - x' \Vert_\infty \leq \epsilon$, where $\epsilon$ is a hyperparameter. Thus, a larger $\epsilon$ is more likely to successfully fool a model, but also more likely to be detectable. We can linearize the objective function as
\[ \Loss(x',y) = \Loss(x,y) + \nabla_x \Loss(x,y)(x-x')\]
which gives the optimal (in this framework) adversarial image 
\begin{equation}\label{eq:fgsm}
     x' = x + \epsilon \cdot \text{sign}(\nabla_x \Loss(x,y))
\end{equation}

It is also standard to clip the resulting adversarial example $x'$ to be within the range of the image space (e.g. $[0,1]$ for grayscale images) by simply reassigning $x'$ to $\min\{1, \max\{0, x'\}\}$.

Intuitively, FGSM computes the gradient of the loss with respect to each pixel of $x$ to determine which direction that pixel should be perturbed to maximize the loss. It should be noted that, as the name suggests, FGSM is designed to be a \textit{fast} attack (the cost is only one call to backpropagation), rather than an optimal attack. 

IFGSM simply iterates FGSM by introducing a second hyperparameter $\alpha$, which controls the step size of the iterates:
\[ x_{i+1} = x_i + \alpha \cdot \text{sign}(\nabla_{x_i} \Loss(x_{i},y)) \]
where $x_0 = x$ and $i = 0, ..., N-1$, and so $x_N$ is the generated adversarial example. However, we must also clip the iterates (pixel-wise) during this iteration to ensure that (1) they remain in the $\epsilon$-neighborhood of $x$ and (2) they remain in the range of the image space (e.g. in [0,1] for a grayscale image). Thus, the IFGSM attack is:
\begin{equation}
    x_{i+1} = \text{clip}_{x,\epsilon} \bigl\{ x_i + \alpha \cdot \text{sign}(\nabla_{x_i} \Loss(x_{i},y)) \bigr\}
\end{equation}
where
\[ \text{clip}_{x,\epsilon}\{z\} = \min \bigl\{1, x + \epsilon, \max\{0,x-\epsilon,z\}  \bigr\} \]

Note that this assumes the pixel values are in the range $[0,1]$, but can be easily tailored to other settings. IFGSM attacks have been shown to be more effective than FGSM attacks (\cite{kurakin2018adversarial}), however we will see in Section \ref{adv_results} that our graph learning layer is highly robust to this iterative method.

CW attacks frame the attack as an optimization problem where the goal is to find a minimal perturbation $\delta$ that changes the model's classification of an input $x$. Let $F$ be the neural network, so that $F(x)$ gives a discrete probability distribution over $k$ classes for an input $x$. Let $C(x) = \argmax_k F(x)$ be the classification decision for $x$. The CW attack is a perturbation $\delta$ that solves:
\begin{align*}
    \min_\delta &\phantom{=} \Vert \delta \Vert_2^2\\
    \text{s.t.} &\phantom{=} C(x+\delta) = t \\
    &\phantom{=} x + \delta \in [0,1]^n
\end{align*}
where $t$ is a classification such that $C(x) \neq t$. In other words, the CW attack seeks a minimal perturbation $\delta$ to apply to $x$ that changes its classification (while remaining a valid image). Other norms on $\delta$ are possible; in this work we focus on the $\ell_2$ norm. In practice, we choose $t$ to be the class with the second highest probability in $F(x)$, i.e. $t = \argmax_{k \neq C(x)} F(x)$

The constraint $C(x+\delta) = t$ is highly non-linear, so Carlini and Wagner instead consider objective functions $g$ that satisfy $C(x+\delta) = t$ if and only if $g(x+\delta) \leq 0$. While there are many choices for $g$ (see \cite{carlini2017towards}, Section IV A), we will use the following in our experiments:
\[g(x') = \bigl(\max_{i \neq t}F(x')_i - F(x')_t \bigr)^+ \]
where $a^+ = \max(0,a)$.  It is worth noting that a common choice for $g$ in the literature is to use the unnormalized logits from the neural network (that is, the output of a neural network just before the standard softmax layer) instead of $F(x)$. While both definitions satisfy the condition, we use the above because graph learning methods (such as Laplace learning) operate in the feature space and directly output probability distributions over the classes, so there is no notion of logits.

Using $g$, we can reformulate the problem as
\begin{align*}
    \min_\delta &\phantom{=} \Vert \delta \Vert_2^2 + c 
    \cdot g(x+\delta)\\
    \text{s.t.} &\phantom{=} x + \delta \in [0,1]^n
\end{align*}
where $c > 0$ is a suitably chosen hyperparameter. To make the problem unconstrained, we can introduce a variable $w$ and write $\delta  = \frac{1}{2}(\tanh{w}+1) - x$; since $-1 \leq \tanh{w} \leq 1$, we have that $0 \leq x + \delta \leq 1$. This gives the unconstrained objective:
\begin{equation}
    \min_w \phantom{=} \Vert \tfrac{1}{2}(\tanh{w}+1) - x \Vert_2^2 + c \cdot g(\tfrac{1}{2}(\tanh{w}+1))
\end{equation}
which can be solved using standard optimization routines. We use the Adam optimizer (\cite{kingma2014adam}).

\subsubsection{Adversarial Training}

To train adversarially robust networks, we use projected gradient descent (PGD) adversarial training (\cite{madry2017towards}) on our models. Motivated by guaranteeing robustness to first-order adversaries, PGD training applies IFGSM to images with a random initial perturbation during training. Hence, the model is trained on adversarial examples. More precisely, PGD training consists of two steps each batch: first, $x$ is randomly perturbed
\[ x^* = x + U(-\epsilon,\epsilon),\]
where $U$ is a uniform random perturbation. Then, IFGSM is applied 
\[ x_{i+1} = \text{clip}_{x,\epsilon} \bigl\{ x_i + \alpha \cdot \text{sign}(\nabla_{x_i} \Loss(x_{i},y)) \bigr\}, \]
where $x_{0} = x^*$ and the number of iterations and $\alpha$ are hyperparameters. This training procedure, while more expensive, has been shown to be effective at improving robustness in deep learning models compared to training on clean (i.e. natural) images (\cite{madry2017towards}, \cite{wang2021graph}).

\subsubsection{Adversarial Experiments}\label{adv_results}

We evaluate the adversarial robustness of neural networks with the GLL classifier on the MNIST (\cite{mnist}), FashionMNIST (\cite{xiao2017fashion}), and CIFAR-10 (\cite{krizhevsky2009learning}) benchmark datasets, and compare our method to the same networks trained with either a softmax classifier \tcb{or the WNLL classifier (\cite{wang2021graph})}. For all datasets, we normalize the images to have mean zero and standard deviation one, and 
% \jason{for consistency, maybe omit hardware here so we don't need to include it elsewhere?}
%; while our natural accuracy on these benchmark datasets is strong, the purpose of these experiments is
we train and run all adversarial attacks on a 11GB GeForce RTX 2080Ti GPU. We note that the goal of these experiments is not to get state-of-the-art natural accuracies, but rather to demonstrate our method's superior robustness to adversarial attacks compared to the typical MLP classifier without sacrificing performance on natural images. A summary of our experiments, highlighting a handful of our adversarial robustness results, can be found in Table \ref{table:adv_results}, while complete results are shown in Figures \ref{fig:mnist_adv}, \ref{fig:fashionmnist_adv}, and \ref{fig:cifar_adv}. Throughout, ``MLP" refers to a model with a final linear layer mapping from the feature space to the logits followed by a softmax classifier, \tcb{``WNLL" refers to a model with the WNLL classifier}, and ``GLL" refers to a model with our graph learning layer classifier. \tcb{WNLL and GLL are both drop-in replacements for the MLP classifier, and applied on the feature space itself. For these experiments, we do not include any unlabeled samples $\mathcal{B}_u$ in batch training (because the setting is fully supervised), and we do not pretrain to demonstrate that GLL-based models can be trained from scratch with desireable natural and robust accuracies.} 

%%%%%%%%%% old table
% \multirow{4}{3.5em}{MNIST} & GL Robust & 99.45 & \textbf{97.72} & \textbf{97.70} & \textbf{98.84} \\ 
% & \tcb{WNLL Robust} & 99.42 & 97.52 & 97.03 & 97.28 \\
% & MLP Robust & 99.20 & 95.17 & 93.57 & 97.41 \\ 
% & GL Natural & \textbf{99.53} & 96.51 & 96.44 & 98.71 \\ 
% & \tcb{WNLL Natural} & 99.50 & 96.46 & 95.36 & 96.81 \\
% & MLP Natural & 99.25 & 94.26 & 91.52 & 97.24 \\

% % eps = 0.3, c = 20

% \hline

% \multirow{4}{3.5em}{Fashion-MNIST} & GL Robust & 91.04 & 88.07 & 87.17 & \textbf{83.81}\\ 
% & \tcb{WNLL Robust} & 91.89 & \textbf{89.03} & \textbf{88.69} & 81.45 \\
% & MLP Robust & 92.08 & 86.50 & 85.67 & 81.00 \\ 
% & GL Natural & 92.14 & 65.44 & 61.81 & 57.19 \\ 
% & \tcb{WNLL Natural} & \textbf{93.56} & 78.54 & 60.67 & 25.41\\
% & MLP Natural & 91.65 & 40.81 &  18.30  & 62.01 \\

% % eps = 0.05, c = 20

% \hline

% \multirow{4}{3.5em}{CIFAR-10} & GL Robust & 90.60 & \textbf{80.50} & \textbf{79.62} & \textbf{78.09}\\
% & \tcb{WNLL Robust} & 90.66 & 74.78 & 71.64 & 49.65 \\
% & MLP Robust & 91.29 & 73.95 & 71.04 & 75.94 \\ 
% & GL Natural & 94.29 & 59.05  & 54.41 & 77.23 \\ 
% & \tcb{WNLL Natural} & 94.87 & 41.80 & 7.67 & 6.13 \\
% & MLP Natural & \textbf{94.98}  & 50.75 & 8.21 & 53.00 \\

%%%%%%%%%%%%%%%%

\begin{table}[!h]
\centering
\begin{tabular}{ llcccc } 
\toprule
Dataset & Model & Natural & FGSM & IFGSM & CW \\
\hline

\multirow{4}{3.5em}{MNIST} & GL Robust & 0.55 & \textbf{2.28} & \textbf{2.30} & \textbf{1.16} \\ 
& \tcb{WNLL Robust} & 0.58 & 2.48 & 2.97 & 2.72 \\
& MLP Robust & 0.80 & 4.83 & 6.43 & 2.59 \\ 
& GL Natural & \textbf{0.47} & 3.49 & 3.56 & 1.29 \\ 
& \tcb{WNLL Natural} & 0.50 & 3.54 & 4.64 & 3.19 \\
& MLP Natural & 0.75 & 5.74 & 8.48 & 2.76 \\
% eps = 0.3, c = 20
\hline
\multirow{4}{3.5em}{Fashion-MNIST} & GL Robust & 8.96 & 11.93 & 12.83 & \textbf{16.19}\\ 
& \tcb{WNLL Robust} & 8.11 & \textbf{10.97} & \textbf{11.31} & 18.55 \\
& MLP Robust & 7.92 & 13.50 & 14.33 & 19.00 \\ 
& GL Natural & 7.86 & 34.56 & 38.19 & 42.81 \\ 
& \tcb{WNLL Natural} & \textbf{6.44} & 21.46 & 39.33 & 74.59\\
& MLP Natural & 8.35 & 59.19 &  81.70  & 37.99 \\
% eps = 0.05, c = 20
\hline
\multirow{4}{3.5em}{CIFAR-10} & GL Robust & 9.40 & \textbf{19.50} & \textbf{20.38} & \textbf{21.91}\\
& \tcb{WNLL Robust} & 9.34 & 25.22 & 28.36 & 50.35 \\
& MLP Robust & 8.71 & 26.05 & 28.96 & 24.06 \\ 
& GL Natural & 5.71 & 40.95  & 45.59 & 22.77 \\ 
& \tcb{WNLL Natural} & 5.13 & 58.20 & 92.33 & 93.87 \\
& MLP Natural & \textbf{5.02}  & 49.25 & 91.79 & 47.00 \\

\bottomrule

\end{tabular}
\caption{Summary of adversarial results of models trained with our method compared to \tcb{WNLL and} the standard MLP classifier, on both robustly and naturally trained models. We report natural \tcb{error rate} (i.e. no attacks), and adversarial \tcb{error rate} on FGSM, IFGSM, and CW attacks. Best \tcb{error rates} are in bold. For MNIST, we report $\epsilon = 0.3$ for the FGSM and IFGSM attacks, and $\epsilon = 0.05$ for the other datasets. We report CW results for $c = 20$. We report more complete results across a range of values for $\epsilon$ and $c$ in Figures \ref{fig:mnist_adv}, \ref{fig:fashionmnist_adv}, and \ref{fig:cifar_adv}. We include the natural \tcb{error rates} to show that the GLL does not sacrifice performance in this regime, while also being significantly more robust to adversaries across datasets. \tcb{One mild exception is WNLL slightly outperforms GLL for the fast gradient attacks on FashionMNIST, however Figure \ref{fig:fashionmnist_adv} shows that in the case of IFGSM, GLL is much better against stronger attacks.} The slight increase in natural test error on robustly trained models compared to naturally trained models has been observed previously in \cite{madry2017towards}; however, as expected, the robust models perform significantly better on adversarially-perturbed data than their naturally-trained counterparts.}
\label{table:adv_results}

\end{table}

For MNIST, we use the same ``small CNN" architecture as in \cite{wang2021graph}: four convolutional layers with ReLU activations after each layer and maxpooling after the second and fourth layer, followed by three linear layers and a softmax classifier. We train for 100 epochs using the Adam optimizer (\cite{kingma2014adam}) with an initial learning rate of $0.01$ which decays by a factor of $0.1$ every 25 epochs. For the robustly trained models, we run PGD training with $\epsilon = 0.3$ and $\alpha = 0.01$ for 5 iterations. In training, we use a batch size of 1000, and for GLL \tcb{and WNLL} we use 100 labeled ``base" points per batch (i.e. 10 per class), which are randomly sampled each epoch. 

To evaluate experimental robustness, we attack the test set with the FGSM, IFGSM, and CW attacks. When using FGSM and IFGSM on GLL, we embed the entire test set of 10000 digits along with 10000 randomly selected data points from the training set (i.e. 1000 per class) when running the attacks. We run CW in batches of 1000 test and 1000 training points due to computational constraints; in general graph learning methods perform better with more nodes on the graph as it better approximates the underlying data manifold. \tcb{We use the same testing batch sizes (full batch on FGSM and IFGSM, 1000 for CW) for the MLP and WNLL models, and the same number of base points for WNLL}. We report results for FGSM and IFGSM at varying levels of $\epsilon$, fixing $\alpha = 0.05$ for IFGSM and letting the number of iterations be $5*(\epsilon/\alpha)$, following a similar heuristic as in \cite{kurakin2018adversarial} to allow the adversarial example sufficient ability to reach the boundary of the $\epsilon$-ball. For CW attacks, we let $c$ vary and run 100 iterations of the Adam optimizer with a learning rate of 0.005. Results for each attack, comparing both robustly and naturally trained models with MLP and GLL classifiers, are shown in Figure \ref{fig:mnist_adv}.

\begin{figure}[ht]
\centering
\begin{subfigure}{.5\textwidth}
    \centering
    \includegraphics[width=\textwidth]{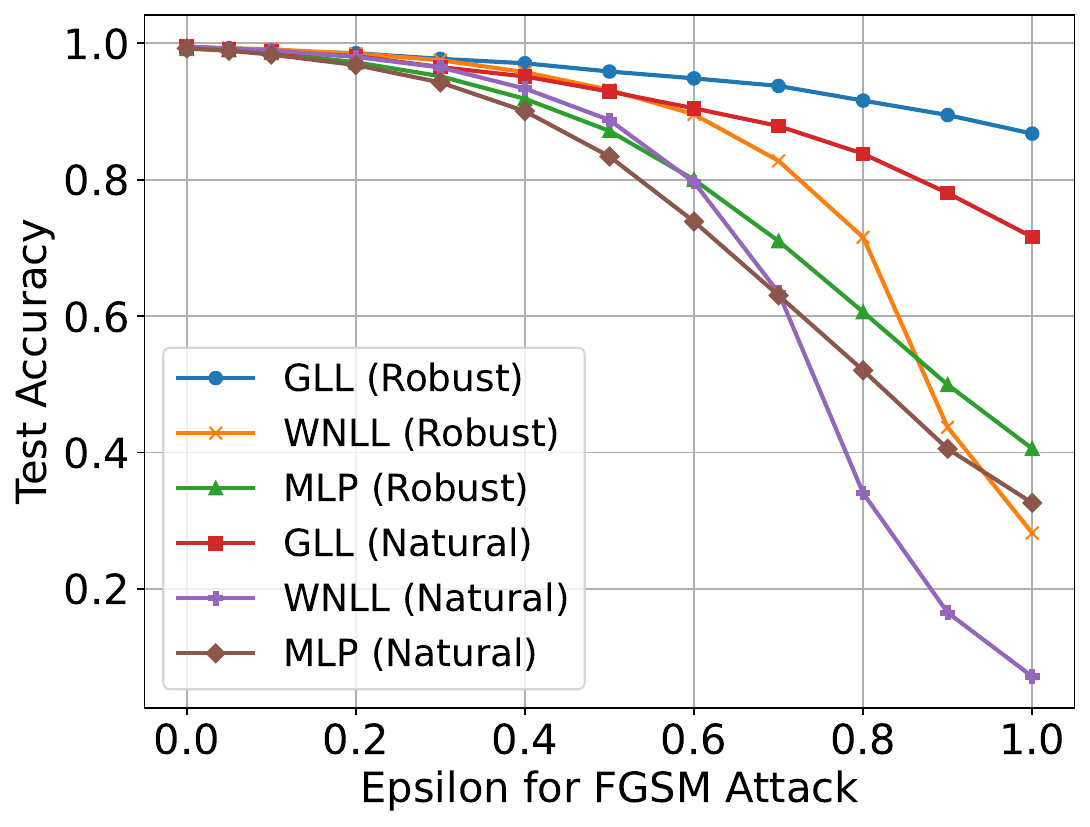}
    \caption{FGSM}
\end{subfigure}%
\begin{subfigure}{.5\textwidth}
    \centering
    \includegraphics[width=\textwidth]{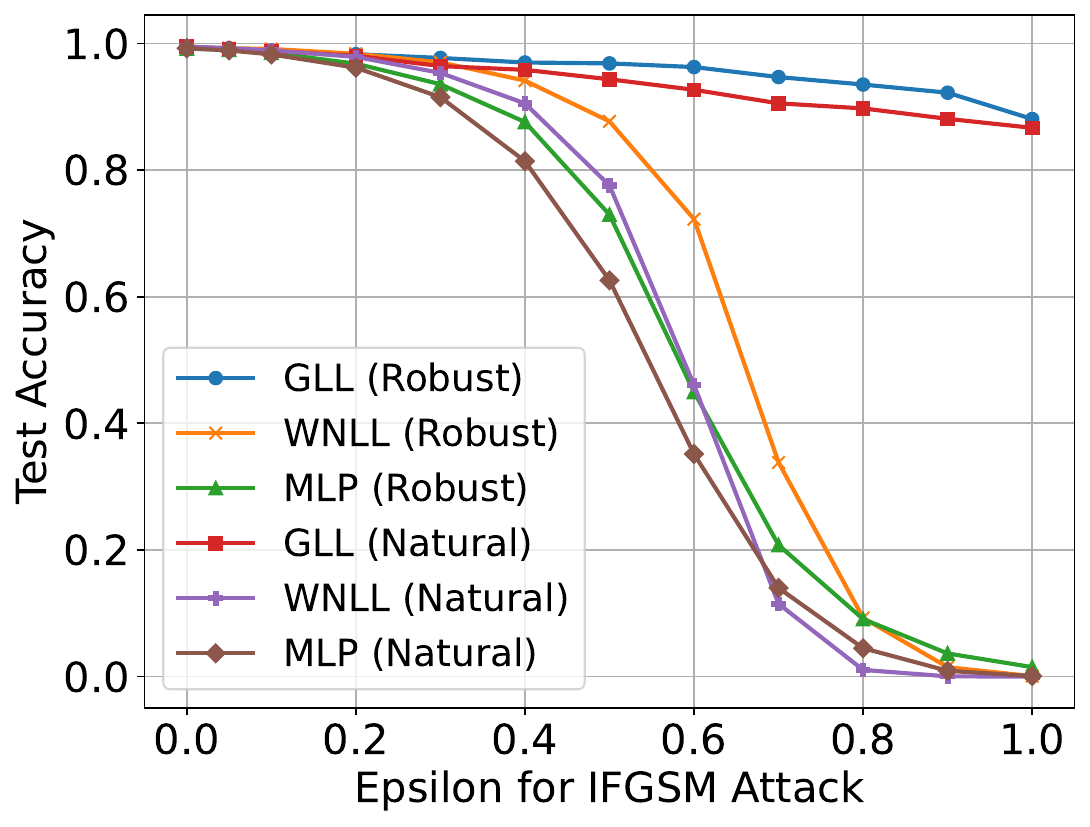}
    \caption{IFGSM}
\end{subfigure}
\begin{subfigure}{.5\textwidth}
    \centering
    \includegraphics[width=\textwidth]{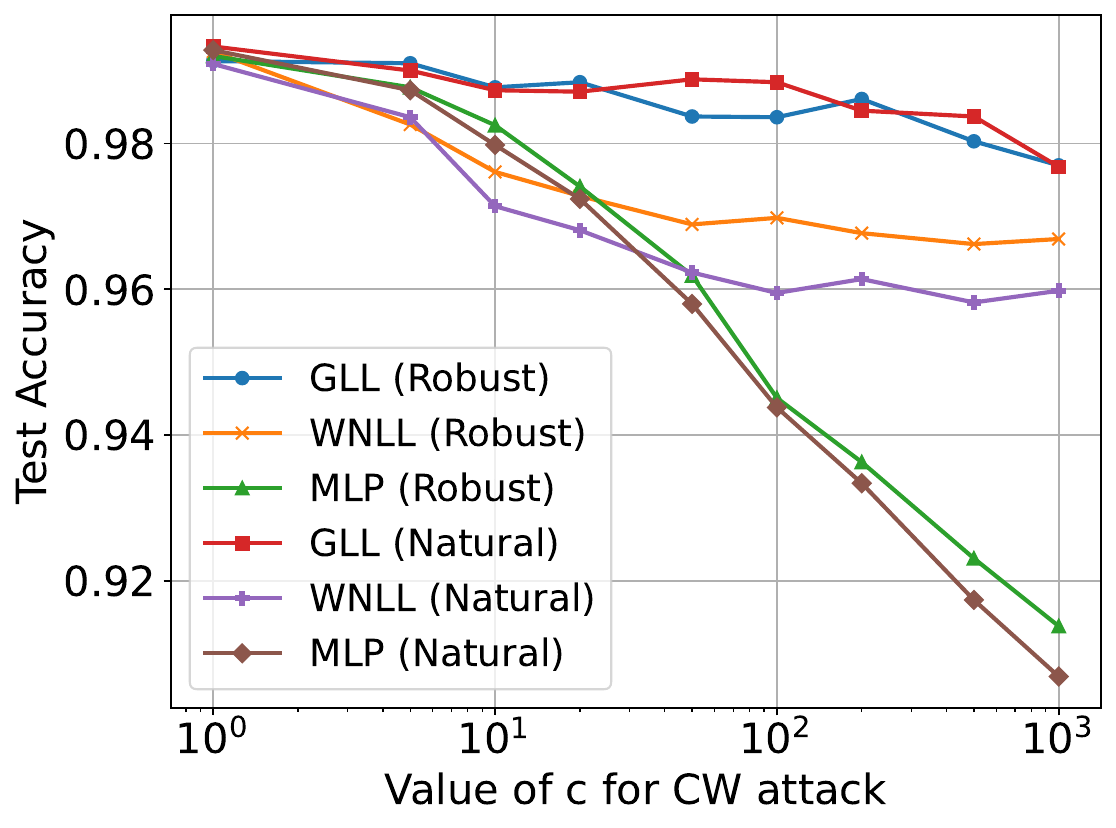}
    \caption{CW}
\end{subfigure}%
\caption{Adversarial robustness to FGSM, IFGSM, and CW attacks on the MNIST dataset. The backbone network is the same in all cases; the only difference is the classifier (either MLP with softmax, \tcb{WNLL}, or our novel GLL). We report results for both robustly and naturally trained models.}
\label{fig:mnist_adv}
\end{figure}

For FashionMNIST, we use ResNet-18 (\cite{he2016deep}), where again we replace the final linear layer and softmax with GLL. We again use the Adam optimizer for 100 epochs, with an intial learning rate of $0.01$ which decays by a factor of $0.5$ every 10 epochs. For the robustly trained models, we run PGD training with $\epsilon = 0.05$ and $\alpha = 0.01$ for 5 iterations. In training, we use a batch size of 2000, and for GLL \tcb{and WNLL} we use 200 labeled ``base" points per batch. We attack in batches of 250 test points, with 500 training points for GLL \tcb{and WNLL}; this was chosen as it was the largest batch size we could use given our computational resources. The hyperparameters for the attacks are identical to those for MNIST. We report results in Figure \ref{fig:fashionmnist_adv}.

\begin{figure}[ht]
\centering
\begin{subfigure}{.5\textwidth}
    \centering
    \includegraphics[width=\textwidth]{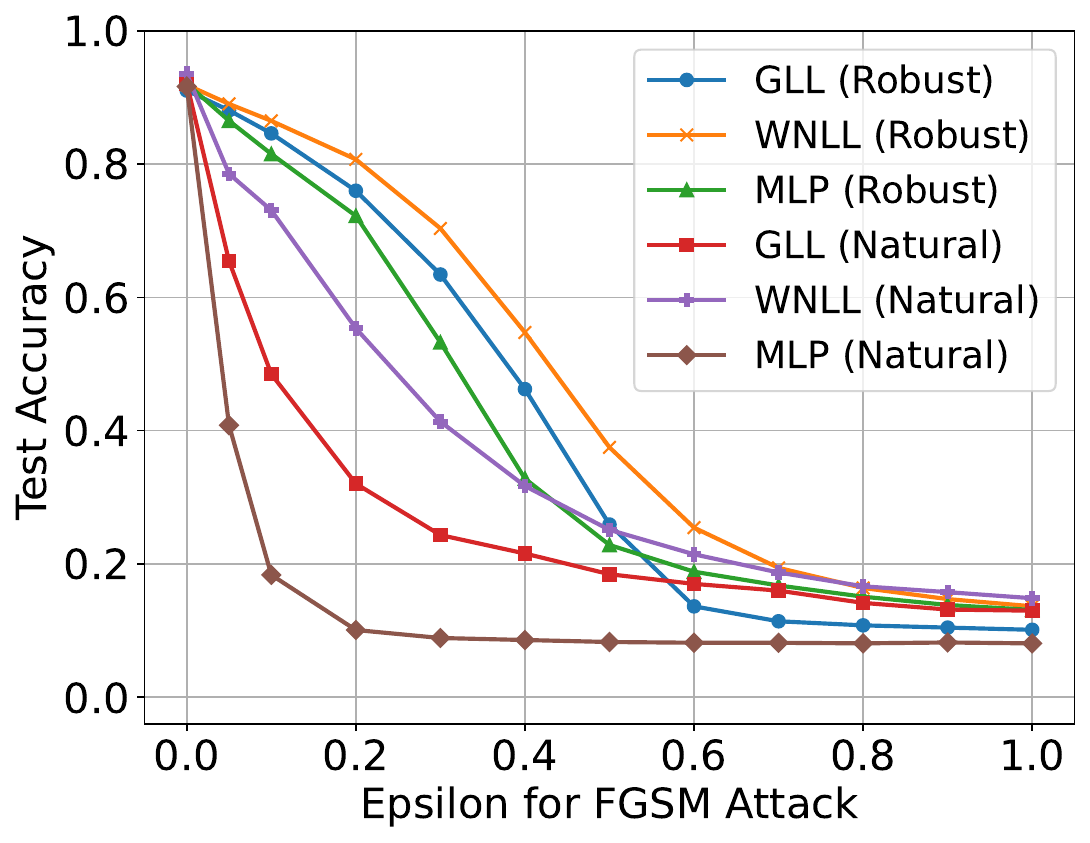}
    \caption{FGSM}
\end{subfigure}%
\begin{subfigure}{.5\textwidth}
    \centering
    \includegraphics[width=\textwidth]{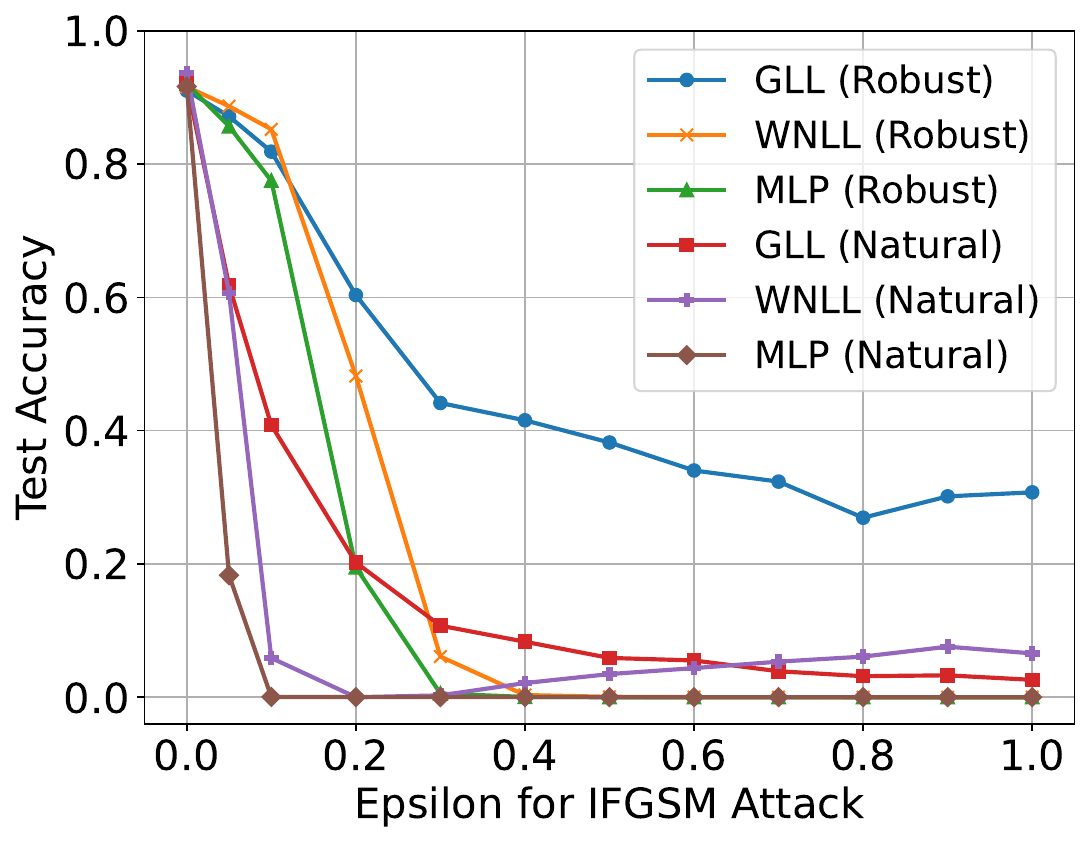}
    \caption{IFGSM}
\end{subfigure}
\begin{subfigure}{.5\textwidth}
    \centering
    \includegraphics[width=\textwidth]{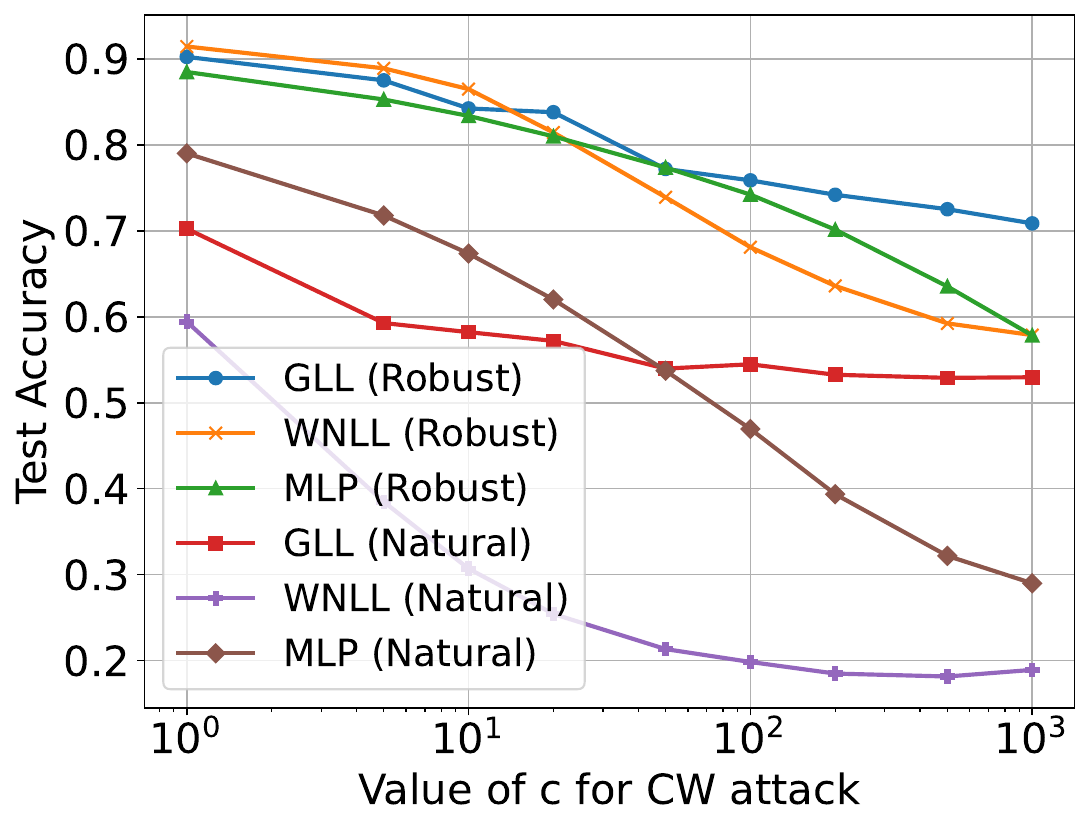}
    \caption{CW}
\end{subfigure}%
\caption{Adversarial robustness to FGSM, IFGSM, and CW attacks on the FashionMNIST dataset. The backbone network is the same in all cases (ResNet-18); the only difference is the classifier (either MLP with softmax,  \tcb{WNLL}, or our novel GLL). We report results for both robustly and naturally trained models.}
\label{fig:fashionmnist_adv}
\end{figure}

Finally, for CIFAR-10, we use PreActResNet-18 (\cite{he2016identity}) as our architecture backbone. We use stochastic gradient descent with momentum set to $0.9$, a weight decay of $0.0005$, and an intial learning rate of $0.1$. We train for 150 epochs and use cosine annealing to adjust the learning rate (\cite{loshchilov2016sgdr}). For the robustly trained models, we run PGD training with $\epsilon = 0.05$ and $\alpha = 0.01$ for 5 iterations. In training, we use a batch size of 200, and for GLL \tcb{and WNLL} we use 100 labeled ``base" points per batch. We attack in batches of 200 test points, with 500 training points for GLL \tcb{and WNLL}; this was chosen as it was the largest batch size we could use given our computational resources. The hyperparameters for the attacks are identical to those for MNIST, except we use 50 iterations of optimization for the CW attacks. We report results in Figure \ref{fig:cifar_adv}.

\begin{figure}[ht]
\centering
\begin{subfigure}{.5\textwidth}
    \centering
    \includegraphics[width=\textwidth]{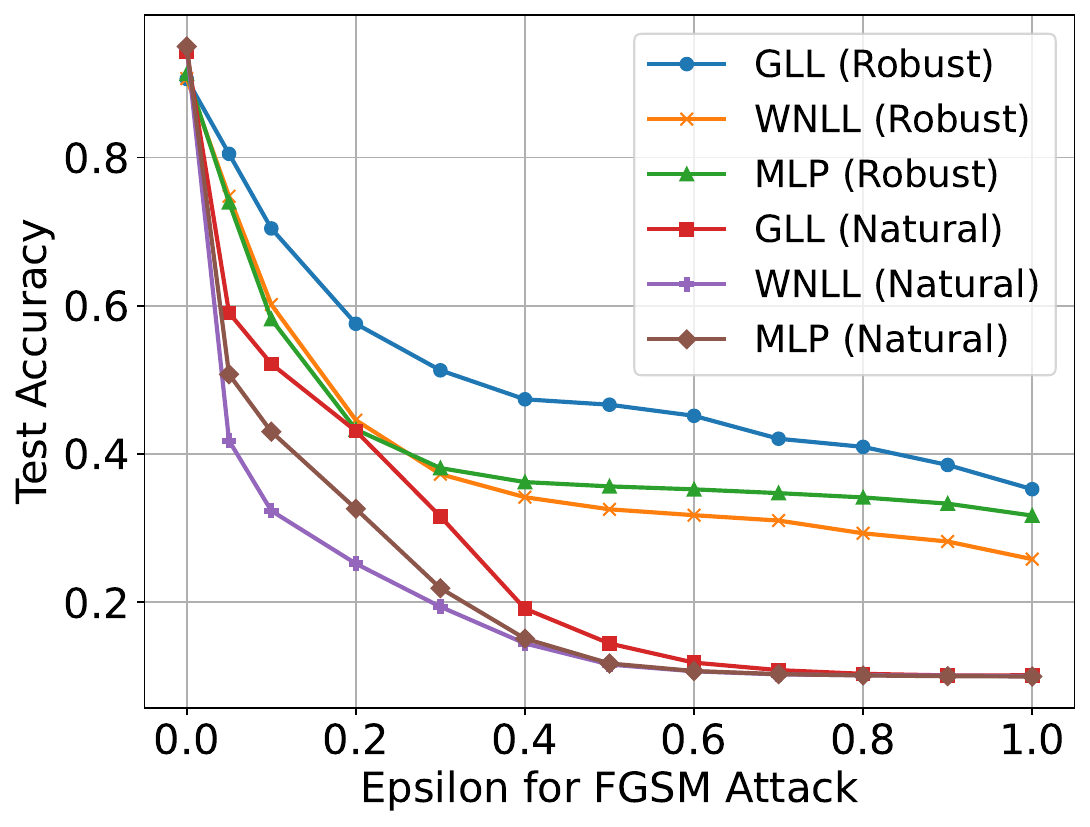}
    \caption{FGSM}
\end{subfigure}%
\begin{subfigure}{.5\textwidth}
    \centering
    \includegraphics[width=\textwidth]{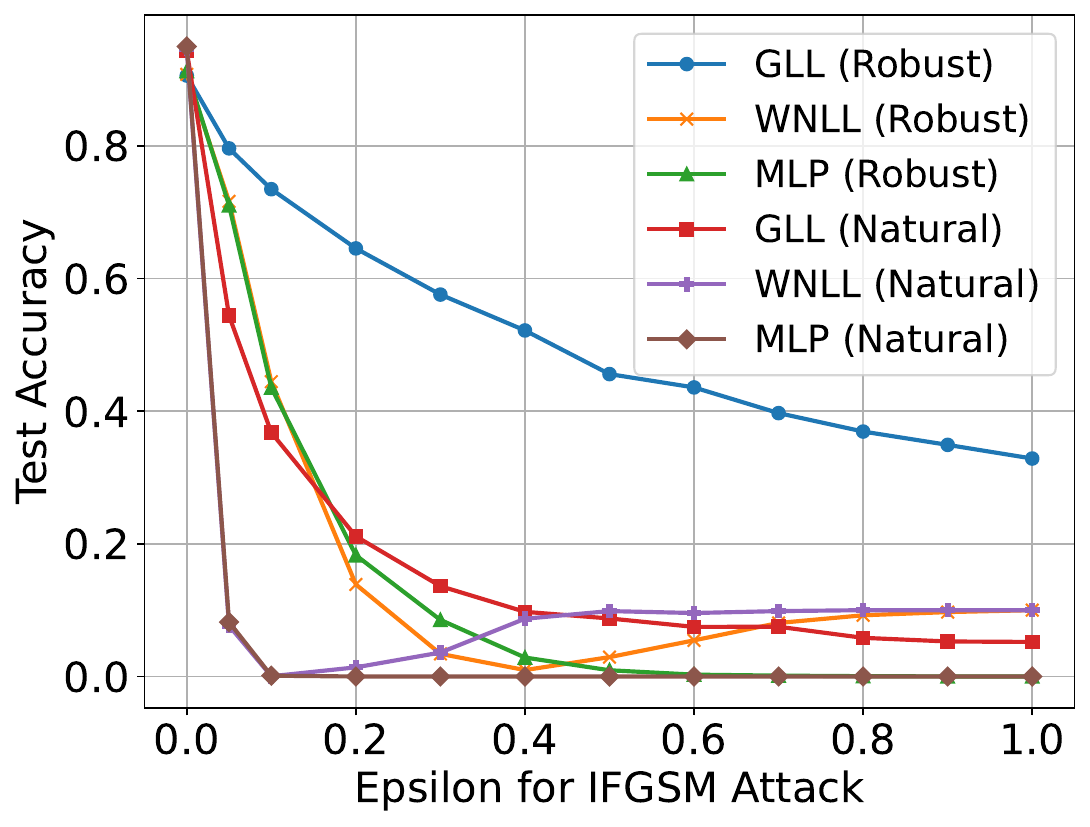}
    \caption{IFGSM}
\end{subfigure}
\begin{subfigure}{.5\textwidth}
    \centering
    \includegraphics[width=\textwidth]{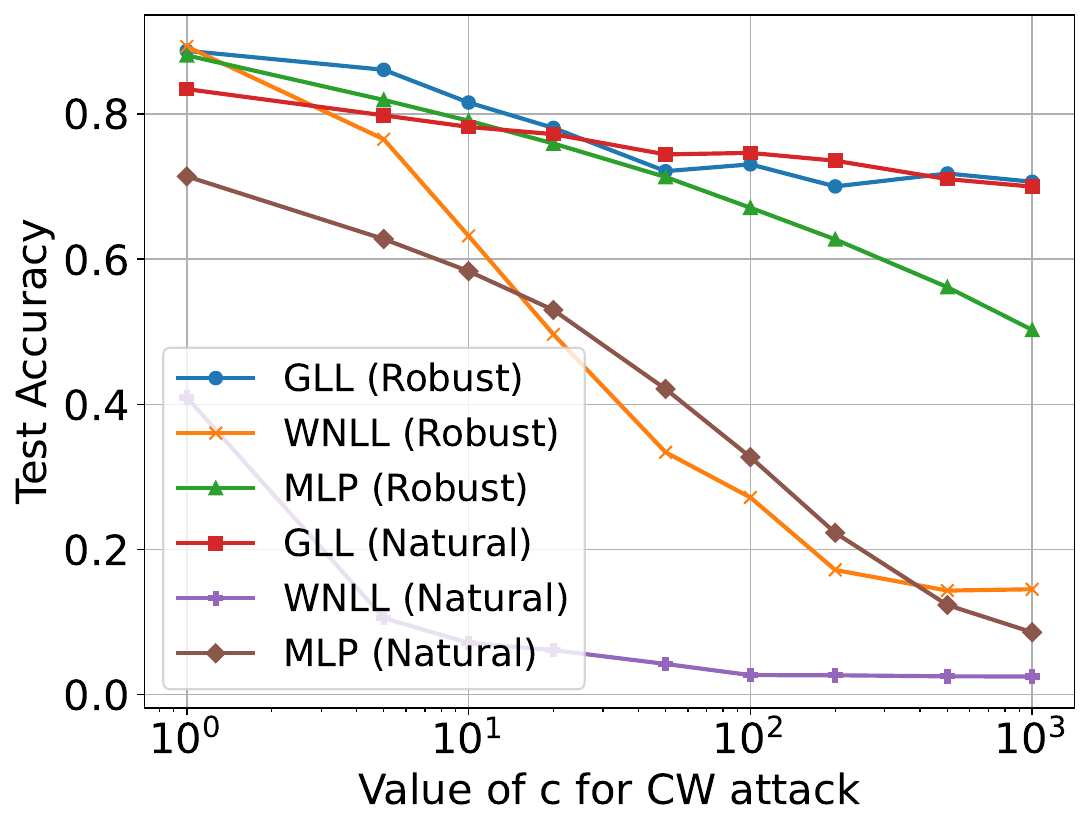}
    \caption{CW}
\end{subfigure}%
\caption{Adversarial robustness to FGSM, IFGSM, and CW attacks on the CIFAR-10 dataset. The backbone network is the same in all cases (PreActResNet-18); the only difference is the classifier (either MLP with softmax, \tcb{WNLL}, or our novel GLL). We report results for both robustly and naturally trained models.}
\label{fig:cifar_adv}
\end{figure}

Taken together, our results (Table \ref{table:adv_results} and Figures \ref{fig:mnist_adv}, \ref{fig:fashionmnist_adv}, and \ref{fig:cifar_adv}) demonstrate that across datasets, training methods (both natural and robust), severity of attacks, and neural network architectures, models trained with GLL are significantly more robust to adversarial attacks, without sacrificing natural accuracy. Laplace learning can be viewed as repeatedly computing a weighted average over a given node's neighbors, and this relational information may be preventing attacks from successfully perturbing the input images to fool the network \tcb{compared to the MLP classifier models}. \tcb{Moreover, our results suggest that tracking the gradients exactly (GLL), rather than approximating them (WNLL), offers a significant increase in robustness. 
% In particular, we note that the models employing the standard classifier \tcb{or WNLL} quickly decay to 0 accuracy against IFGSM and CW attacks, while GLL shows robustness to these attacks, even for extremely large values of $\epsilon$ and $c$, respectively. 
GLL's advantage is even more apparent for stronger attacks, whereas MLP and WNLL-based models decay much more quickly toward random guessing or even 0\% accuracy.}
Moreover, performance of our GLL could likely be even further improved with more computational resources, as graph learning techniques tend to improve as the number of nodes (i.e. the batch size) increases.

\begin{figure}[!t]
\centering
\begin{subfigure}{.5\textwidth}
    \centering
    \includegraphics[width=\textwidth]{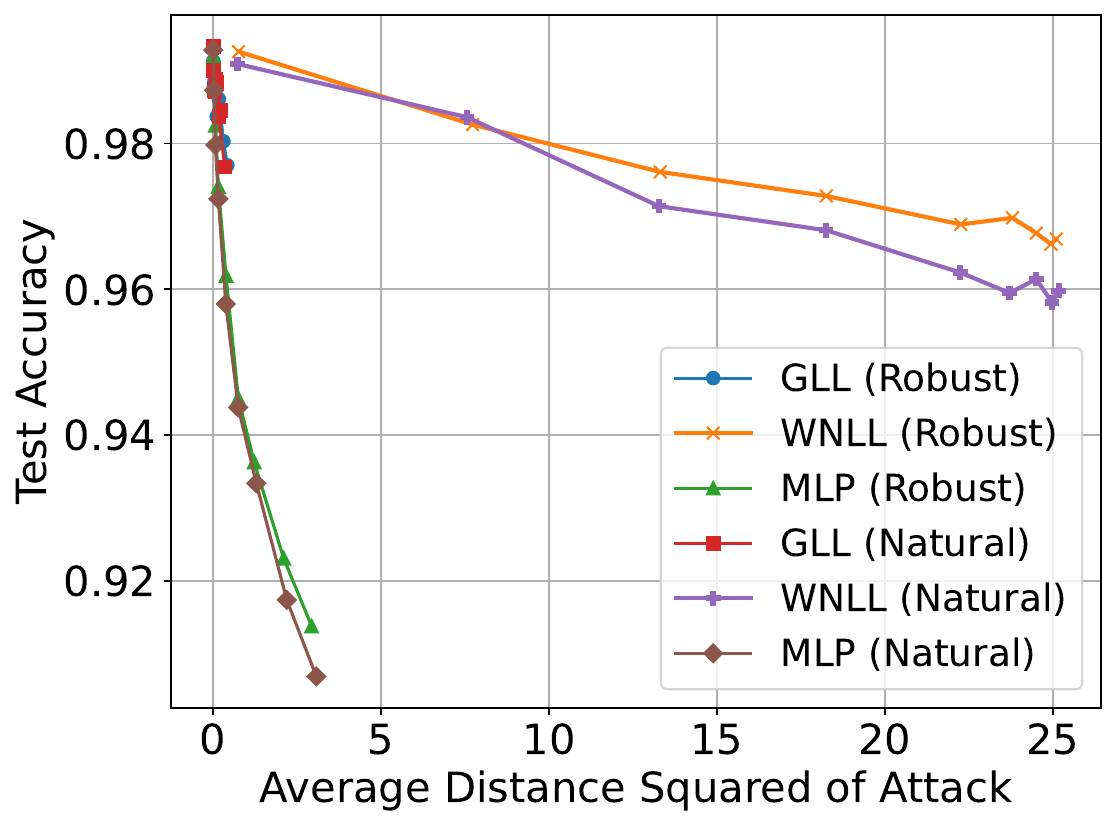}
    \caption{MNIST}
\end{subfigure}%
\begin{subfigure}{.5\textwidth}
    \centering
    \includegraphics[width=\textwidth]{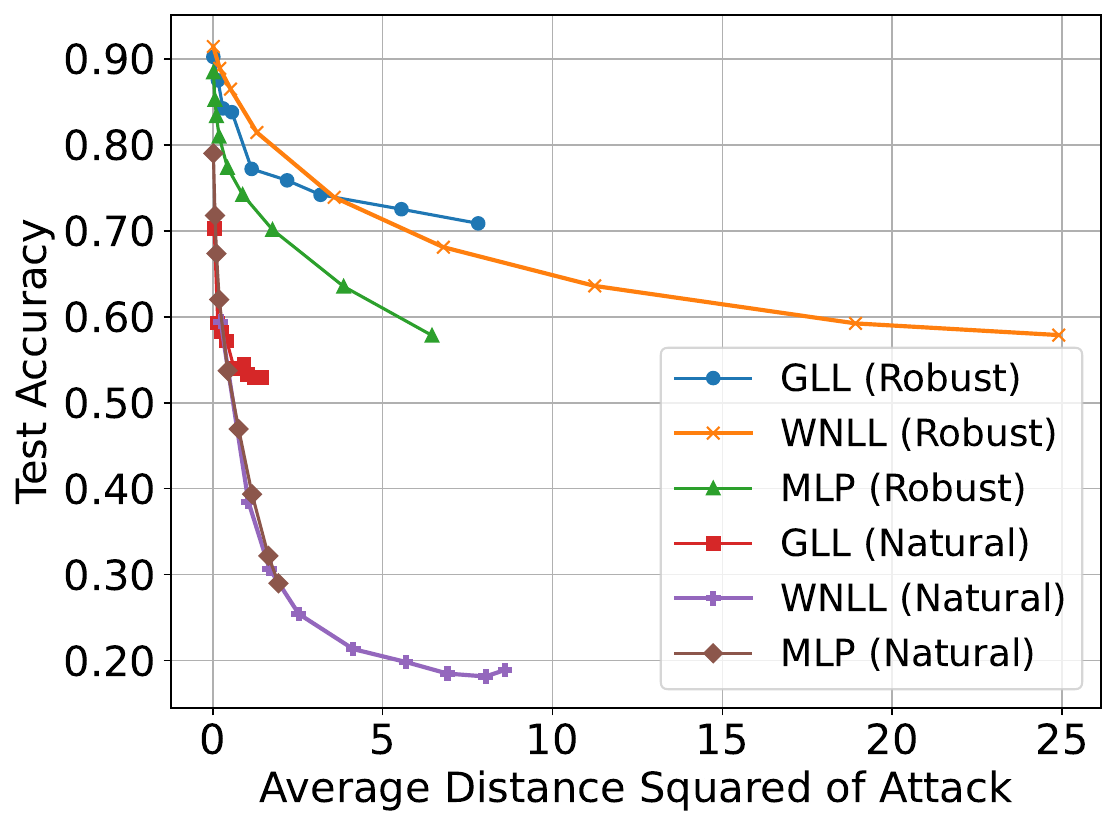}
    \caption{FashionMNIST}
\end{subfigure}
\begin{subfigure}{.5\textwidth}
    \centering
    \includegraphics[width=\textwidth]{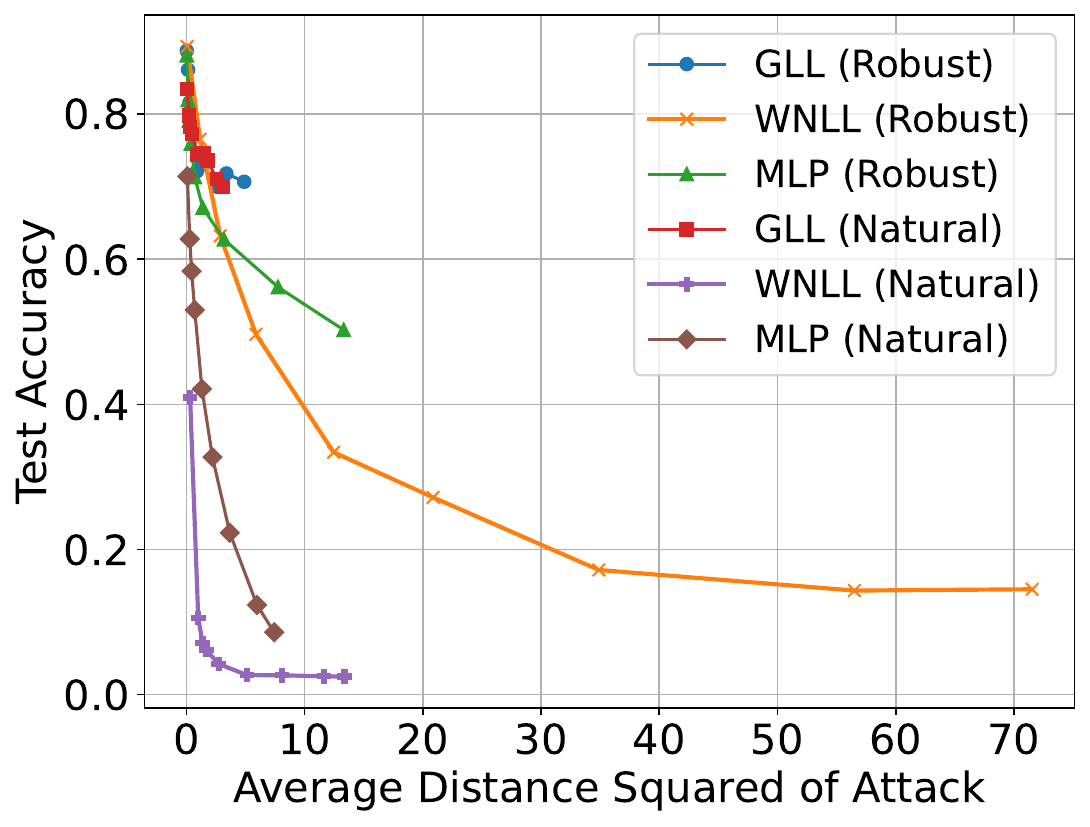}
    \caption{CIFAR-10}
\end{subfigure}%
\caption{Comparison of accuracy vs. average distance of attack for CW attack on all three datasets. Each point represents a different value of $c$. We see that GLL resists the adversary from moving the target images significantly more than the standard MLP or \tcb{WNLL classifier, suggesting that graph-based classification \textit{and} precise tracking of gradients is key for adversarial robustness.}}
\label{fig:cw_dist}
\end{figure}

We further investigate the improved robustness of GLL against CW attacks in Figure \ref{fig:cw_dist} (see also Appendix \ref{app:more_cw}). This figure plots the accuracy versus the average distance (squared) that the adversary moves an image in the test set (instead of the value of $c$). This reveals a striking difference: models trained with the GLL classifier \textit{resist the CW adversary from moving the images} significantly more than the MLP \tcb{or WNLL}-based models, resulting in more robustness to these attacks. \tcb{Interestingly, WNLL is the least resistant to these perturbations. 
%This suggests that, under small perturbations, the relational information used in graph-based classifiers is alone sufficient for robustness to attacks. However, as $c$ increases against WNLL, the attacker perturbs the images more aggressively to create misclassifications. GLL appears to prevent these big perturbations, and hence have much better accuracy against the CW attacks. 
Since WNLL uses linear layer gradients and GLL uses precise gradients, these results suggest that the exact backpropagation in GLL aids robustness to these attacks. One hypothesis is that the relational information in GLL's gradients ``trap" data points in neighborhoods with points of the same class. Theoretically guaranteeing GLL's adversarial robustness is an interesting line of future work.}

% We evaluate by attacking the test set using FGSM with varying values of $\epsilon$, see Figure \ref{fig:fgsm}.

% \begin{figure}
%     \centering
%     \includegraphics[width=\textwidth]{figures/pgd 5 iter.png}
%     \caption{Adversarial Robustness of CNN and GL (ours) to FGSM attacks at different values of $\epsilon$.}
%     \label{fig:fgsm}
% \end{figure}

% We see that graph learning is significantly more robust to adversarial attacks for high values of $\epsilon$. At $\epsilon = 0.3$ - the value of $\epsilon$ employed for PGD training - GL outperforms the standard CNN with 96.73\% robust accuracy compared to 95.62\%. Moreover, GL maintains 73.33\% accuracy at $\epsilon = 1$, while CNN decays to below 50\%.

% conclusion and acknowledgements
\section{Conclusion}

Graph-based learning is a powerful class of machine learning techniques that have been developed throughout the 21st century. As research in deep learning has exploded throughout the same time period, there has been significant interest in integrating deep learning with similarity matrix-based graph learning.  This paper gives a precise mathematical framework to do exactly that. We derived the necessary gradient backpropagation equations for general nonlinear elliptic graph Laplace equations, allowing us to fully integrate graph-based classifiers into the training of neural networks. Using this novel framework, we introduced the Graph Learning Layer, or GLL, which replaces the standard projection head and softmax classifier in classification tasks. Experimentally, we demonstrated how the GLL improves upon the standard classifier and previous graph-based layers (WNLL) in several ways. \tcb{Two illustrative examples demonstrated that GLL-based networks can learn the geometry of the underlying data and improve training dynamics in overparameterized settings, respectively. In a suite of large scale experiments, we showed that networks trained with a GLL classifier have improved generalization on benchmark datasets, especially when training labels are limited. Although the improvement in natural accuracy was occasionally modest, later experiments demonstrated the GLL's significant and consistent performance gains for networks in the adversarial setting. This points toward a practical advantage of GLL in safety- and security-critical settings.
% est accuracy on benchmark datasets in both small and large networks, and it is significantly more robust to a variety adversarial attacks in both naturally and robustly trained models.
Moreover, by comparing to WNLL-based networks, our results also empirically validated the advantage of tracking gradients through graph learning \textit{precisely}, rather than approximately}.

There are many exciting avenues of future work in this area. Given the strong empirical results, theoretically guaranteeing the adversarial robustness of the GLL is an important extension of this work. Graphs are a common setting for active learning techniques (\cite{settles2009active, miller2024model}); the full integration of graph learning into a neural network provides a natural way to do active learning in a deep learning context. Our theoretical work also provides a way to integrate other graph learning algorithms - such as $p$-Laplace learning and Poisson learning - into a deep learning framework, and numerical experiments could yield similarly strong results as reported in this paper. Many performant \tcb{semi-supervised} deep learning techniques utilize psuedolabels generated by graph learning methods (\cite{iscen2019label, sellars2021laplacenet}); we hope the methods presented here will lead to further improvements and insights.

% Acknowledgements and Disclosure of Funding should go at the end, before appendices and references

\acks{
HHM and JB were partially supported by NSF research training grant DGE-1829071.
HHM was supported by the National Science Foundation Graduate Research Fellowship Program under Grant No. DGE-2034835. 
JB and ALB were also supported by Simons Math + X award 510776.
ALB and BC were supported by NSF grants DMS-2027277 and DMS-2318817.
BC was supported by the UC-National Lab In- Residence Graduate Fellowship Grant L21GF3606.
JC was supported by NSF-CCF:2212318, the Alfred P. Sloan Foundation, and an Albert and Dorothy Marden Professorship. 
Any opinions, findings, and conclusions or recommendations expressed in this material are those of the author and do not necessarily reflect the views of the National Science Foundation. The authors declare no competing interests.}

% appendix
% Manual newpage inserted to improve layout of sample file - not
% needed in general before appendices/bibliography.

\newpage

\appendix

\section{Training Details}\label{app:training_details}

This section provides a detailed discussion of various aspects of the training process. These details are implemented in our experiments (Section~\ref{sec:experiments}), especially in the fully-supervised training for comparison between the MLP and GLL classifiers (Section~\ref{sec:comparison_mlp_gll}).

{
\subsection{Sampling from Base and Unlabeled Datasets}\label{sec:base_dataset_sampling}
In the GLL-based training process, our approach to constructing mini-batches differs from classical fully supervised training. As described in Section~\ref{sec:sampling_and_training_on_batches}, each mini-batch \(\mathcal{B}\) contains a base dataset \(\mathcal{L}_b \subset \mathcal{L}\) together with additional labeled and unlabeled samples \(\mathcal{B}_l\) and \(\mathcal{B}_u\). For each mini-batch, in addition to the images whose labels are used to compute the loss (the samples in \(\mathcal{B}_l\)), we also require base images and their associated base labels (the samples in \(\mathcal{L}_b\)), which act as labeled nodes for graph-based label propagation. While this base dataset can be kept fixed throughout training, we also consider two update strategies for \(\mathcal{L}_b\) that can further enhance performance.

First, in the \emph{random resampling} strategy, we update the base dataset at the beginning of each epoch by sampling a random subset \(\mathcal{L}_b \subset \mathcal{L}\) with \(|\mathcal{L}_b| = N_b\). To provide stable supervision for label propagation, \(\mathcal{L}_b\) is chosen so that the class proportions in \(\mathcal{L}_b\) match the class proportions in the full labeled set \(\mathcal{L}\), for example by stratified sampling.

Second, in an \emph{uncertainty-based} strategy, we explicitly score labeled samples by how informative they would be if selected into \(\mathcal{L}_b\). In the graph Laplace learning framework, the base dataset \(\mathcal{L}_b\) represents the labeled nodes and can be viewed as the sources in a label propagation process. An ideal base dataset should be relatively dispersed in the graph and lie near decision boundaries between clusters, so that it provides informative supervision to the unlabeled nodes.

Let \(\mathcal{X} = \mathcal{L} \cup \mathcal{U}\) denote the full training set, where \(\mathcal{L}\) is the labeled subset and \(\mathcal{U}\) is the unlabeled subset, and let \(N_b\) be the desired size of the base dataset. Assume that at a given epoch, the GLL prediction for a labeled sample \(\mathbf{x}_j \in \mathcal{L} \setminus \mathcal{L}_b\) is
\[
\bfu_j = \bigl(u_j^1, u_j^2, \ldots, u_j^C\bigr),
\]
where \(C\) is the number of classes and \(u_j^c\) is the predicted confidence for class \(c\). We can define two uncertainty-based scores for \(\bfu_j\):
\begin{align}
\text{Entropy score} &= -\sum_{c=1}^C u_j^c \log(u_j^c), \label{eq:entropy_score}\\
\text{L2 score} &= 1 - \|\bfu_j\|_2. \label{eq:l2_score}
\end{align}
Both scores are larger when the prediction \(\bfu_j\) is more uncertain. In this second strategy, we select the new base dataset \(\mathcal{L}_b\) as the \(N_b\) labeled samples in \(\mathcal{L}\) with the highest scores, according to either \eqref{eq:entropy_score} or \eqref{eq:l2_score}.

However, there is a risk that the highest-scoring samples may cluster in a small region of the graph, providing redundant information. A more effective variant is to employ the LocalMax batch active learning method (\cite{chapman2023novel,chen2024batch}). In this approach, nodes in the graph are first assigned uncertainty scores, and then a subset is selected that both has high scores and satisfies a local maximum condition in the graph. This discourages the selection of nodes that are too close to each other and leads to a more diverse and informative base dataset \(\mathcal{L}_b\) for GLL.

\subsection{Contrastive Learning and Warm Up}\label{sec:appendix_simclr}
Prior to GLL-based fully supervised training for the CIFAR-10 dataset, we initialize the network's feature encoder using pre-trained weights from SimCLR (\cite{chen2020simple, chen2020big}). SimCLR is a framework for self-supervised contrastive learning of visual representations. For a neural network $f$ and a minibatch of $m$ samples $\{\x_1,\x_2,\ldots,\x_m\}$, each $\x_k$ is augmented into pairs $\tilde{\x}_{2k-1}, \tilde{\x}_{2k}$. These are processed through the network: $\z_{2k-1} = f(\tilde{\x}_{2k-1}), \z_{2k} = f(\tilde{\x}_{2k})$. The SimCLR loss is defined as:
\begin{equation}\label{eq: SimCLR_loss}
\begin{split}
    &\mathcal{L}_\mathrm{sim} = \frac{1}{2m}\sum_{k=1}^{m}[\ell(2k,2k-1)+\ell(2k-1,2k)],\\
    &\ell(i, j) = -\log \frac{\exp(g(\z_i,\z_j)/\tau)}{\sum_{k=1, k\neq i}^{2m} \exp(g(\z_i,\z_k)/\tau)},
\end{split}
\end{equation}
where $\tau$ is a temperature parameter, and $g(\z_i,\z_j) = \z_i^\top \z_j / (\|\z_i\|\|\z_j\|)$ is the cosine similarity. Essentially, SimCLR teaches the network to increase the similarity between embedded pairs and reduce the similarity between the other embedded pairs. It is worth noting that the cosine similarity used in the SimCLR loss \eqref{eq: SimCLR_loss} closely resembles the similarity measure we employ when constructing graphs in graph Laplace learning. This similarity suggests that the feature vectors generated by a network trained with SimCLR already possess a favorable clustering structure. Consequently, utilizing SimCLR to pretrain the feature encoder in our network can significantly reduce the difficulty of subsequent training stages.

In the supervised setting, SupCon (\cite{khosla2020supervised}) can be used instead of SimCLR. With full access to the underlying labels, the goal is now to maximize the similarity between all embedded pairs within the same class and minimize the similarity of samples across classes. The loss function is 

\begin{equation}\label{eq: SupCon_loss}
\mathcal{L}_\mathrm{sup}=\sum_{i \in I} \frac{-1}{|P(i)|} \sum_{p \in P(i)} \log \frac{\exp \left(z_i \cdot z_p / \tau\right)}{\sum_{a \in A(i)} \exp \left(z_i \cdot z_a / \tau\right)}
\end{equation}
where \(P(i) \equiv\left\{p \in A(i): \tilde{\boldsymbol{y}}_p=\tilde{\boldsymbol{y}}_i\right\}\) is the set of indices for all samples that belong to the same class as sample \(i\) (excluding itself) and \(A(i)\) is the set containing all indices except \(i\). 

We employed augmentation techniques in both the SimCLR pre-training process and the subsequent fully supervised training. For data augmentation in both stages, we utilized the strong augmentation method defined in the SupContrast GitHub Repository\footnote{\url{https://github.com/HobbitLong/SupContrast/tree/master}}. The use of data augmentation is helpful for two main reasons. Firstly, in contrastive learning, augmentation creates diverse views of the same instance, which is fundamental to learning robust and invariant representations. This process helps the model to focus on essential features while disregarding irrelevant variations. Secondly, for general training tasks, augmentation enhances the model's ability to generalize by exposing it to a wider variety of data representations, thereby reducing overfitting and improving performance on unseen data.

For the purposes of warming up the GLL, either SimCLR or SupCon can be used depending on the framework. In our work, we have even seen success in balancing a combination of the two methods to utilize the labels for a strong embedding while not overfitting too strongly. The pretraining loss is taken to be the same mixed objective as in \eqref{eq:pretrain_loss},
\[
    \mathcal{L}_\mathrm{pretrain} = \gamma \mathcal{L}_\mathrm{sim} + (1-\gamma)\mathcal{L}_\mathrm{sup}
\]
for some \(\gamma \in [0,1]\).

%%% augmentation
\subsection{Data Augmentation Details}\label{sec:appendix_augmentation}
We describe here the data augmentation policies used during both contrastive pretraining and supervised training.

\paragraph{Augmentation pool for color images.}
For color image datasets, we adopt a policy similarly to the RandAugment \cite{cubuk2019randaugment}. Let $\mathcal{T}$ denote a pool of primitive transformations including
\begin{itemize}
    \item Photometric operations: autocontrast, histogram equalization, brightness, contrast, color, sharpness, posterization, solarization;
    \item Geometric operations: rotations, horizontal shears and vertical shears, horizontal translations and vertical translations;
    \item The identity transform.
\end{itemize}
For each image in a mini-batch, we randomly sample a fixed number of transformations from $\mathcal{T}$ with replacement and apply them sequentially. Each selected transform is assigned a magnitude drawn uniformly from a transform-specific range. For example, rotation angles up to a few tens of degrees, shear factors bounded by a small absolute value, or contrast/brightness factors within a bounded interval. After these operations, we apply a cutout transform that removes a randomly located square patch of the image, with the side length sampled as a fraction of the image size. All sampling is done independently for each image and for each view, so that different augmented views of the same sample are stochastically decorrelated.

\paragraph{Augmentation pool for grayscale images.}
For grayscale datasets, we use a variant of the above policy tailored to single-channel images. In this case the transformation pool contains only operations that are meaningful on grayscale inputs:
\begin{itemize}
    \item Geometric transforms: rotations, shears, and translations in the horizontal and vertical directions;
    \item Photometric transforms: inversion, histogram equalization, solarization, brightness, contrast, and sharpness adjustments;
    \item Cutout with a square mask whose side length is sampled up to a fixed proportion, filled with a constant gray or black intensity.
\end{itemize}
As in the color case, we randomly select a small number of transforms for each image, sample their magnitudes from predefined ranges, and apply them sequentially. The parameters of the grayscale policy such as the number of operations, the global magnitude index, and optional magnitude jitter are shared across the dataset, while the specific operations and their sampled magnitudes vary independently from image to image.

\paragraph{Usage in pretraining and supervised training.}
During contrastive pretraining, the above policies are used to create multiple stochastic views of each input example. All samples in a pretraining mini-batch participate in the SimCLR loss via their augmented views, while only the labeled samples in the batch participate in the SupCon loss. In the subsequent supervised training stage with GLL, we continue to apply the same augmentation pipeline to every training sample before feeding it through the encoder and the classifier. In this way, both the contrastive pretraining and the supervised training phases benefit from the same set of strong yet label-preserving data transformations.
}

\section{Supplementary Adversarial Results}

\subsection{Visualization of Attacks}

Adversarial attacks have to balance a trade-off between effectiveness (measured by decay in accuracy of the model) and detectability. Detectability is often quantified by a distance metric, but can be evaluated qualitatively - attacks that are successful but are not noticeable to the human eye are desireable, as this is the type of attack that poses the biggest threat in real-world applications. In Figure \ref{fig:adv_viz}, we visualize natural and attacked images (along with their corresponding attacks) for FGSM, IFGSM, and CW attacks on MNIST of the naturally trained GLL model. Comparing the left and right-most columns, we see that FGSM and IFGSM models are easily detectable by the human eye, whereas the CW attacks are much more focused, and much less noticeable to the human eye. CW attacks are known to be much ``smarter" than fast gradient methods (at the cost of more compute); GLL's overall robustness to all three attacks is a strong result for our proposed model.

\begin{figure}[ht]
\centering
\begin{subfigure}{.75\textwidth}
    \centering
    \includegraphics[trim={0 3.75cm 0 3.75cm}, clip,width=\textwidth]{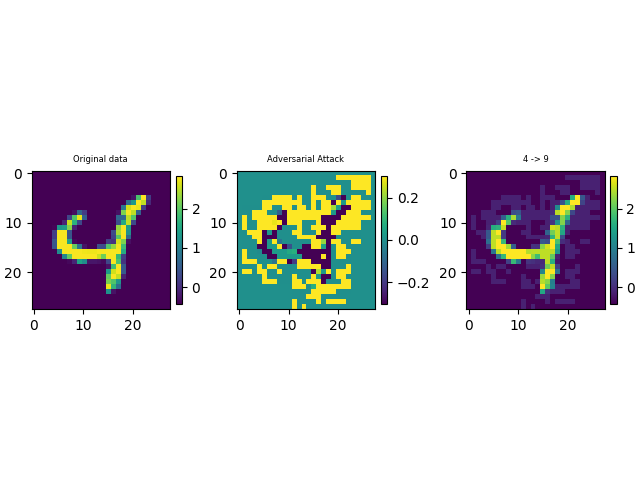}
    \caption{FGSM}
\end{subfigure}%
\hspace{1cm}
\begin{subfigure}{.75\textwidth}
    \centering
    \includegraphics[trim={0 3.75cm 0 3.75cm}, clip,width=\textwidth]{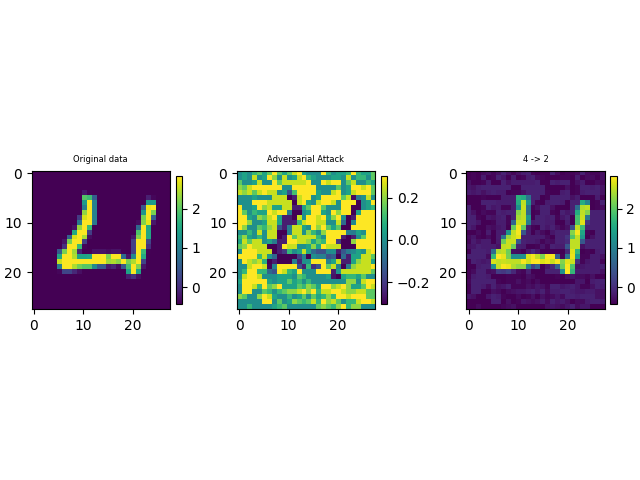}
    \caption{IFGSM}
\end{subfigure}
\hspace{1cm}
\begin{subfigure}{.75\textwidth}
    \centering
    \includegraphics[trim={0 3.75cm 0 3.75cm}, clip,width=\textwidth]{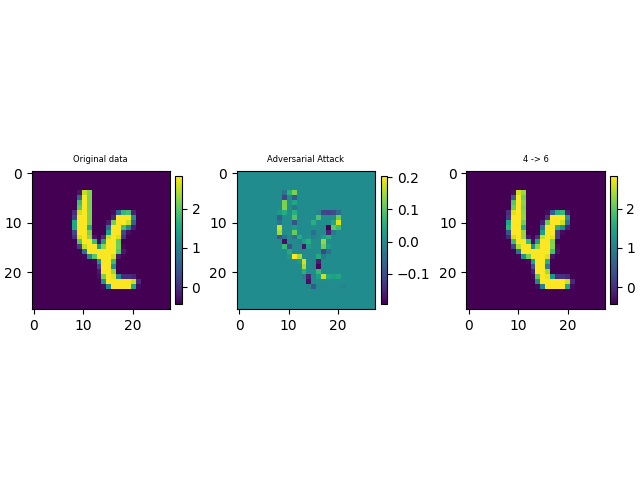}
    \caption{CW}
\end{subfigure}%
\caption{Visualization of FGSM, IFGSM, and CW attacks. The left, middle, and right columns show the original image, the attack, and the adversarial image, respectively. The title in the rightmost column indicates the original classification and the classification after the attack. These images were chosen to demonstrate attacks where the model was fooled by the advesary.}
\label{fig:adv_viz}
\end{figure}

\subsection{Replacing an MLP Classifier with GLL in Testing}

Recall the MNIST setting from Section \ref{adv_results}, where we used a ``small CNN". To further demonstrate the adversarial robustness of GLL, we replaced the final linear layer and softmax classifier used in the "small CNN" with GLL \textit{during the attacks and testing} and \textit{only during testing}. That is, we trained in a normal fashion, before replacing the softmax classifier with GLL for either the attacking (that is, the attacks use the gradient information from GLL) and inference phases, or just for the inference phase. The results are presented in Figure \ref{fig:glhead}. We see that - surprisingly - when we use the gradients from GLL for the attacks, the model is much more robust to FGSM. This supports the idea that GLL is more difficult to attack than the usual linear layer and softmax classifier.

\begin{figure}
    \centering
    \includegraphics[width=.5\textwidth]{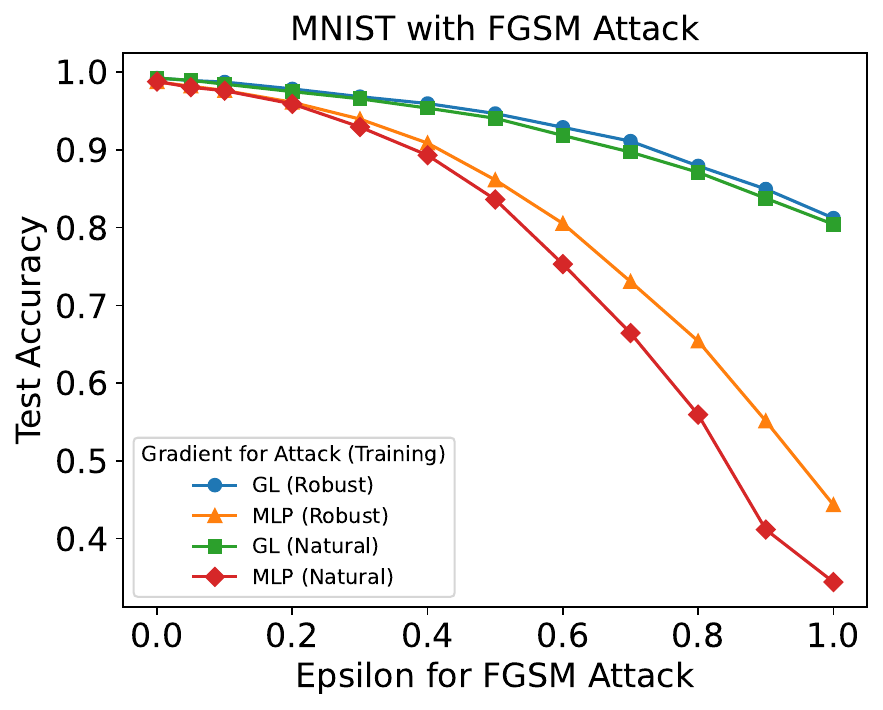}
    \caption{FGSM results on MNIST. The backbone network was trained in the standard way (with a softmax classifier). In the legend, MLP indicates that the FGSM attacks used the gradients from the MLP classifier (ie a linear layer and a softmax), whereas GLL used our graph learning layer. In all cases, the final classifier on the attacked images was GLL. Hence, the only difference between the results is which gradient was used for the attack. We see that the models that utilized the GLL gradients were significantly more robust to the FGSM attacks than using the standard MLP gradients.}
    \label{fig:glhead}
\end{figure}

\subsection{Further CW Experiments}\label{app:more_cw}

In Figure \ref{fig:cw_dist}, we showed that for a fixed $c$, GLL resists the adversary from moving the images compared to the softmax classifier. However, we also wanted to demonstrate that across \textit{distances} in the CW attacks, GLL is also superior. To do so, we ran further experiments where we increased the value of $c$ so that the average distance increased further to the right. The results are shown in Figure \ref{fig:cw_dist_bonus}. For MNIST, we ran additional experiments with $c = 10^4, 10^5, ..., 10^9$. For CIFAR-10, we ran additional experiments with $c = 2000, 5000, 10^4, 10^5, 10^6$.

\begin{figure}[ht]
\centering
\begin{subfigure}{.5\textwidth}
    \centering
    \includegraphics[width=\textwidth]{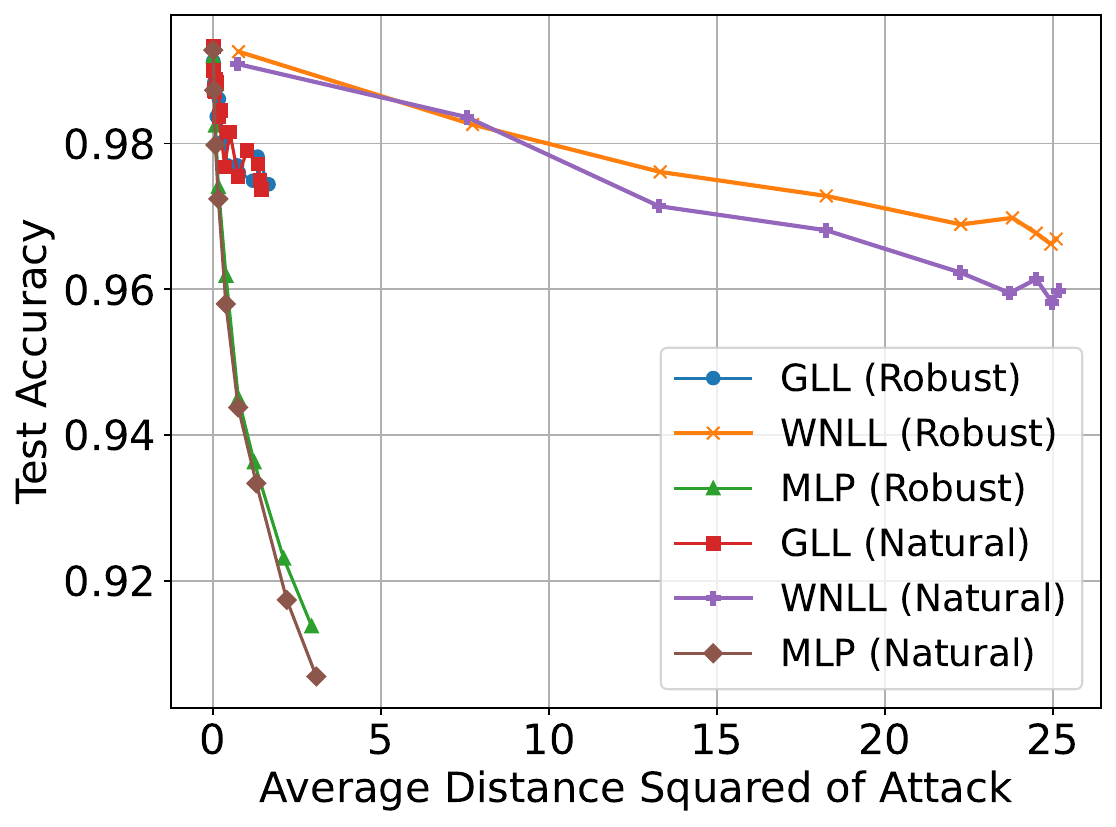}
    \caption{MNIST}
\end{subfigure}%
\begin{subfigure}{.5\textwidth}
    \centering
    \includegraphics[width=\textwidth]{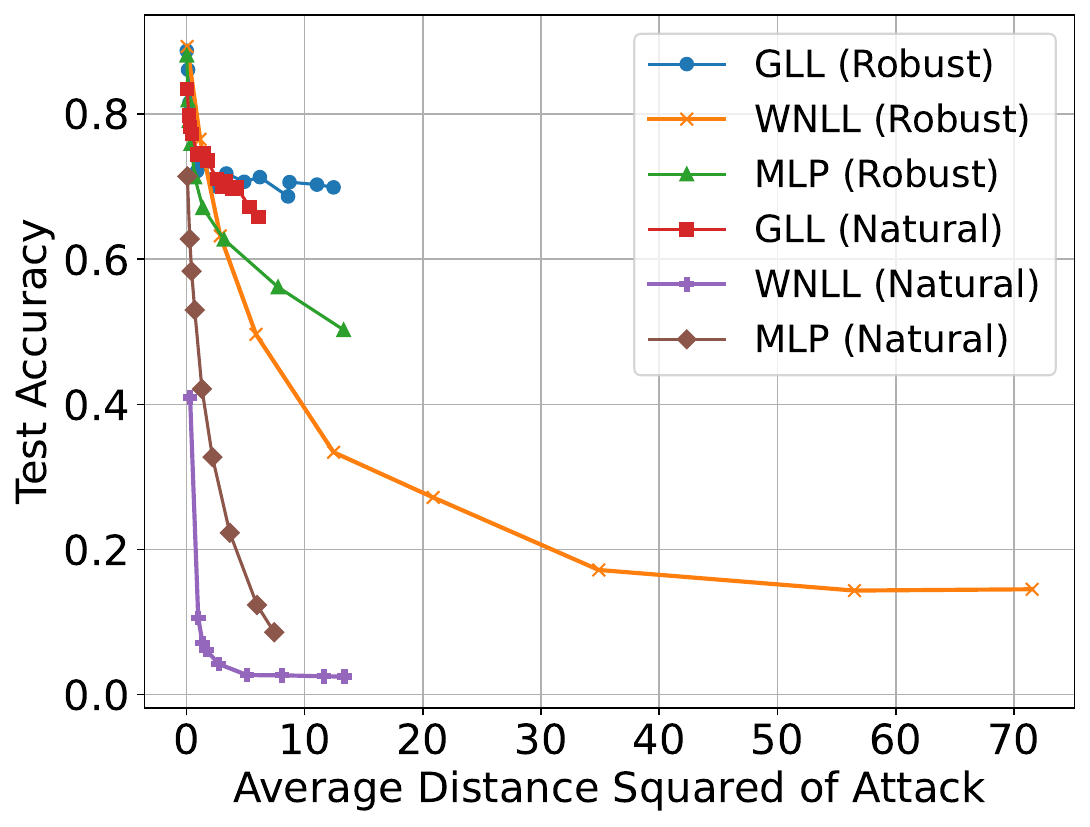}
    \caption{CIFAR-10}
\end{subfigure}
\caption{Comparison of accuracy vs. average distance of attack for CW attack on both datasets. Each point represents a different value of $c$, where we plot additional values of $c$ for the GLL-based models compared to Figure \ref{fig:cw_dist}. We see that GLL is still superior \tcb{at resisting perturbations by CW attack, even when $c$ is orders of magnitude larger than for WNLL or MLP, and even when controlling for distance on CIFAR-10.}}
\label{fig:cw_dist_bonus}
\end{figure}

% {\noindent \em Remainder omitted in this sample. See http://www.jmlr.org/papers/ for full paper.}

\vskip 0.2in
\bibliography{references,references2}

@article{garcia2014multiclass,
  title={Multiclass data segmentation using diffuse interface methods on graphs},
  author={Garcia-Cardona, Cristina and Merkurjev, Ekaterina and Bertozzi, Andrea L and Flenner, Arjuna and Percus, Allon G},
  journal={IEEE transactions on pattern analysis and machine intelligence},
  volume={36},
  number={8},
  pages={1600--1613},
  year={2014},
  publisher={IEEE}
}

@article{manfredi2015nonlinear,
  title={Nonlinear elliptic partial differential equations and p-harmonic functions on graphs},
  author={Manfredi, Juan J and Oberman, Adam M and Sviridov, Alexander P},
  journal={Differential Integral Equations},
  volume={28},
  number={1-2},
  pages={79--102},
  year={2015}
}

@article{barles1991convergence,
  title={Convergence of approximation schemes for fully nonlinear second order equations},
  author={Barles, Guy and Souganidis, Panagiotis E},
  journal={Asymptotic analysis},
  volume={4},
  number={3},
  pages={271--283},
  year={1991},
  publisher={IOS Press}
}

@article{calder2024consistency,
  title={Consistency of semi-supervised learning, stochastic tug-of-war games, and the p-Laplacian},
  author={Calder, Jeff and Drenska, Nadejda},
  arxiv={https://arxiv.org/abs/2401.07463},
  journal={To appear in Active Particles, Volume 4, Advances in Theory, Models, and Applications},
  code={https://github.com/jwcalder/p-Laplace-consistency},
  year={2024}
}

@article{velivckovic2017graph,
  title={Graph attention networks},
  author={Velickovi{\'c}, Petar and Cucurull, Guillem and Casanova, Arantxa and Romero, Adriana and Lio, Pietro and Bengio, Yoshua},
  journal={arXiv preprint arXiv:1710.10903},
  year={2017}
}

@article{wu2020comprehensive,
  title={A comprehensive survey on graph neural networks},
  author={Wu, Zonghan and Pan, Shirui and Chen, Fengwen and Long, Guodong and Zhang, Chengqi and Philip, S Yu},
  journal={IEEE transactions on neural networks and learning systems},
  volume={32},
  number={1},
  pages={4--24},
  year={2020},
  publisher={IEEE}
}

@article{defferrard2016convolutional,
  title={Convolutional neural networks on graphs with fast localized spectral filtering},
  author={Defferrard, Micha{\"e}l and Bresson, Xavier and Vandergheynst, Pierre},
  journal={Advances in neural information processing systems},
  volume={29},
  year={2016}
}

@inproceedings{nadler2009infiniteunlabelled,
author = {Nadler, Boaz and Srebro, Nathan and Zhou, Xueyuan},
title = {Semi-Supervised Learning with the Graph {Laplacian}: {The} Limit of Infinite Unlabelled Data},
year = {2009},
isbn = {9781615679119},
publisher = {Curran Associates Inc.},
address = {Red Hook, NY, USA},
booktitle = {Proceedings of the 22nd International Conference on Neural Information Processing Systems},
pages = {1330–1338},
numpages = {9},
location = {Vancouver, British Columbia, Canada},
series = {NIPS'09}
}

@article{dunlop2020large,
  title={Large data and zero noise limits of graph-based semi-supervised learning algorithms},
  author={Dunlop, Matthew M and Slepcev, Dejan and Stuart, Andrew M and Thorpe, Matthew},
  journal={Applied and Computational Harmonic Analysis},
  volume={49},
  number={2},
  pages={655--697},
  year={2020},
  publisher={Elsevier}
}

@article{slepcev2019analysis,
  title={Analysis of p-{Laplacian} regularization in semisupervised learning},
  author={Slepcev, Dejan and Thorpe, Matthew},
  journal={SIAM Journal on Mathematical Analysis},
  volume={51},
  number={3},
  pages={2085--2120},
  year={2019},
  publisher={SIAM}
}

@article{calder2018game,
  title={The game theoretic p-{Laplacian} and semi-supervised learning with few labels},
  author={Calder, Jeff},
  journal={Nonlinearity},
  volume={32},
  number={1},
  pages={301},
  year={2018},
  publisher={IOP Publishing}
}

@article{calder2022improved,
  title={Improved spectral convergence rates for graph {Laplacians} on $\varepsilon$-graphs and k-{NN} graphs},
  author = {Jeff Calder and Garc\'ia Trillos, N.},
  journal = "Applied and Computational Harmonic Analysis",
  volume = {60},
  pages = {123--175},
  arxiv = {https://arxiv.org/abs/1910.13476},
  code = {https://github.com/jwcalder/kNNSpectralRates},
  url = {https://doi.org/10.1016/j.acha.2022.02.004},
  year={2022}
}

@article{calder_consistency_2019,
	journal = {SIAM J. on Mathematics of Data Science},
	doi = {10.1137/18m1199241},
  year = {2019},
  month = jan,
  publisher = {Society for Industrial {\&} Applied Mathematics ({SIAM})},
  volume = {1},
  number = {4},
  pages = {780--812},
  author = {Jeff Calder},
  title = {Consistency of {Lipschitz} Learning with Infinite Unlabeled Data and Finite Labeled Data}
}

@article{lee2013graph,
  title={Graph-based semi-supervised learning with multi-modality propagation for large-scale image datasets},
  author={Lee, Wen-Yu and Hsieh, Liang-Chi and Wu, Guan-Long and Hsu, Winston},
  journal={Journal of visual communication and image representation},
  volume={24},
  number={3},
  pages={295--302},
  year={2013},
  publisher={Elsevier}
}

@inproceedings{ham2005semisupervised,
  title={Semisupervised alignment of manifolds},
  author={Ham, Jihun and Lee, Daniel and Saul, Lawrence},
  booktitle={International Workshop on Artificial Intelligence and Statistics},
  pages={120--127},
  year={2005},
  organization={PMLR}
}

@article{garcia2020error,
  title={Error estimates for spectral convergence of the graph {L}aplacian on random geometric graphs toward the {Laplace--Beltrami} operator},
  author={Garc{\'\i}a Trillos, Nicol{\'a}s and Gerlach, Moritz and Hein, Matthias and Slepcev, Dejan},
  journal={Foundations of Computational Mathematics},
  volume={20},
  number={4},
  pages={827--887},
  year={2020},
  publisher={Springer}
}

@inproceedings{hein2005graphs,
  title={From Graphs to Manifolds-Weak and Strong Pointwise Consistency of Graph Laplacians.},
  author={Hein, Matthias and Audibert, Jean-Yves and Von Luxburg, Ulrike},
  booktitle={COLT},
  volume={3559},
  pages={470--485},
  year={2005},
  organization={Springer}
}

@article{hein2007graph,
  title={Graph {L}aplacians and their convergence on random neighborhood graphs.},
  author={Hein, Matthias and Audibert, Jean-Yves and Luxburg, Ulrike von},
  journal={Journal of Machine Learning Research},
  volume={8},
  number={6},
  year={2007}
}

@inproceedings{wang2013dynamic,
  title={Dynamic label propagation for semi-supervised multi-class multi-label classification},
  author={Wang, Bo and Tu, Zhuowen and Tsotsos, John K},
  booktitle={Proceedings of the IEEE international conference on computer vision},
  pages={425--432},
  year={2013}
}

@inbook{bengio2006label,
author = {Bengio, Yoshua and Delalleau, Olivier and Le Roux, Nicolas},
title = {Label Propagation and Quadratic Criterion},
booktitle = {Semi-Supervised Learning},
year = {2006},
month = {January},
publisher = {MIT Press},
url = {https://www.microsoft.com/en-us/research/publication/label-propagation-and-quadratic-criterion/},
pages = {193-216},
edition = {Semi-Supervised Learning},
}

@inproceedings{belkin2004regularization,
  title={Regularization and semi-supervised learning on large graphs},
  author={Belkin, Mikhail and Matveeva, Irina and Niyogi, Partha},
  booktitle={Learning Theory: 17th Annual Conference on Learning Theory, COLT 2004, Banff, Canada, July 1-4, 2004. Proceedings 17},
  pages={624--638},
  year={2004},
  organization={Springer}
}

@article{belkin2004semi,
  title={Semi-supervised learning on {Riemannian} manifolds},
  author={Belkin, Mikhail and Niyogi, Partha},
  journal={Machine learning},
  volume={56},
  pages={209--239},
  year={2004},
  publisher={Springer}
}

@article{agrawal2019differentiable,
  title={Differentiable convex optimization layers},
  author={Agrawal, Akshay and Amos, Brandon and Barratt, Shane and Boyd, Stephen and Diamond, Steven and Kolter, J Zico},
  journal={Advances in neural information processing systems},
  volume={32},
  year={2019}
}

@article{calder2022hamilton,
  title={Hamilton-Jacobi equations on graphs with applications to semi-supervised learning and data depth},
  author={Calder, Jeff and Ettehad, Mahmood},
  journal={Journal of Machine Learning Research},
  volume={23},
  number={318},
  pages={1--62},
  year={2022}
}

@inproceedings{zhu2003semi,
  title={Semi-supervised learning using gaussian fields and harmonic functions},
  author={Zhu, Xiaojin and Ghahramani, Zoubin and Lafferty, John D},
  booktitle={Proceedings of the 20th International conference on Machine learning (ICML-03)},
  pages={912--919},
  year={2003}
}

@inproceedings{miller2022graph,
  title={Graph-based active learning for semi-supervised classification of SAR data},
  author={Miller, Kevin and Mauro, Jack and Setiadi, Jason and Baca, Xoaquin and Shi, Zhan and Calder, Jeff and Bertozzi, Andrea L},
  booktitle={Algorithms for Synthetic Aperture Radar Imagery XXIX},
  volume={12095},
  pages={126--139},
  year={2022},
  organization={SPIE}
}

@article{wang2021graph,
  title={Graph interpolating activation improves both natural and robust accuracies in data-efficient deep learning},
  author={Wang, Bao and Osher, Stan J},
  journal={European Journal of Applied Mathematics},
  volume={32},
  number={3},
  pages={540--569},
  year={2021},
  publisher={Cambridge University Press}
}

@article{sellars2021laplacenet,
  title={Laplacenet: A hybrid energy-neural model for deep semi-supervised classification},
  author={Sellars, Philip and Aviles-Rivero, Angelica I and Sch{\"o}nlieb, Carola-Bibiane},
  journal={arXiv preprint arXiv:2106.04527},
  year={2021}
}

@inproceedings{brown2023utilizing,
  title={Utilizing contrastive learning for graph-based active learning of SAR data},
  author={Brown, Jason and O'Neill, Riley and Calder, Jeff and Bertozzi, Andrea L},
  booktitle={Algorithms for Synthetic Aperture Radar Imagery XXX},
  volume={12520},
  pages={181--195},
  year={2023},
  organization={SPIE}
}

@inproceedings{enwright2023deep,
  title={Deep semi-supervised label propagation for SAR image classification},
  author={Enwright, Joshua and Hardiman-Mostow, Harris and Calder, Jeff and Bertozzi, Andrea},
  booktitle={Algorithms for Synthetic Aperture Radar Imagery XXX},
  volume={12520},
  pages={160--172},
  year={2023},
  organization={SPIE}
}

@article{von2007tutorial,
  title={A tutorial on spectral clustering},
  author={Von Luxburg, Ulrike},
  journal={Statistics and computing},
  volume={17},
  pages={395--416},
  year={2007},
  publisher={Springer}
}

@inproceedings{iscen2019label,
  title={Label propagation for deep semi-supervised learning},
  author={Iscen, Ahmet and Tolias, Giorgos and Avrithis, Yannis and Chum, Ondrej},
  booktitle={Proceedings of the IEEE/CVF conference on computer vision and pattern recognition},
  pages={5070--5079},
  year={2019}
}

@article{han2022vision,
  title={Vision gnn: An image is worth graph of nodes},
  author={Han, Kai and Wang, Yunhe and Guo, Jianyuan and Tang, Yehui and Wu, Enhua},
  journal={Advances in neural information processing systems},
  volume={35},
  pages={8291--8303},
  year={2022}
}

@article{zheng2022graph,
  title={A graph-transformer for whole slide image classification},
  author={Zheng, Yi and Gindra, Rushin H and Green, Emily J and Burks, Eric J and Betke, Margrit and Beane, Jennifer E and Kolachalama, Vijaya B},
  journal={IEEE transactions on medical imaging},
  volume={41},
  number={11},
  pages={3003--3015},
  year={2022},
  publisher={IEEE}
}

@article{chen2024survey,
  title={A survey on graph neural networks and graph transformers in computer vision: A task-oriented perspective},
  author={Chen, Chaoqi and Wu, Yushuang and Dai, Qiyuan and Zhou, Hong-Yu and Xu, Mutian and Yang, Sibei and Han, Xiaoguang and Yu, Yizhou},
  journal={IEEE Transactions on Pattern Analysis and Machine Intelligence},
  year={2024},
  publisher={IEEE}
}

@article{shi2017weighted,
  title={Weighted nonlocal laplacian on interpolation from sparse data},
  author={Shi, Zuoqiang and Osher, Stanley and Zhu, Wei},
  journal={Journal of Scientific Computing},
  volume={73},
  pages={1164--1177},
  year={2017},
  publisher={Springer}
}

@article{calder2023rates,
  title={Rates of convergence for Laplacian semi-supervised learning with low labeling rates},
  author={Calder, Jeff and Slepcev, Dejan and Thorpe, Matthew},
  journal={Research in the Mathematical Sciences},
  volume={10},
  number={1},
  pages={10},
  year={2023},
  publisher={Springer}
}

@article{scikit-learn,
  title={Scikit-learn: Machine Learning in {P}ython},
  author={Pedregosa, F. and Varoquaux, G. and Gramfort, A. and Michel, V.
          and Thirion, B. and Grisel, O. and Blondel, M. and Prettenhofer, P.
          and Weiss, R. and Dubourg, V. and Vanderplas, J. and Passos, A. and
          Cournapeau, D. and Brucher, M. and Perrot, M. and Duchesnay, E.},
  journal={Journal of Machine Learning Research},
  volume={12},
  pages={2825--2830},
  year={2011}
}

@article{goodfellow2014explaining,
  title={Explaining and harnessing adversarial examples},
  author={Goodfellow, Ian J and Shlens, Jonathon and Szegedy, Christian},
  journal={arXiv preprint arXiv:1412.6572},
  year={2014}
}

@article{madry2017towards,
  title={Towards deep learning models resistant to adversarial attacks},
  author={Madry, Aleksander and Makelov, Aleksandar and Schmidt, Ludwig and Tsipras, Dimitris and Vladu, Adrian},
  journal={arXiv preprint arXiv:1706.06083},
  year={2017}
}

@article{kingma2013auto,
title={Auto-encoding variational bayes},
author={Kingma, Diederik P and Welling, Max},
journal={ArXiv Preprint ArXiv},
year={2013}
}

@article{pu2016variational,
title={Variational autoencoder for deep learning of images, labels and captions},
author={Pu, Yunchen and Gan, Zhe and Henao, Ricardo and Yuan, Xin and Li, Chunyuan and Stevens, Andrew and Carin, Lawrence},
journal={Advances in Neural Information Processing Systems},
volume={29},
year={2016}
}

@article{lecun1989backpropagation,
title={Backpropagation applied to handwritten zip code recognition},
author={LeCun, Yann and Boser, Bernhard and Denker, John S and Henderson, Donnie and Howard, Richard E and Hubbard, Wayne and Jackel, Lawrence D},
journal={Neural Computation},
volume={1},
number={4},
pages={541--551},
year={1989},
publisher={MIT Press}
}

@inproceedings{calder2020poisson,
title={Poisson learning: Graph based semi-supervised learning at very low label rates},
author={Calder, Jeff and Cook, Brendan and Thorpe, Matthew and Slepcev, Dejan},
booktitle={International Conference on Machine Learning},
pages={1306--1316},
year={2020},
organization={PMLR}
}

@article{calder2020properly,
title={Properly-weighted graph Laplacian for semi-supervised learning},
author={Calder, Jeff and Slepcev, Dejan},
journal={Applied Mathematics \& Optimization},
volume={82},
pages={1111--1159},
year={2020},
publisher={Springer}
}

@article{miller2023poisson,
title={Poisson reweighted Laplacian uncertainty sampling for graph-based active learning},
author={Miller, Kevin and Calder, Jeff},
journal={SIAM Journal on Mathematics of Data Science},
volume={5},
number={4},
pages={1160--1190},
year={2023},
publisher={SIAM}
}

@inproceedings{zhou2011semi,
title={Semi-supervised learning by higher order regularization},
author={Zhou, Xueyuan and Belkin, Mikhail},
booktitle={Proceedings of the Fourteenth International Conference on Artificial Intelligence and Statistics},
pages={892--900},
year={2011},
organization={JMLR Workshop and Conference Proceedings}
}

@inproceedings{el2016asymptotic,
title={Asymptotic behavior of $\ell_p$-based Laplacian regularization in semi-supervised learning},
author={El Alaoui, Ahmed and Cheng, Xiang and Ramdas, Aaditya and Wainwright, Martin J and Jordan, Michael I},
booktitle={Conference on Learning Theory},
pages={879--906},
year={2016},
organization={PMLR}
}

@article{flores2022analysis,
title={Analysis and algorithms for $\ell_p$-based semi-supervised learning on graphs},
author={Flores, Mauricio and Calder, Jeff and Lerman, Gilad},
journal={Applied and Computational Harmonic Analysis},
volume={60},
pages={77--122},
year={2022},
publisher={Elsevier}
}

@article{bertozzi2016diffuse,
  title={Diffuse interface models on graphs for classification of high dimensional data},
  author={Bertozzi, Andrea L and Flenner, Arjuna},
  journal={SIAM Review},
  volume={58},
  number={2},
  pages={293--328},
  year={2016},
  publisher={SIAM}
}

@article{merkurjev2017modified,
  title={Modified Cheeger and ratio cut methods using the Ginzburg--Landau functional for classification of high-dimensional data},
  author={Merkurjev, Ekaterina and Bertozzi, Andrea and Yan, Xiaoran and Lerman, Kristina},
  journal={Inverse Problems},
  volume={33},
  number={7},
  pages={074003},
  year={2017},
  publisher={IOP Publishing}
}

@article{chen2023batch,
title={Batch active learning for multispectral and hyperspectral image segmentation using similarity graphs},
author={Chen, Bohan and Miller, Kevin and Bertozzi, Andrea L and Schwenk, Jon},
journal={Communications on Applied Mathematics and Computation},
pages={1--21},
year={2023},
publisher={Springer}
}

@article{merkurjev2018semi,
  title={A semi-supervised heat kernel pagerank MBO algorithm for data classification},
  author={Merkurjev, Ekaterina and Bertozzi, Andrea L and Chung, Fan},
  journal={Communications in mathematical sciences},
  volume={16},
  number={5},
  pages={1241--1265},
  year={2018}
}

@article{khosla2020supervised,
  title={Supervised contrastive learning},
  author={Khosla, Prannay and Teterwak, Piotr and Wang, Chen and Sarna, Aaron and Tian, Yonglong and Isola, Phillip and Maschinot, Aaron and Liu, Ce and Krishnan, Dilip},
  journal={Advances in neural information processing systems},
  volume={33},
  pages={18661--18673},
  year={2020}
}

@inproceedings{chen2020simple,
  title={A simple framework for contrastive learning of visual representations},
  author={Chen, Ting and Kornblith, Simon and Norouzi, Mohammad and Hinton, Geoffrey},
  booktitle={International conference on machine learning},
  pages={1597--1607},
  year={2020},
  organization={PMLR}
}

@article{chen2020big,
  title={Big self-supervised models are strong semi-supervised learners},
  author={Chen, Ting and Kornblith, Simon and Swersky, Kevin and Norouzi, Mohammad and Hinton, Geoffrey E},
  journal={Advances in neural information processing systems},
  volume={33},
  pages={22243--22255},
  year={2020}
}

@inproceedings{
wang2023message,
title={A Message Passing Perspective on Learning Dynamics of Contrastive Learning},
author={Yifei Wang and Qi Zhang and Tianqi Du and Jiansheng Yang and Zhouchen Lin and Yisen Wang},
booktitle={International Conference on Learning Representations},
year={2023},
}

@article{xiao2017fashion,
  title={Fashion-mnist: a novel image dataset for benchmarking machine learning algorithms},
  author={Xiao, Han and Rasul, Kashif and Vollgraf, Roland},
  journal={arXiv preprint arXiv:1708.07747},
  year={2017}
}

@inproceedings{he2016deep,
  title={Deep residual learning for image recognition},
  author={He, Kaiming and Zhang, Xiangyu and Ren, Shaoqing and Sun, Jian},
  booktitle={Proceedings of the IEEE conference on computer vision and pattern recognition},
  pages={770--778},
  year={2016}
}

@incollection{kurakin2018adversarial,
  title={Adversarial examples in the physical world},
  author={Kurakin, Alexey and Goodfellow, Ian J and Bengio, Samy},
  booktitle={Artificial intelligence safety and security},
  pages={99--112},
  year={2018},
  publisher={Chapman and Hall/CRC}
}

@inproceedings{carlini2017towards,
  title={Towards evaluating the robustness of neural networks},
  author={Carlini, Nicholas and Wagner, David},
  booktitle={2017 ieee symposium on security and privacy (sp)},
  pages={39--57},
  year={2017},
  organization={Ieee}
}

@article{szegedy2013intriguing,
  title={Intriguing properties of neural networks},
  author={Szegedy, Christian and Zaremba, Wojciech and Sutskever, Ilya and Bruna, Joan and Erhan, Dumitru and Goodfellow, Ian and Fergus, Rob},
  journal={arXiv preprint arXiv:1312.6199},
  year={2013}
}

@inproceedings{chapman2023novel,
  title={Novel batch active learning approach and its application to synthetic aperture radar datasets},
  author={Chapman, James and Chen, Bohan and Tan, Zheng and Calder, Jeff and Miller, Kevin and Bertozzi, Andrea L},
  booktitle={Algorithms for Synthetic Aperture Radar Imagery XXX},
  volume={12520},
  pages={95--110},
  year={2023},
  organization={SPIE}
}

@article{wu2022nodeformer,
  title={Nodeformer: A scalable graph structure learning transformer for node classification},
  author={Wu, Qitian and Zhao, Wentao and Li, Zenan and Wipf, David P and Yan, Junchi},
  journal={Advances in Neural Information Processing Systems},
  volume={35},
  pages={27387--27401},
  year={2022}
}

@article{chen2024batch,
  title={Batch active learning for multispectral and hyperspectral image segmentation using similarity graphs},
  author={Chen, Bohan and Miller, Kevin and Bertozzi, Andrea L and Schwenk, Jon},
  journal={Communications on Applied Mathematics and Computation},
  volume={6},
  number={2},
  pages={1013--1033},
  year={2024},
  publisher={Springer}
}

@article{cortes1995support,
title={Support-vector networks},
author={Cortes, Corinna and Vapnik, Vladimir},
journal={Machine Learning},
volume={20},
number={3},
pages={273--297},
year={1995},
publisher={Springer}
}

@inproceedings{ho1995random,
title={Random decision forests},
author={Ho, Tin Kam},
booktitle={Proceedings of 3rd International Conference on Document Analysis and Recognition},
volume={1},
pages={278--282},
year={1995},
organization={IEEE}
}

@inproceedings{brown2023material,
title={Material identification in complex environments: Neural network approaches to hyperspectral image analysis},
author={Brown, Jason and Chen, Bohan and Hardiman-Mostow, Harris and Weihs, Adrien and Bertozzi, Andrea L and Chanussot, Jocelyn},
booktitle={2023 13th Workshop on Hyperspectral Imaging and Signal Processing: Evolution in Remote Sensing (WHISPERS)},
pages={1--5},
year={2023},
organization={IEEE}
}

@INPROCEEDINGS{chen2023graphigarss,
author={Chen, Bohan and Miller, Kevin and Bertozzi, Andrea L. and Schwenk, Jon},
booktitle={IGARSS 2023 - 2023 IEEE International Geoscience and Remote Sensing Symposium},
title={Graph-based active learning for surface water and sediment detection in multispectral images},
year={2023},
volume={},
number={},
pages={5431-5434},
doi={10.1109/IGARSS52108.2023.10282009}}

@article{kingma2014adam,
  title={Adam: A method for stochastic optimization},
  author={Kingma, Diederik P and Ba, Jimmy},
  journal={arXiv preprint arXiv:1412.6980},
  year={2014}
}

@inproceedings{kipf2016semi,
  title={Semi-Supervised Classification with Graph Convolutional Networks},
  author={Kipf, Thomas N and Welling, Max},
  booktitle={International Conference on Learning Representations},
  year={2016}
}

@misc{mnist,
  title        = "MNIST handwritten digit database",
  author       = {Yann LeCun, Corinna Cortes, Christopher J.C. Burges},
  howpublished = "\url{http://yann.lecun.com/exdb/
mnist/}",
  year         = 2010
}

@article{krizhevsky2009learning,
  title={Learning multiple layers of features from tiny images},
  author={Krizhevsky, Alex and Hinton, Geoffrey and others},
  year={2009},
  journal={Toronto, ON, Canada}
}

@inproceedings{he2016identity,
  title={Identity mappings in deep residual networks},
  author={He, Kaiming and Zhang, Xiangyu and Ren, Shaoqing and Sun, Jian},
  booktitle={Computer Vision--ECCV 2016: 14th European Conference, Amsterdam, The Netherlands, October 11--14, 2016, Proceedings, Part IV 14},
  pages={630--645},
  year={2016},
  organization={Springer}
}

@article{loshchilov2016sgdr,
  title={Sgdr: Stochastic gradient descent with warm restarts},
  author={Loshchilov, Ilya and Hutter, Frank},
  journal={arXiv preprint arXiv:1608.03983},
  year={2016}
}

@article{settles2009active,
  title={Active learning literature survey},
  author={Settles, Burr},
  year={2009},
  journal={University of Wisconsin-Madison Department of Computer Sciences}
}

@article{miller2024model,
  title={Model Change Active Learning in Graph-Based Semi-supervised Learning},
  author={Miller, Kevin S and Bertozzi, Andrea L},
  journal={Communications on Applied Mathematics and Computation},
  pages={1--29},
  year={2024},
  publisher={Springer}
}

@inproceedings{Mumfordgraph,
author="Huiyi Hu and Justin Sunu and Andrea L. Bertozzi", 
title="Multi-class Graph Mumford-Shah Model for Plume Detection using the MBO scheme", 
journal="Proceedings of the EMMCVPR conference in Hong Kong 2015",
pages="209-222",
editors="X. -C. Tai et al (Eds)",
booktitle="Springer Lecture Notes in Computer Science",
volume=8932,
year=2015}

@article{Boyd18,
author = {Boyd, Zachary M. and Bae, Egil and Tai, Xue-Cheng and Bertozzi, Andrea L.},
title = {Simplified Energy Landscape for Modularity Using Total Variation},
journal = {SIAM Journal on Applied Mathematics},
volume = {78},
number = {5},
pages = {2439-2464},
year = {2018},
doi = {10.1137/17M1138972}
}

@article{simonyan2014very,
  title={Very deep convolutional networks for large-scale image recognition},
  author={Simonyan, Karen and Zisserman, Andrew},
  journal={arXiv preprint arXiv:1409.1556},
  year={2014}
}

@article{chen2025cgap,
  author={Chen, Bohan and Miller, Kevin and Bertozzi, Andrea L. and Schwenk, Jon},
  journal={IEEE Journal of Selected Topics in Applied Earth Observations and Remote Sensing}, 
  title={CGAP: A Hybrid Contrastive and Graph-Based Active Learning Pipeline to Detect Water and Sediment in Multispectral Images}, 
  year={2025},
  volume={18},
  number={},
  pages={446-462},
  doi={10.1109/JSTARS.2024.3493073}}

@article{cubuk2019randaugment,
  author       = {Ekin D. Cubuk and
                  Barret Zoph and
                  Jonathon Shlens and
                  Quoc V. Le},
  title        = {RandAugment: Practical data augmentation with no separate search},
  journal      = {CoRR},
  volume       = {abs/1909.13719},
  year         = {2019},
  url          = {http://arxiv.org/abs/1909.13719},
  eprinttype    = {arXiv},
  eprint       = {1909.13719},
  bibsource    = {dblp computer science bibliography, https://dblp.org}
}

@article{merkurjev_mbo_2013,
	title = {An MBO Scheme on Graphs for Classification and Image Processing},
	volume = {6},
	url = {https://epubs.siam.org/doi/abs/10.1137/120886935},
	doi = {10.1137/120886935},
	number = {4},
	urldate = {2020-06-11},
	journal = {SIAM Journal on Imaging Sciences},
	author = {Merkurjev, Ekaterina and Kostić, Tijana and Bertozzi, Andrea L.},
	month = jan,
	year = {2013},
	publisher = {SIAM},
	pages = {1903--1930},
}

@misc{cohen2017emnist,
  author = {Cohen, Gregory and Afshar, Saeed and Tapson, Jonathan and van Schaik, André},
  description = {[1702.05373] EMNIST: {An} extension of MNIST to handwritten letters},
  title = {{EMNIST}: {An} extension of {MNIST} to handwritten letters},
  url = {http://arxiv.org/abs/1702.05373},
  year = 2017
}

\end{document}